\documentclass{article}
\usepackage[preprint]{neurips_2026}

\usepackage{graphicx} 
\usepackage{amsmath, amssymb, amsthm}
\usepackage{mathtools}
\usepackage{csquotes}
\usepackage{hyperref}
\usepackage{tikz}
\usetikzlibrary{arrows.meta}
\usetikzlibrary{decorations.pathreplacing}

\usepackage{xcolor}
\usepackage{algorithm}
\usepackage{algorithmic}
\usepackage{graphicx}
\theoremstyle{plain}
\newtheorem{theorem}{Theorem}[section]
\newtheorem{lemma}[theorem]{Lemma}
\theoremstyle{definition}

\newtheorem{proposition}[theorem]{Proposition}
\usepackage{subcaption}

\title{Representations from Pretrained Machine-Learning Interatomic Potentials as Coarse Coordinates for Material Generation and Evaluation}
\author{
  Paul Hagemann \\
  \texttt{paul.hagemann@bam.de}
  \And
  Katharina Ueltzen \\
  \texttt{katharina.ueltzen@bam.de}
  \And
  Simon Müller \\
  \texttt{simon.mueller@bam.de}
  \And
  Janine George \\
  \texttt{janine.george@bam.de}
  \And
  Philipp Benner \\
  \texttt{philipp.benner@bam.de}
}
\date{January 2026}

\begin{document}

\maketitle
\begin{abstract}
Generative machine learning is increasingly used for inorganic crystal structure generation. Most models and the corresponding evaluation approaches rely on simple forms of crystal structure representation. In this paper, we showcase the power of atom-averaged features from pretrained Machine-Learning Interatomic Potentials (MLIPs), such as MACE, for such tasks. We first introduce a distance measure that assesses the output of material generative models by capturing both quality and novelty in a single distribution-based evaluation framework. In particular, we introduce the Coarse-Fine Transport Distance (CFTD) using two different featurizers, where the quality component is based on coarse MACE features. We showcase CFTD's versatility in capturing crystal-structure quality while also detecting memorization, and compare it with the recently introduced continuous SUN metrics. We further show that coarse MACE features can be used as guidance for a material generative model.
\end{abstract}

\section{Introduction}
Materials are an important foundation of technological progress, and finding new materials with desired properties is crucial for future developments in many different fields. Inorganic crystals underpin many devices and applications, but because the space of possible crystal structures is extremely vast, suitable candidates are difficult to identify and this process is time-consuming and expensive. With recent advancements in generative machine learning, various models have emerged as promising tools for creating inorganic crystal structures from scratch \cite{debreuck2026}. Multiple publications \cite{MatterGen2025, xie2022crystal, wu2026dmflow, WyckoffDiff} have demonstrated strong potential to generate realistic materials not present in the training data, and new models are published frequently. 

Conversely, synthesized crystals originating from generative models remain rare, and the novelty of such discoveries has been heavily debated \cite{MatterGen2025, juesholt}. Additionally, and perhaps related to this, metric development for evaluating materials generative models seems harder than in other domains: in particular, \emph{novelty} and \emph{stability} are crucial for model design and evaluation. Determining these two characteristics is not straightforward, and, lacking alternatives, most material generators are evaluated based on their SUN-metrics, as for instance done in MatterGen \cite{MatterGen2025}. This now-established evaluation criterion measures stability, uniqueness, and novelty at the instance level (i.e., structure-by-structure), which would allow an e.g. highly collapsed generative model to still produces novel and stable materials. 

Since generative models are usually trained to match the training data distribution, distributional metrics might offer a more natural approach to evaluation. 

Other evaluation criteria, as used in \cite{xie2022crystal, betala2026lematgenbenchunifiedevaluationframework}, target specific properties such as atom or space-group diversity or the volume of the generated crystals. These are more distributional but yield a reliable picture of performance only when combined. This inspired us to consider a more Fréchet Inception Distance (FID)-like approach \cite{fid_paper} for metric design. In their paper, the authors argued that a classifier for images also has a good feature representation of images. One possible analog for materials is provided by the hidden features of pre-trained Machine-Learning Interatomic Potentials (MLIPs), such as MACE \cite{batatia2025foundationmodelatomisticmaterials}, which learn energetic representations of materials based on atomic environments.
However, classical distributional metrics are minimized when we copy the exact training distribution (i.e., we do not create any novel samples and their novelty is, accordingly, very low). We will therefore focus on our recently introduced Transport Novelty Distance (TNovD) \cite{hagemann2025transportnoveltydistancedistributional}, which unifies quality and novelty within a single framework. It is inspired by optimal transport \cite{peyré2020computationaloptimaltransport} with a distributional flavor. The initial TNovD used a contrastive GNN featurizer to evaluate both quality and novelty. However, we decided to extend the metric because, by using only a single featurizer, every improvement in quality necessarily decreases novelty, and vice versa. To overcome this drawback, we established a Goldilocks zone, which lies between a memorization and a quality regime and includes structures not considered for evaluation at all. The two separated regimes are established through two different featurizers: a MACE featurizer for quality, and a contrastive GNN featurizer for novelty evaluation. Since each featurizer contains either coarse or fine information of the investigated structure, we call the new metric Coarse-Fine Transport distance (CFTD), borrowing the terminology from vision community, where coarse and fine features capture different levels of visual detail.

We discuss how the classical coverage-novelty tradeoff relates to our metric, and, by replicating toy experiments from \cite{hagemann2025transportnoveltydistancedistributional}, we show that the new approach including an additional MLIP featurizer is chemically more sensible and captures stability much better than \cite{hagemann2025transportnoveltydistancedistributional}. To investigate differences relative to more instance-based metrics, such as continuous SUN \cite{negishi2026continuoussunstableunique}, we added experiments that resemble mode-collapsing generative models. In such cases, our newly introduced Coarse-Fine Transport Distance (CFTD) behaves more reasonably than instance-based metrics.
We further analyze the geometries of the fine contrastive featurizer and the coarse MACE featurizer (Appendix \ref{app:fine_coarse}) and additionally show, which informational content the MACE features possess (Appendix \ref{app:coarse_mace}). 

We then benchmark sets of 10,000 generated crystal structures, each created with a different materials generative model \cite{MatterGen2025, jiao2024crystalstructurepredictionjoint, jiao2024spacegroupconstrainedcrystal, xie2022crystal, joshi2025allatomdiffusiontransformersunified, llmcrystalgeneration, park2025guidinggenerativemodelsuncover}. Benchmarking is done before and after relaxation of the resulting crystal structures with a pre-trained MLIP, and the qualitative behavior is discussed as opposed to \cite{negishi2026continuoussunstableunique}. In particular, we observe that relaxing the generated structures generally improves the model's quality but often increases its memorization. This suggests that, to some extent, generative models push existing materials slightly away from their energy minima, rather than generating truly novel and stable structures.

\begin{figure}[h!]
\centering
\includegraphics[height=0.33\textheight]{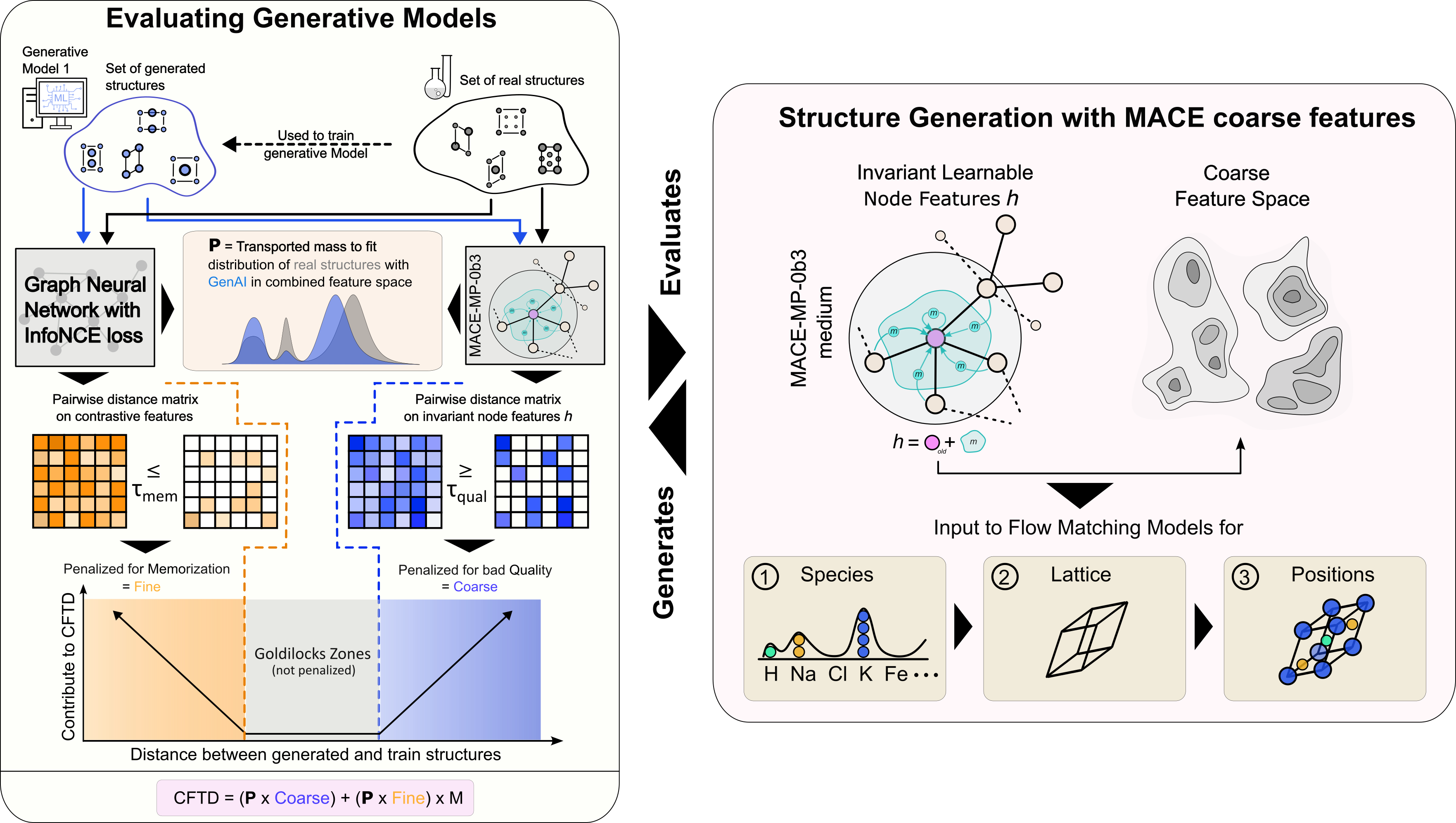}%
\caption{Overview of the CFTD metric and the generative pipeline. We propose using mean-pooled, invariant h features from the message passing step of a pretrained MACE model in both the evaluation and generation procedures, a practice vital for material design.}
\label{fig:CFTD}
\end{figure}

The results from CFTD indicate that the hidden MACE features carry structural, chemical, and physical information, even if they are used in a compressed form. 
To further investigate possible applications of the hidden MACE features, we created an MLIP-conditional material generative model and trained it by including empirically sampled hidden MACE features from the training set as an additional conditioning signal. The model overfits when trained on only a small subset of the training set, a behavior that is detected by the CFTD. We then trained our generative model again on the full training set, this time with and without MACE feature conditioning. The results demonstrate that the architecture indeed benefits from MACE conditioning by showing that the generated materials obey the relation to their condition. The MACE-conditional model exhibits a very different stability-novelty profile, confirming our thesis. We evaluated both models on the recently introduced LeMat-GenBench benchmark \cite{betala2026lematgenbenchunifiedevaluationframework}. 

\subsection{Related Work}
\paragraph{Evaluation of generative models}
In imaging, early work used precision-recall \cite{prec_rec1, prec_rec2} or nearest-neighbor checks. Then, feature-based metrics like LPIPS \cite{zhang2018perceptual} for single images or Frechet Inception Distance (FID) \cite{fid_paper} became more popular, as they better aligned with human visual judgments. 

In terms of materials generative model evaluation, the most common methods are the SUN-metrics, to the best of our knowledge first introduced in \cite{MatterGen2025}. These have been refined recently in \cite{negishi2025continuous}, which treats stability, uniqueness, and novelty not as binary objectives but as continuous ones. Further, Kelvinius et al. \cite{WyckoffDiff} proposes a FID-style evaluation in the activation space of a GNN, which is close in spirit but does not capture novelty nor stability. In this paper, we follow up on \cite{hagemann2025transportnoveltydistancedistributional}, where we introduced a quality-novelty metric. 
First, we establish the connection to existing metrics, then refine it within a dual-featurizer framework. We do this to avoid the otherwise opposing behavior of novelty and quality. 
We then show that these changes enable the new metric to capture stability much better than \cite{hagemann2025transportnoveltydistancedistributional} and to capture mode collapse much better than \cite{negishi2025continuous}.

When evaluating materials generative models, particular focus is given to novelty. The general problem is that there is no easy way to distinguish whether something is just another representation of an existing material, in part because there is debate about the definition of what actually counts as the same material \cite{widdowson2022}. Recently, there has been some controversy over whether some of the materials discovered by generative models were already known \cite{juesholt}. Additionally, Negishi et al. \cite{negishi2026} showed that a large majority of chemically valid and metastable crystals created by current generative models are either training duplicates or can be reproduced by substitutions of training samples with chemically similar elements. In our new metric, we mainly treat novelty with respect to the featurizer used for detecting memorization. In similar spirit to \cite{negishi2026}, this featurizer is trained to treat substitutions with chemically similar elements as the same sample, since those structures can be generated using more classical generative approaches.
This featurizer can be easily modified to fit new or other definitions of novel. 

\paragraph{MLIP-conditional generation}
Recently, there has been some work using MLIP features for materials generation. This development is consistent with a recent commentary in ref. \cite{didi2025r4g}, which advocates using feature-based representations for generative modeling, as is common in the vision community \cite{oquab2024dinov, rombach2022high}. For crystal structures, this translates to using the understanding of pretrained foundation models (in the form of their hidden feature space), instead of generating or classifying in the raw materials space. 

Le et al. \cite{le2025equivariant} trained an encoder-decoder framework to conditionally generate ligands in structure-based drug design, and their encoder is partly obtained by learned training representations. In ref. \cite{li2025geometric}, they learn a conditional/hierarchical model for generating organic molecules based on the UniMol representation. This is very similar to what we do in practice, with both approaches introducing a two-stage training process. However, their conditioning is focused on the semantic meaning, whereas ours is based on hidden features from the MACE foundation model. It thereby captures many more physical and chemical properties, as MACE is focused on predicting energies, forces and stresses. In ref. \cite{li2025platonicrepresentationfoundationmachine}, a strong case for using hidden MLIP-features for conditioning generative models is made: they show that different MLIPs seem to converge to the same representation, once anchored. Li et al. \cite{li2026elignequivariantdiffusionmodel} use MLIP force fields in a reinforcement learning-style guidance procedure, highlighting their effectiveness for generating molecular structures. Pinede et al. \cite{pinede2025unifying} also use MLIP representations by regularizing the Boltzmann generators' hidden states to be close to a pretrained MACE model. Recently, also Hessmann el al. \cite{hessmann2026generativepseudoforcefieldsmolecular} draw the analogy between the denoising behavior of diffusion \cite{song2021scorebased} and relaxation by learning a pseudo force fields. By perturbing equilibrium molecules and learning a suitable denoiser, they overcome the need for calculating DFT and get a generative model for trajectory data.

\subsection{Contributions}
Overall, our approach is inspired by hierarchical image generation, where a coarse-to-fine approach is often pursued; see, e.g., \cite{kadkhodaie2023learning, tian2024visual}. The first thesis of our paper is that a metric that captures both memorization and quality requires a dual-featurizer geometry, in which each component (coarse or fine) captures distinct properties. For the coarse (or quality) component, we introduce a specific kind of MACE feature that is useful both for evaluation and for guiding a material generator. Our contributions are as followed:
 
\begin{itemize}
    \item We introduce the Coarse-Fine Transport Distance (CFTD), which improves upon TNovD \cite{hagemann2025transportnoveltydistancedistributional}. It resolves the quality-novelty trade-off intrinsic to metrics that capture both with a single featurizer, and we relate it to classical coverage and novelty. 
    \item For the coarse (or quality) component, we take random projections of pooled hidden MACE features. We also investigate what information these features capture. For the fine (or novelty) component, we build upon the contrastive GNN from \cite{hagemann2025transportnoveltydistancedistributional} and refine it using substitutions with chemically similar elements \cite{Glawe_2016}. 
    \item We show that, in contrast to cSUN \cite{negishi2025continuous}, the CFTD reliably tracks stability and mode collapse, respectively. Using toy experiments, we verify that it detects quality errors and memorization, and show that it provides a better representation of stability and chemical composition than \cite{hagemann2025transportnoveltydistancedistributional}.
    \item We benchmark different material generators using CFTD, and show that it captures memorization in an overfitted material generator.
    \item We also show that a material generative model benefits from integrating MACE features and that these features can serve as a guidance signal. 
\end{itemize}

\section{Resolving the Coverage-Novelty tradeoff for Materials Generation with the Coarse-Fine Transport Distance}

In this section, we review established methods for measuring the performance of generative models and motivate our proposed Coarse-Fine Transport Distance (CFTD). We first discuss more classical, instance-level coverage and novelty scores and how they relate to quality and novelty of structures. 

\paragraph{The classical Coverage-Novelty tradeoff and its implications for material generation}
One of the main issues in generative modeling is that we want a generative model that captures statistics of the data distribution, i.e., reproduces similar behavior to the data, while not overfitting on the specific training examples. In essence, we want to generate \emph{novel} samples that still resemble the features from the data distribution. Classically, this is often summed up in the terms \emph{coverage} and \emph{novelty}, which we define in Appendix \ref{app:tnovd}. In the general case, coverage describes the amount of training data for which a close generated sample can be found, while novelty describes the amount of generated samples for which no close training sample exists. These two points are generally at odds with one another when measured under the same notion of closeness. 
In materials generation, high coverage is specifically wanted for the quality parameters of a generated structure, while, novelty here refers to structural and compositional resemblance of a generated structure to the training data. These two things do not necessarily need to be measured in the same reference space (i.e. under the same notion of closeness), and we resolve the classical tradeoff by introducing two featurizers that decouple quality coverage from novelty, allowing us to evaluate both at the same time. In the following, we will use the term quality coverage to refer to how well the quality aspects of the generated structures is covered by the training data, while the term novelty will refer to structure memorization. 

\paragraph{Coarse-Fine Transport Novelty Distance}

 The CFTD is an extension of the recently proposed Transport Novelty Distance (TNovD) \cite{hagemann2025transportnoveltydistancedistributional}, see Appendix \ref{app:tnovd} for a concise summary and the relation to coverage and novelty. 

However, by using only a single featurizer, the TNovD has the above-mentioned quality coverage-novelty tradeoff problem. 
More specifically, any generated structure either counts towards the novelty or the quality regime of the metric. By penalizing all structures, the ones with good novelty (i.e., no memorization) necessarily reduce the quality coverage. As written above, we argue that \emph{novelty} and \emph{quality coverage} are distinct tasks that require different featurizers. In particular, judging whether two materials are the same requires close consideration and detailed verification of whether they have the same atom types at similar positions. On the other hand, judging how good quality is covered involves other factors, such as the structure's stability. To account for this, we introduce a second featurizer extending the TNovD \cite{hagemann2025transportnoveltydistancedistributional} to better respect stability and evaluate quality coverage individually. Note that when talking about stability in this work, we refer to \emph{implicit} stability: stability-related information encoded in the hidden MACE features introduced below, rather than an explicit thermodynamic quantity such as energy above hull. Our perturbation experiments and correlations with polymorph formation energies in Appendix \ref{app:info_mace} indeed suggest that they are a useful proxy. Direct stability measures, such as energy above hull, are complementary but often prohibitively costly. 

\paragraph{The "Fine" in CFTD: The Identity Featurizer}
In ref. \cite{hagemann2025transportnoveltydistancedistributional}, the featurizer $F_{id}$ is a GNN trained with an InfoNCE loss for contrastive learning. This is suitable for novelty detection, but does not necessarily encode a good means of quality coverage. 

More precisely, we used the InfoNCE loss \cite{oord2019representationlearningcontrastivepredictive}, which maximizes the similarity between pairs $(z^1_i, z^2_i) = (F_{id}(x_i^1), F_{id}(x_i^2))$ by minimizing the loss 
$$L(\theta) =- \sum_{i=1}^n \log\left(\frac{\exp(z_i^1 z_i^2)}{\sum_{j \in N_i} \exp(z_i^1 z_j^2) } \right), $$
where $N_i$ is the set of indices without $z_i^1$. We provide some justification for using the InfoNCE loss in section \ref{sec:infonce}.

For the similarity calculation, we augment a structure and treat the original structure and its augmentation as a pair. We thereby need to specify the $\mathrm{augment}$-function, which returns these \emph{positive} pairs for the contrastive procedure. In \cite{hagemann2025transportnoveltydistancedistributional}, we used rotation and translation. However, to make the identity featurizer $F_{id}$ more chemically aware and more robust towards noise, we now added a small amount of random zero-mean Gaussian noise with a standard deviation of $0.01$ to the fractional coordinates, so that small differences in atomic positions between generated samples and training materials are not penalized as heavily. Additionally, we expanded the augmentation by substitution with similar elements, replacing each site of an element with either its nearest or its next-nearest neighbors on the modified Pettifor scale with a certain probability (0.1 and 0.05, respectively)\cite{Glawe_2016}. Pettifor neighbors are often chemically similar atom types that have a high probability of substituting for one another in a crystal structure. We therefore expect similar bonding properties and consider the adapted structure to be a comparable material, or at least a material for which we do not require a generative artificial intelligence model to discover it. This is also relevant for the consideration of substitutional disorder in crystal structures \cite{jakob2026}, which is often neglected in generative models for crystals. The modification will therefore penalize generative models that mostly create new crystal structures by substituting supercells of well-known crystal structures with low concentrations of chemically similar elements. 
We note that the augmentation can be flexibly adapted to the specific definition of identical materials, or based on the materials discovery use case.

\paragraph{The "Coarse" in CFTD: The MACE Featurizer}
The featurizer in TNovD \cite{hagemann2025transportnoveltydistancedistributional} is trained with a contrastive loss function determining whether two structures are identical. However, there is no reason to expect it to directly capture a structure's quality parameters.

For this, we implemented a second featurizer $F_{mace}$, which is directly derived from MACE-MP-0b3, a machine learning interatomic potential (MLIP) foundation model \cite{Batatia2022mace, Batatia2022Design, batatia2025foundationmodelatomisticmaterials} that captures stability information. 

To be more specific: $F_{mace}$ does not consider the final output of MACE, i.e., forces or energy, but rather statistics of its layers, in a similar fashion to FID \cite{fid_paper}. Technically speaking, we consider the invariant part of the pooled "node features" of MACE-MP-0b3 "medium": Let $h^k = (h^k_1,...,h^k_{n_{atoms}})$ be the hidden states of its $k$-th interaction block, which are of size $(n_{atoms}, h_{dim})$. Within MACE, the invariant part of these hidden features is used to predict a hierarchical decomposition of site energies within the read-out phase. We apply mean pooling to obtain features of the same size, regardless of the number of atoms in a structure, and then concatenate the resulting layers. This yields $H = (\mathrm{mean}(h^1),...,\mathrm{mean}(h^L))$. 

To create the final feature space, we use a random projection by considering a fixed set of projections $(\xi_1,...,\xi_{n_{proj}})$ with $\Vert \xi_i \Vert = 1$ and Gaussian initialization. From this, we then obtain the map $F_{mace}(x) = \xi H(x)$, where $x$ is the material and $\xi \in \mathbb{R}^{n_{proj}, dim(H)}$. This operation allows us to reduce the dimensionality of the pooled MACE features, while keeping all information intact. The result is an $n_{proj}$-dimensional representation for the coarse feature space, which encodes quality information of the structure. 
We verify its coarse character and the contained information in Appendix \ref{app:coarse_mace} and \ref{app:stability-mace} by reconstructing the space-group, lattice, and composition of the test set, as well as the formation energies of polymorphs from it. 
None of those parameters can be reconstructed perfectly, underlining the coarse nature of the final MACE feature space. Composition is recovered best, while lattice and space-group type are captured less accurately. We further find that MACE feature distances are highly correlated with polymorph energy distances. Separately, on random pairs from the training set, MACE cosine similarities differ systematically between pairs of (meta)stable structures and pairs of structures, for which the fractional coordinates where altered by inducing random mean Gaussian noise. This finding shows sensitivity to local structural perturbations. Both findings demonstrate that MACE features carry implicit stability information.
We also discuss the geometries of the feature spaces invoked by the coarse MACE-featurizer (for quality coverage) and the fine GNN-based featurizer (for novelty) in Appendix \ref{app:fine_coarse} and show that both geometries differ from each other. 

\begin{algorithm}[t]
\caption{Coarse-Fine Transport Distance}
\label{alg:cftd}
\small

\textbf{Input:} Empirical measures
$\mu=\frac{1}{n}\sum_{i=1}^n\delta_{x_i}$ and
$\nu=\frac{1}{m}\sum_{j=1}^m\delta_{g_j}$,
train and validation calibration sets, featurizers $F_{id},F_{mace}$, parameter $\beta$.

\textbf{Step 1: Calibration.}
\begin{enumerate}
    \item Compute the optimal coupling (i.e. the optimal transport plan) between the train and validation sets.
    \item Choose $\tau_{mem}$ and $\tau_{qual}$ so that $\frac{1-\beta}{2}$ of the coupled mass lies in the memorization and low-quality regimes, respectively. Choose $M$ so that both error terms are balanced; see Appendix~\ref{app:hyp_sel}.
\end{enumerate}

\textbf{Step 2: Evaluation.}
\[
C_{i,j}^{id}=\lVert F_{id}(x_i)-F_{id}(g_j)\rVert,\quad
C_{i,j}^{mace}=\lVert F_{mace}(x_i)-F_{mace}(g_j)\rVert,\quad
C_{i,j}=\tfrac12 C_{i,j}^{id}+\tfrac12 C_{i,j}^{mace}.
\]
Calculate the optimal transport plan $\pi(\mu,\nu)$ for train and generated samples with respect to the pairwise distances $C_{i,j}$. Define
\[
\begin{aligned}
\mathrm{CFTD}(\mu,\nu)
&=
\sum_{i=1}^n \sum_{j=1}^m
\pi_{i,j}(\mu,\nu)
\bigl[
M\,\mathrm{ReLU}(\tau_{mem}-C^{id}_{i,j})
\\
&\qquad\qquad
+
\mathrm{ReLU}(C^{mace}_{i,j}-\tau_{qual})
\bigr].
\end{aligned}
\]
\end{algorithm}

\paragraph{Coarse-Fine Transport Distance}

Given training samples $(x_i)_{i=1}^n$ and generated samples $(g_i)_{i=1}^m$, we developed an algorithm similar to the Transport Novelty Distance. In particular, we implement a "Goldilocks" zone separating novelty and quality coverage into two distinct regimes. This zone defines the percentage of structures that are not penalized in a reasonable train split. We call this proportion $\beta$. We then proceed with the following steps. We describe the algorithm in \ref{alg:cftd}, where we first determine the two thresholds for which we penalize memorized or low quality data, $\tau_{mem}$ and $\tau_{qual}$ and choose a scale $M$ that acts as a tuning knob balancing the two parts. We then calculate the optimal transport plan for a weighted average of both featurizers and use it to penalize coupled generated and training set pairs by assigning a cost to generated samples that are either too close in the identity feature space (below $\tau_{mem}$) or too far away in the MACE feature space (above $\tau_{qual}$).
This procedure allows us to use different criteria to detect memorization (fine) and low quality (coarse) in the generated samples, resolving the novelty-quality coverage tradeoff. 

At heart, the metric remains distributional, but the optimal transport coupling now accounts for two distinct featurizers: novelty and quality coverage. Informally, quality coverage ($\mathrm{Cov}$) and novelty ($\mathrm{Nov}$) approach 1 as the CFTD vanishes. We state this here informally, with the full statement in the Appendix \ref{thm:cftd}. This means that $\mathrm{CFTD}$ penalizes mass outside the Goldilocks zone.

\begin{proposition}
    For suitable radii $r_{\mathrm{qual}}>\tau_{\mathrm{qual}}$ and $r_{\mathrm{mem}}<\tau_{\mathrm{mem}}$, we have in the feature space that $\mathrm{Nov}_{r_{mem}} \rightarrow 1$ and $\mathrm{Cov}_{r_{qual}} \rightarrow 1$ as the $\mathrm{CFTD}\rightarrow 0$.
\end{proposition}

\section{CFTD experiments}

\paragraph{Quality Coverage-Novelty Checks}
As a first sanity check for the new metric, we repeat some experiments from \cite{hagemann2025transportnoveltydistancedistributional}. In particular, we want to show that the CFTD can indeed detect issues of generative models with regard to novelty and quality coverage. Our experiments in \cite{hagemann2025transportnoveltydistancedistributional} revealed that unstable materials were not that heavily penalized by the contrastive GNN, since even heavy structural distortions only resulted in small increases of the final TNovD. To test this for the new metric, we strain the lattice, add noise to the fractional coordinates, substitute atomic sites with neighbors on the modified Pettifor scale, elements of the same group, and random elements, and iteratively replace an increasing number of test-set data points with training data. 
\begin{figure}[!htpb]
    \centering
    \begin{subfigure}{0.48\textwidth}
        \includegraphics[width=\linewidth]{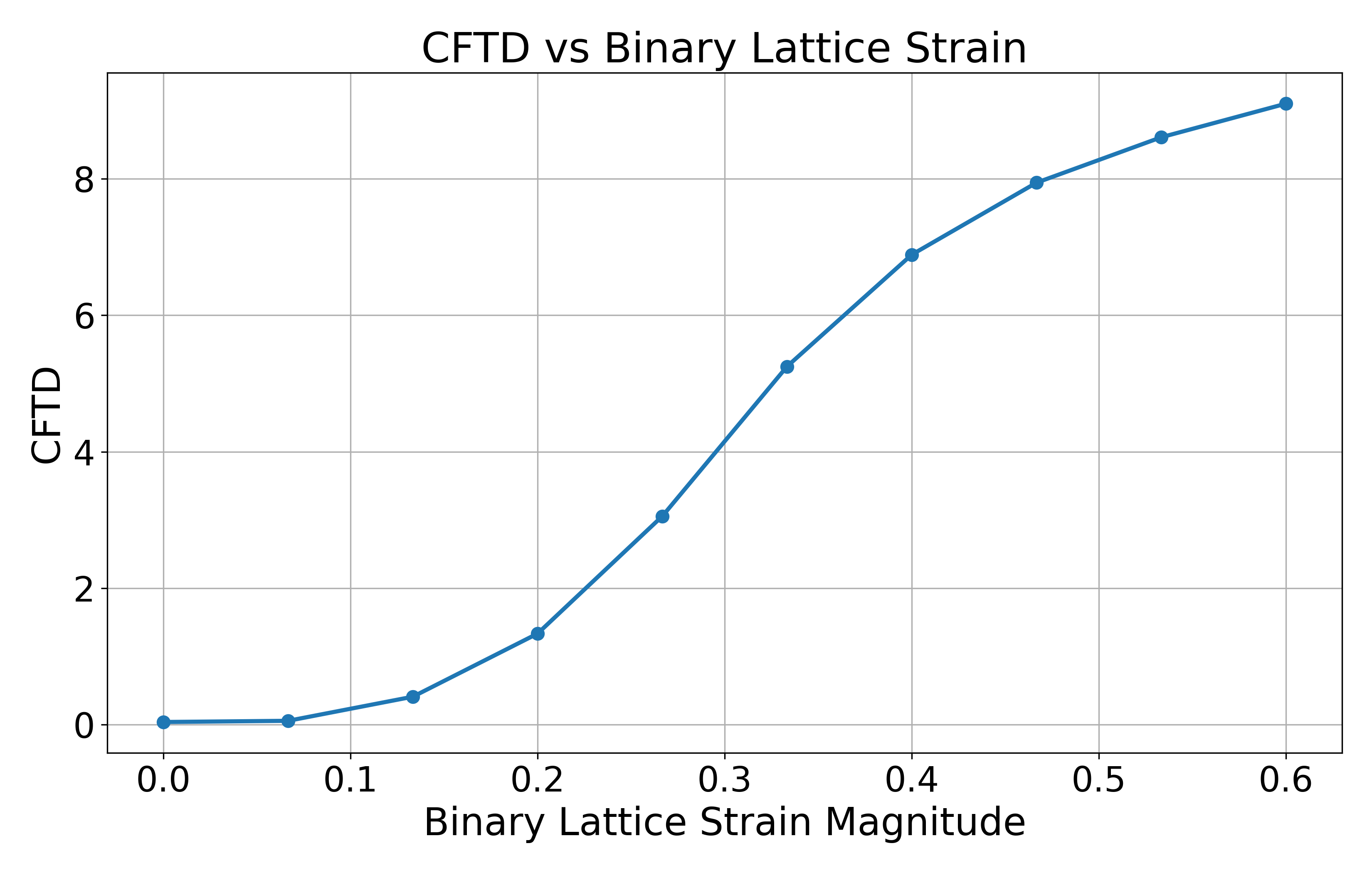}
        \caption{Lattice strain.}
    \end{subfigure}
    \hfill
    \begin{subfigure}{0.48\textwidth}
        \includegraphics[width=\linewidth]{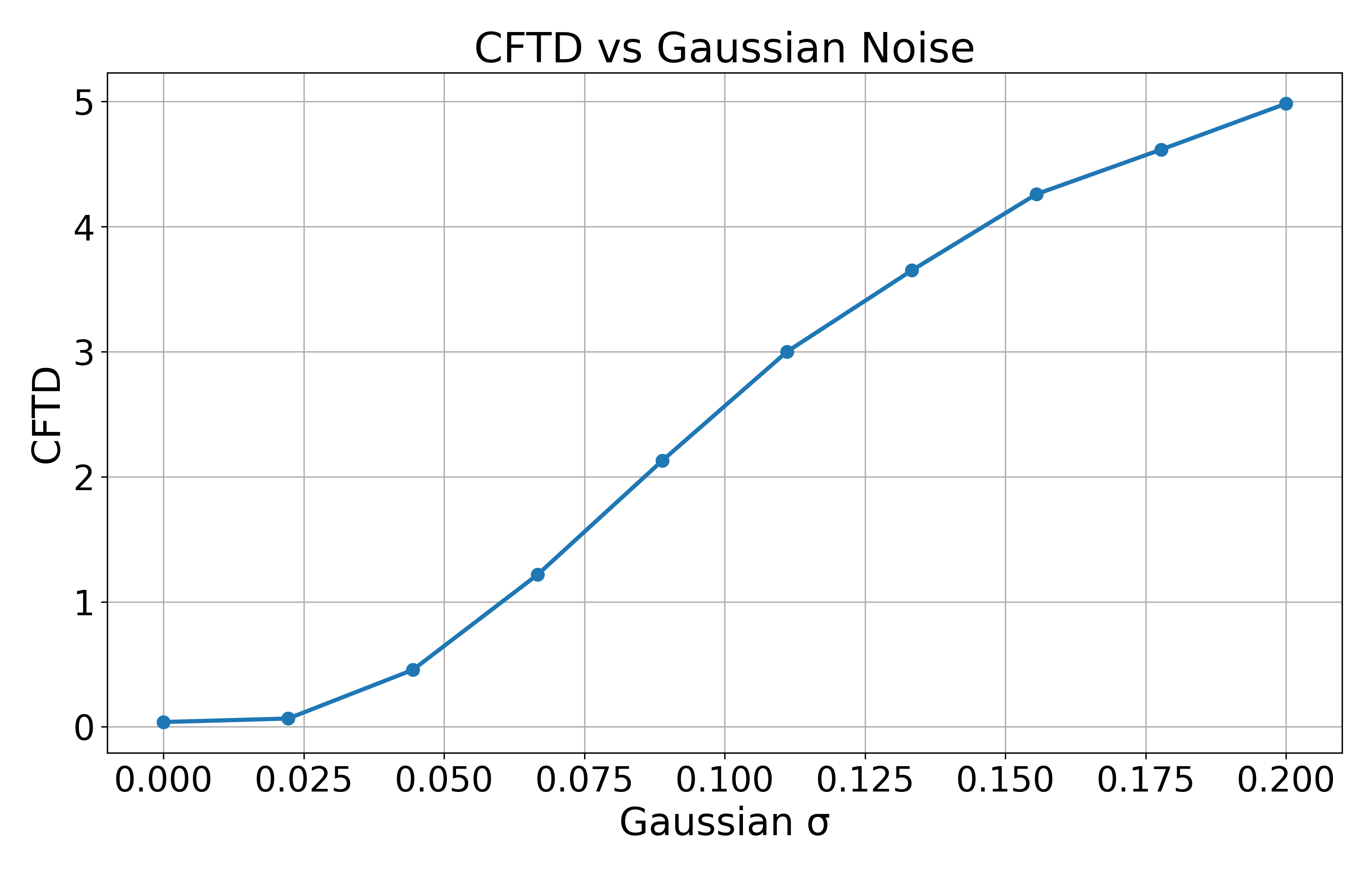}
        \caption{Gaussian noise.}
    \end{subfigure}
    \hfill
    \begin{subfigure}{0.48\textwidth}
        \includegraphics[width=\linewidth]{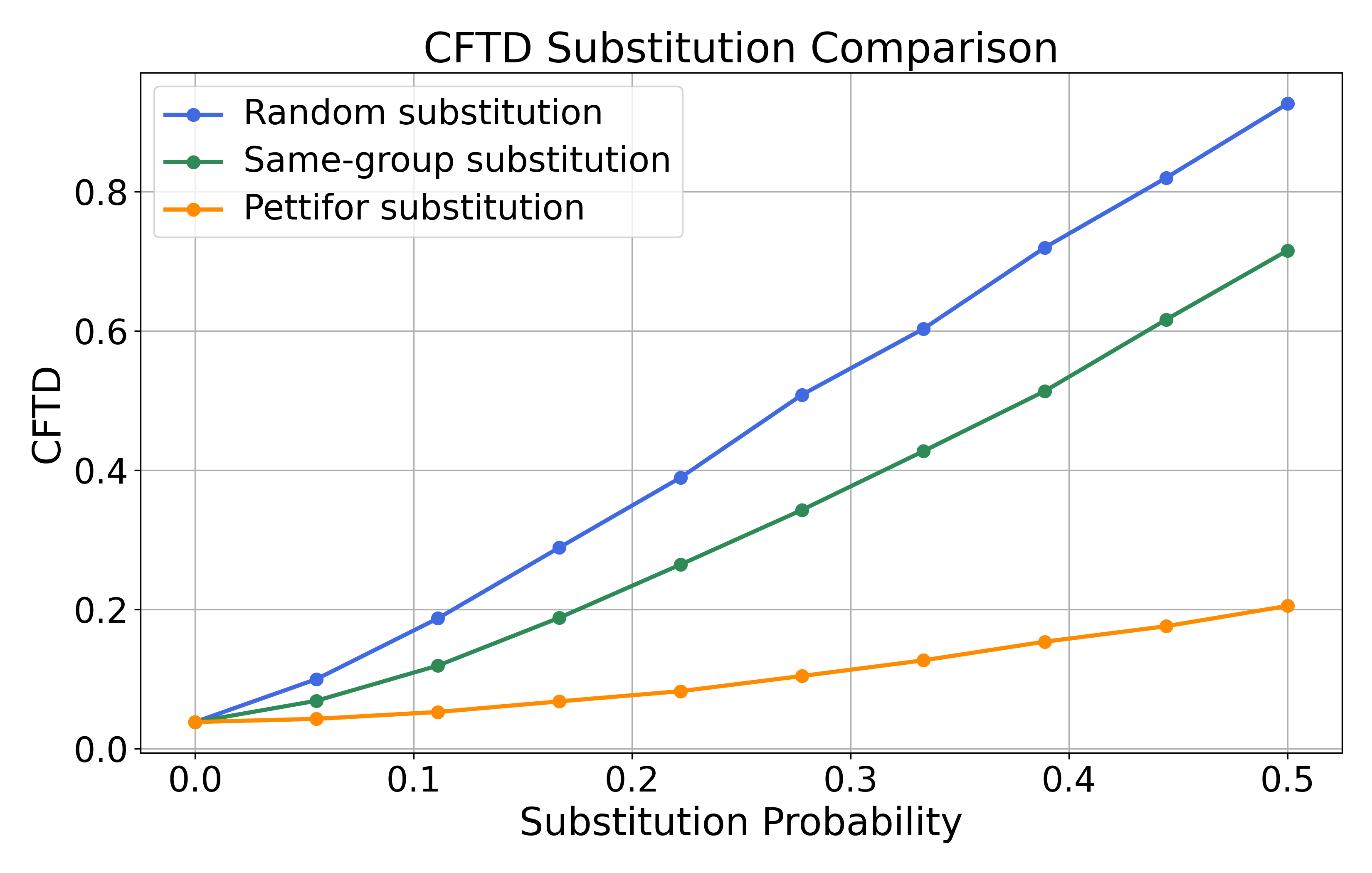}
        \caption{Substitutions.}
    \end{subfigure}
    \hfill
    \begin{subfigure}{0.48\textwidth}
        \includegraphics[width=\linewidth]{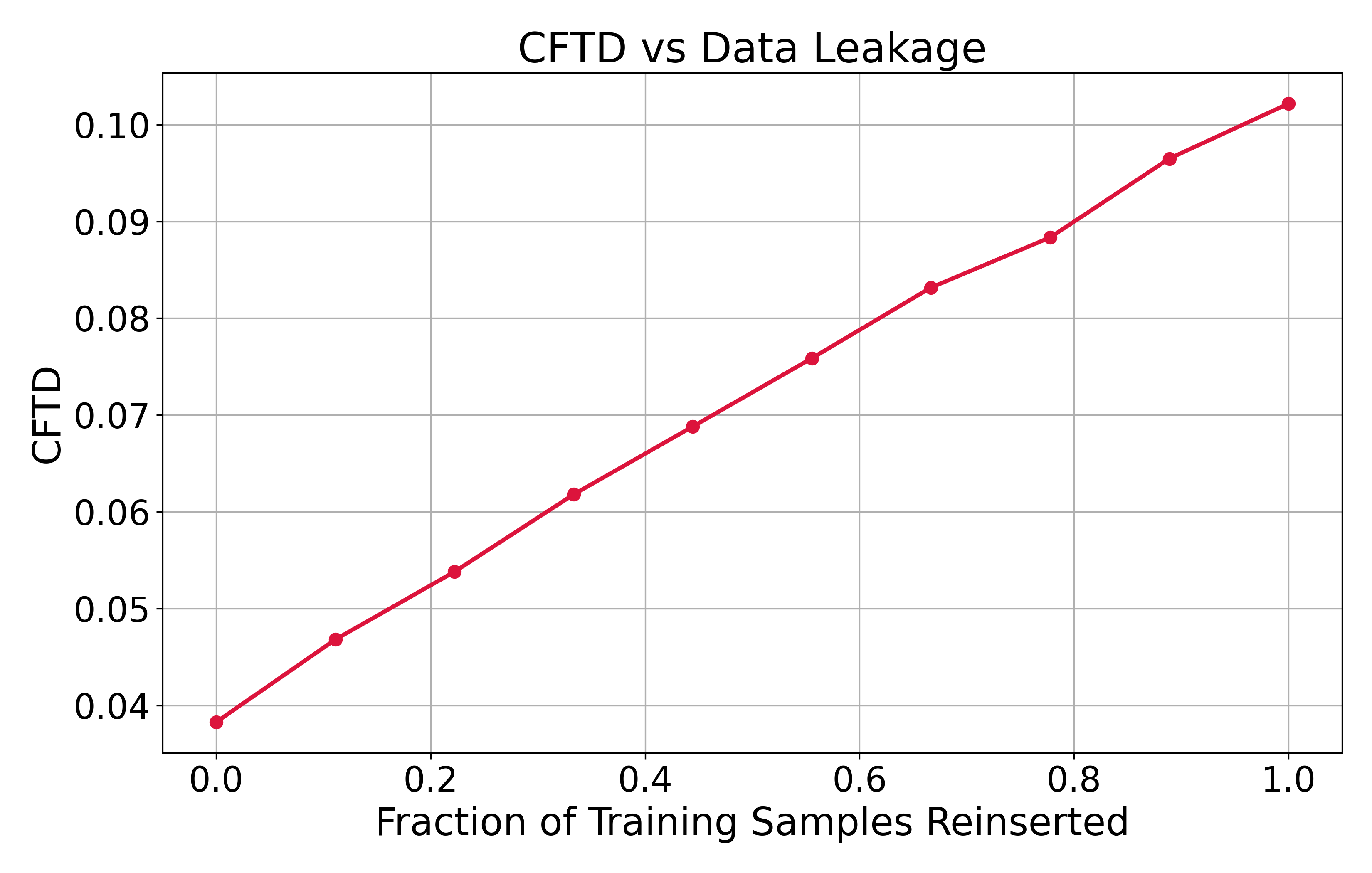}
        \caption{Train leakage.}
    \end{subfigure}

    \caption{CFTD behavior of the MP20 train and test set across different perturbation settings of the test set.}
    \label{fig:CFTD_perts}
\end{figure}

The behavior of these tests is shown in Fig.\ref{fig:CFTD_perts}. We observe that Gaussian noise and lattice strain both increase the CFTD heavily. This reflects the necessary behavior of our metric, since both experiments create structures with increasingly poorer quality. Fig.\ref{fig:CFTD_perts} (d) shows the results for iteratively replacing more and more test set structures with structures from the train set, displaying a monotone increase in the CFTD. Since this experiment simulates the memorization of structures, this is the expected result. Interestingly, the metric is very sensitive to atom substitutions (Fig.\ref{fig:CFTD_perts} (c)). Random substitutions are penalized more strongly than in-group substitutions, which, in turn, are more strongly penalized than modified Pettifor substitutions. This is reasonable, as structures created with Pettifor substitutions are expected to be higher in quality compared to structures created with in-group or random substitutions. This gain in quality compensates the effect of increased memorization in structures with Pettifor substitution, which comes from the inclusion of modified Pettifor neighbors in the augment function of the contrastive GNN. The Pettifor substitutions are penalized more strongly in the memorization component, which is not seen in the final CFTD value of Fig.\ref{fig:CFTD_perts} (c), due to the higher quality of the structures. This shows that with this change, the CFTD is able to detect chemical similarity very well. We also remark that, in contrast to TNovD (see the same Figure in \cite{hagemann2025transportnoveltydistancedistributional}), our metric is able to capture chemical differences in a distinctive way.
Further, our approach seems to penalize Gaussian noise and lattice strain more heavily than TNovD, indicating a better treatment of stability. 

Although the experiments yield highly divergent CFTD values that prevent a straightforward interpretation, the results in Fig.\ref{fig:CFTD_perts} prove that problematic structures and data leakage are indeed punished by our metric. We also repeat the same experiment with the Perov-5 dataset \cite{castelli_perov5, castelli_perov52, xie2022crystal} in the Appendix \ref{sec:perov_5}, where similar conclusions are drawn.

\paragraph{Comparison with continuous SUN}
Since our CFTD bears similarities to the continuous SUN metrics \cite{negishi2026continuoussunstableunique}, we propose a small toy experiment to showcase their differences. We use their implementation in xtalmet \cite{negishi2025continuous} and create three subsets of 1,000 structures each. Two subsets were created by randomly drawing 1000 structures from the MP20 train and validation set, respectively. For the third subset, only structures from the validation set with fewer than ten atoms were considered. 
In the toy experiment, the first subset was treated as the training set, while the other two simulated the generated structures. With this, CFTD and cSUN were calculated for the second and the third subset.
This experiment simulates a strong distributional shift, since the third subset covers only a specific part of the full distribution of structures (only those with at most ten atoms). 
As shown in Fig. \ref{fig:sunvscftd} (a), cSUN values barely differ between subsets two and three. This is not surprising, since, at an instance level, there is no reason to assume that materials with fewer than ten atoms are less unique or novel than those with more atoms. Nevertheless, if a model trained on the complete dataset generates only structures with fewer than 10 atoms, it would be considered mode-collapsed. The CFTD penalizes memorization and quality by optimizing the transport plan between the training and generated distributions. It is therefore primed to detect the artificially induced distributional shift, and, as indicated by the higher CFTD values in Fig. \ref{fig:sunvscftd} (a), it indeed does so. This shows that CFTD complements existing metrics by clearly detecting distributional shifts between training and generated samples. Since the training objective of generative models is to create samples that fit the training distribution without exactly mirroring it, this helps to identify potentially problematic models that score well in other metrics. To simulate a different kind of distributional shift, we perform a second toy experiment by creating a subset of 1,000 randomly drawn structures from the validation set that do not contain oxygen. As shown by the results in Fig. \ref{fig:sunvscftd} (b), we observe the same behavior.

\begin{figure}[!htpb]
    \centering    
    \begin{subfigure}{0.48\linewidth}
        \centering
        \includegraphics[width=\linewidth]{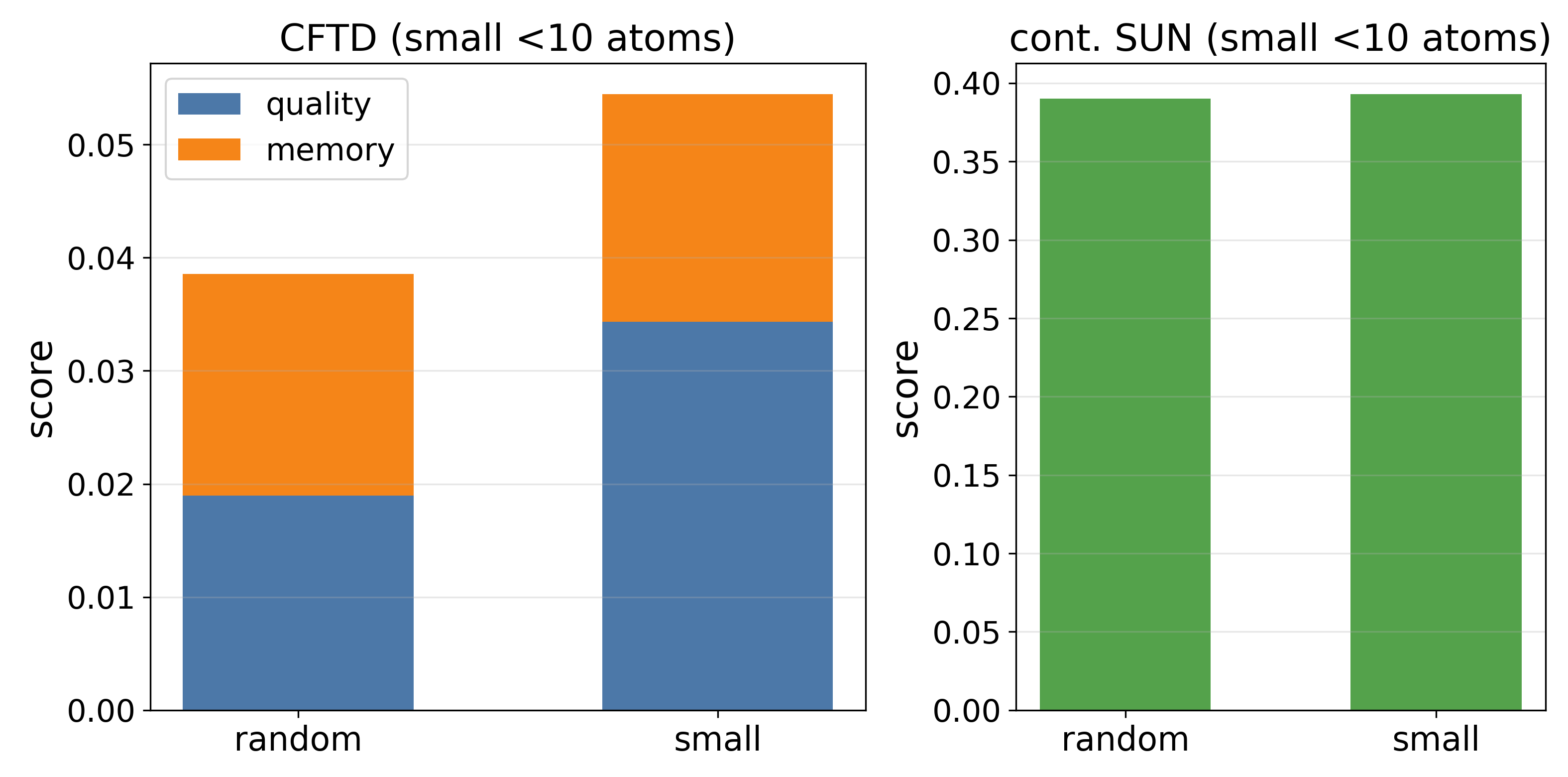}
        \caption{Comparison of a reference set versus a subset with less than 10 atoms.}
    \end{subfigure}
    \hfill
    \begin{subfigure}{0.48\linewidth}
        \centering
        \includegraphics[width=\linewidth]{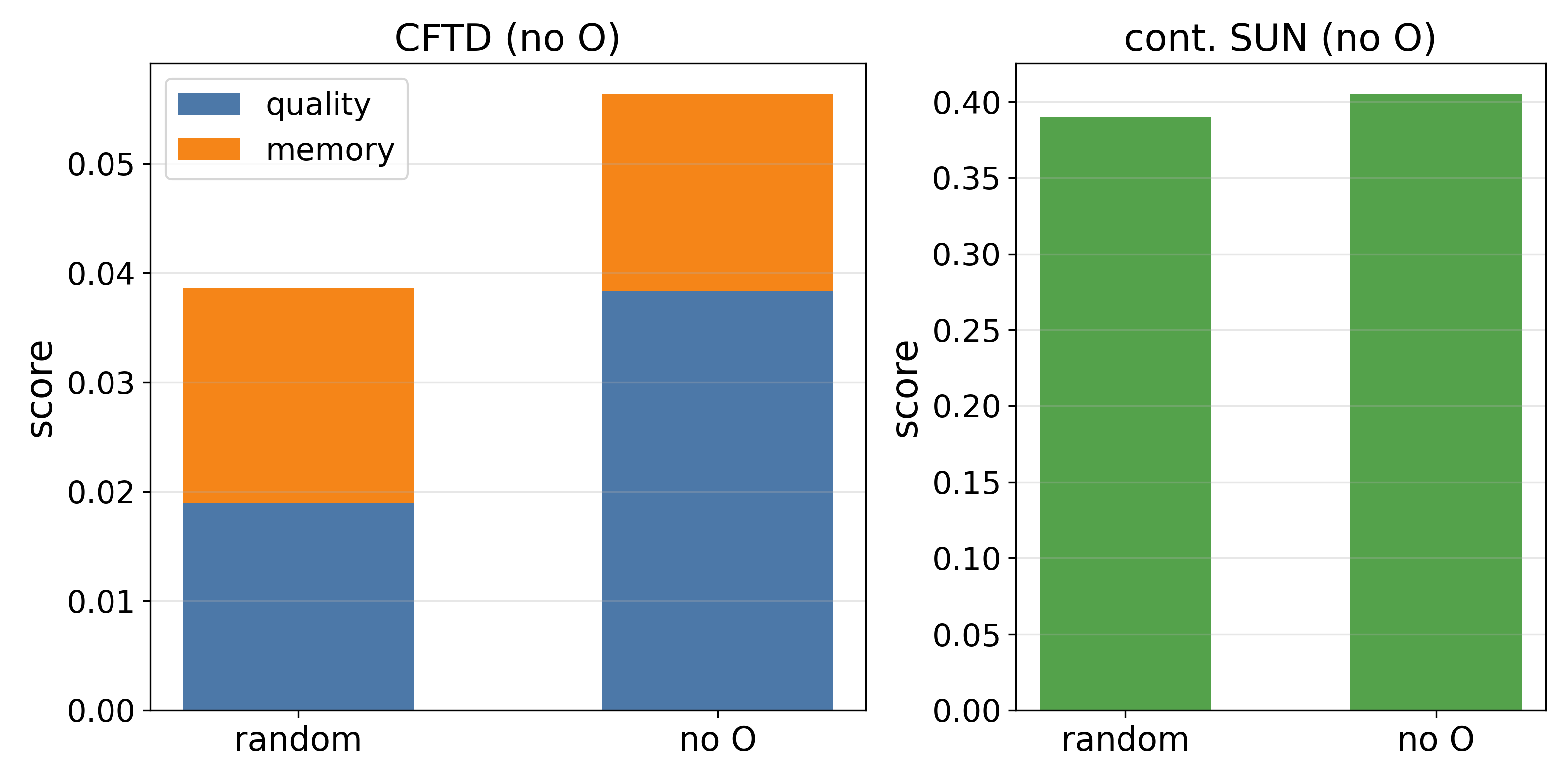}
        \caption{Comparison of a reference set versus a subset without Oxygen.}
    \end{subfigure}
    
    \caption{Comparison of continuous SUN versus CFTD.}
    \label{fig:sunvscftd}
\end{figure}

\begin{figure}[!htpb]
    \centering
    \begin{subfigure}{0.45\textwidth}
        \centering
        \includegraphics[width=\linewidth, scale = 0.8]{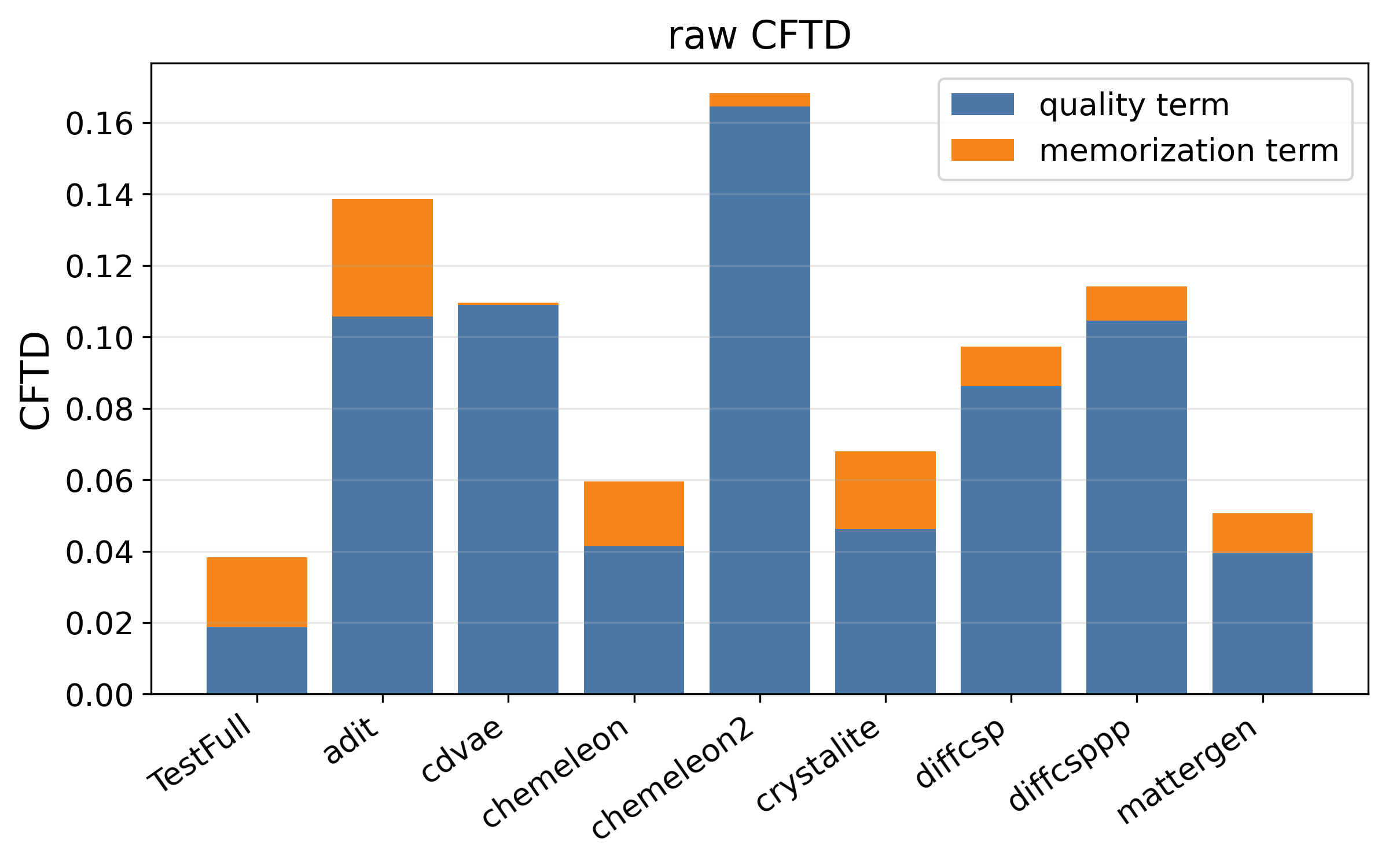}
        \caption{CFTD for unrelaxed structures from different generative models.}
    \end{subfigure}
    \hfill
    \begin{subfigure}{0.45\textwidth}
        \centering
        \includegraphics[width=\linewidth, scale = 0.8]{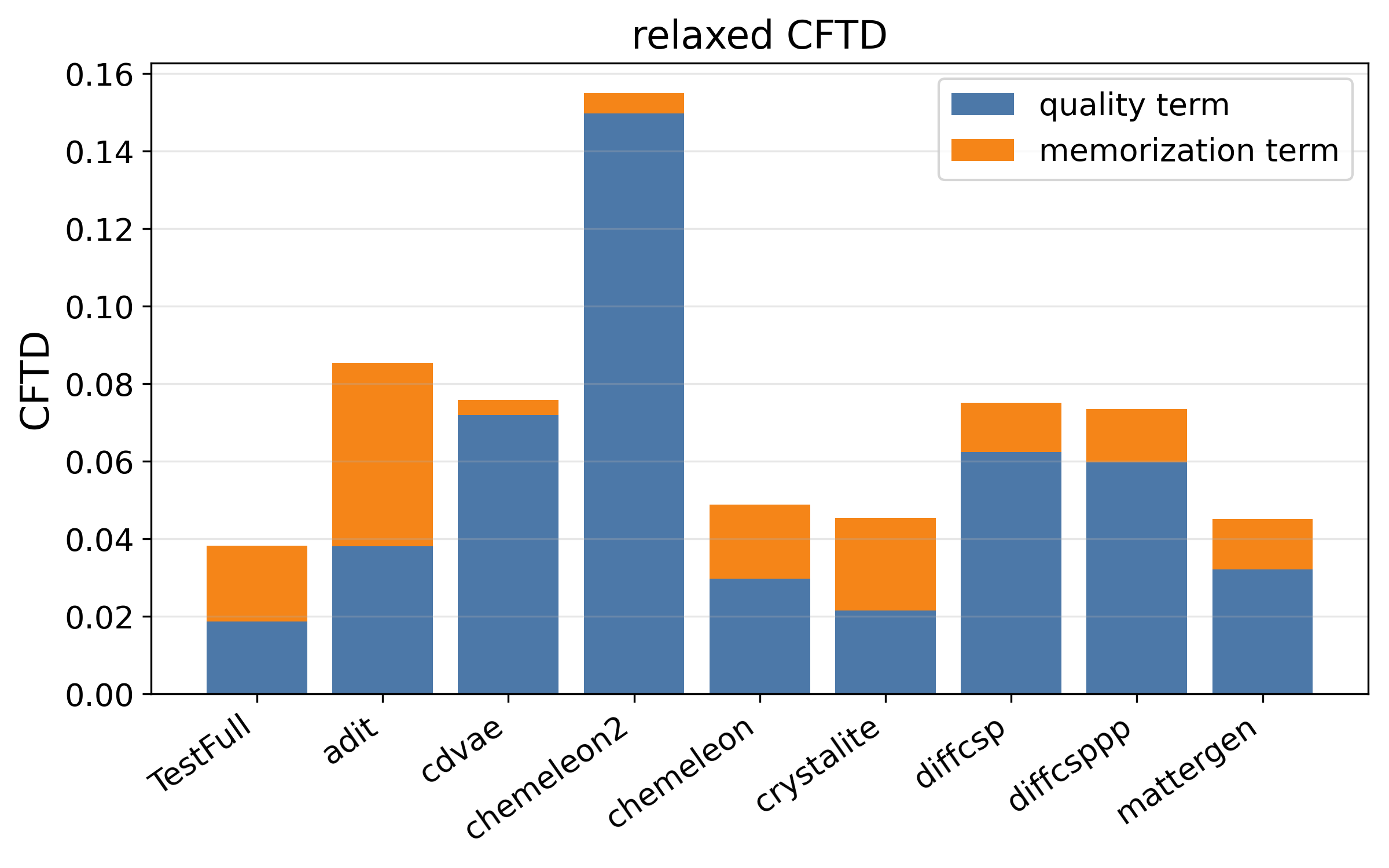}
        \caption{CFTD for relaxed structures from different generative models.}
    \end{subfigure}
\caption{Comparison of common material generative models: relaxed and unrelaxed.}
\label{fig:models_cftd}
\end{figure}

\subsection{Evaluation of current material generators}

We also evaluate current material generators with samples provided in \cite{negishi2026continuoussunstableunique}. This comes in handy, because our proposed CFTD bears many similarities with these continuous stability, uniqueness, and novelty metrics. Additionally, the same MP20 dataset was used to train the generators and our contrastive GNN, enabling us to use the generated structures directly without retraining the identity featurizer.  

We follow our proposed hyperparameter fitting procedure on the MP20 train and validation sets and provide evaluations of the materials generative models MatterGen \cite{MatterGen2025}, ADiT \cite{joshi2025allatomdiffusiontransformersunified}, Chemeleon \cite{llmcrystalgeneration}, Chemeleon2 \cite{park2025guidinggenerativemodelsuncover}, CdVAE \cite{xie2022crystal}, DiffCSP \cite{jiao2024crystalstructurepredictionjoint}, DiffCSP++ \cite{jiao2024spacegroupconstrainedcrystal}, Crystalite \cite{crystalite}, as well as the MP20 test set.

To investigate the influence of relaxation (i.e., a structural optimization of the atomic positions and unit cell parameters), we relaxed all generated structures with the MACE-MP-0b3 foundation model (model: "medium") for 100 steps and 0.02~\text{\AA}. In general, the relaxed structures from the generative models achieve lower CFTD values (i.e., can be considered better) but also exhibit higher memorization. 

As shown in Fig. \ref{fig:models_cftd}, our metric consistently ranks the MP20 test set structures as the best, with MatterGen following. The Chemeleon and Crystalite models also perform well, albeit with a higher memorization term than MatterGen. In particular, Crystalite and MatterGen are very close in the relaxed case and mainly differ in their novelty-quality coverage trade-off. 
 
The biggest difference from \cite{negishi2026continuoussunstableunique} seems to be the performance of Chemeleon2. It is ranked highest in \cite{negishi2026continuoussunstableunique} but scores the lowest on our metric. As a reinforcement-learning-based model, it deliberately overweights rare events, thereby shifting the generated distribution away from the training distribution. This is consistent with analysis in \cite{negishi2026continuoussunstableunique}, where they write that Chemeleon2 employs "reward hacking" under their CSUN. We hypothesize that this is caused by overconcentration, which is penalized by a distributional metric like CFTD, but mostly ignored by an \emph{instance-level} one.

\section{Guiding a material generator using MLIP features}
In addition to evaluating generative models, we wanted to test whether the coarse MACE features carry a usable signal for generation itself and therefore designed a MACE-conditional generator. The aim of this section is not to propose a SOTA generator, but to verify two things: that a generator can use the MACE features, and that it faithfully obeys them as a condition, i.e., generates materials that respect the conditional input. Both of these are necessary conditions for a \emph{hierarchical} generator, which we briefly describe in Appendix \ref{app:ext}. 

In the following, we describe a material as a tuple $(n, (a_i)_{i=1}^n, L, (x_i)_{i=1}^n, h)$, where $n$ describes the number of atoms, $a_i$ describes the (categorical) atom types, $L$ the lattice, and $(x_i)$ the fractional coordinates of atom positions. Compared with existing generative models, we include an additional parameter $h$, which corresponds to our coarse mace features.  
We choose the factorization $$p(n,h,L,a,x) = p(n,h)p(a|n,h) p(L|n,h,a)p(x|a,n,h,L),$$
which says that based on the MACE features and the number of atoms, we generate the atom types, then the lattice, and then the positions. 
In particular, we learn all the conditional probabilities $p(a|n,h)$, $p(L|a,n,h)$ and $p(x|L,a, n, h)$ using the means of flow matching \cite{albergo2023building,lipman2023flow, liu2023flow, gat2024discrete}. Given the variable type, we use continuous-flow matching for the position and lattice, and discrete-flow matching for the atom types. 

For simplicity, we choose $p(n,h) =\frac{1}{N}\sum_{i=1}^N \delta_{(n_i,h_i)}$, i.e., the empirical distribution. The other components are learned as follows:

\begin{itemize}
    \item \emph{Atom types: } We use a transformer trained with a log-likelihood objective. By applying a discrete flow matching-style approach \cite{gat2024discrete}, we model each atom type as a categorical value with a masking path. We sample using the substitution-based parameterization (SUBS) approach from \cite{sahoo2024simple}, which discretizes unmasking based on the learned probabilities. 
    \item \emph{Lattice: }We follow the parametrization used in \cite{jiao2024spacegroupconstrainedcrystal}, where they essentially choose a basis of the set of symmetric matrices, and use Euclidean flow matching in the space of this basis. 
    \item \emph{Atom positions :} We adopt an adapted version of an equivariant GNN \cite{pmlr-v139-satorras21a} and modify it into a non-equivariant architecture. We do this by adding directional and moment information \cite{Novikov_2021} that is passed through a multilayer perceptron (MLPs) and FiLM-style conditioning \cite{perez2018film}. The geometry chosen is flow matching on fractional spaces, similar to \cite{wu2026dmflow, sriram2024flowllm}. 
\end{itemize}
A detailed description can be found in Appendix \ref{sec:model}. Our modeling ideas build upon recent works of material generators.

\subsection{Evaluation}
\paragraph{Overfitting on a small subset}

To judge both CFTD and the consistency of our generator, we first train our model on a small subset of 2,000 materials from the MP20 dataset \cite{xie2022crystal}. With so few samples for training, we expect our model to memorize to some degree. We train it for 16,000 epochs and validate the generated structures, both relaxed and unrelaxed, every 2,000 steps using the CFTD metric. The CFTD allows an insightful comparison of quality coverage and novelty, which we visualize in Fig. \ref{fig:macematgen}. We see that both unrelaxed (Fig. \ref{fig:macematgen} (a)) and relaxed (Fig. \ref{fig:macematgen} (b)) runs have a minimal CFTD at epoch 2000. Memorization increases much faster for relaxed structures, which we trace back to imperfections in the positional generator. The generated fractional coordinates are sometimes close to, but not at the exact equilibrium positions. These small positional offsets can prevent generated structures that are identical copies of structures in the train set from being identified as such. By applying the MACE relaxation, the atoms are moved to equilibrium positions and, as a result, the affected structures are then correctly identified as memorized. These results show a vital advantage of CFTD for evaluating the training of material generative models: it can be easily combined with relaxation procedures, unlike, e.g., the validation loss. Increased memorization is a clear indicator of overfitting, and since this might only be visible after relaxation (as seen by comparing Fig. \ref{fig:macematgen} (a) and (b)), CFTD can identify such behavior more easily than other metrics. 

\begin{figure}[h!]
    \centering

    \begin{subfigure}{0.45\textwidth}
        \centering
        \includegraphics[width=\linewidth]{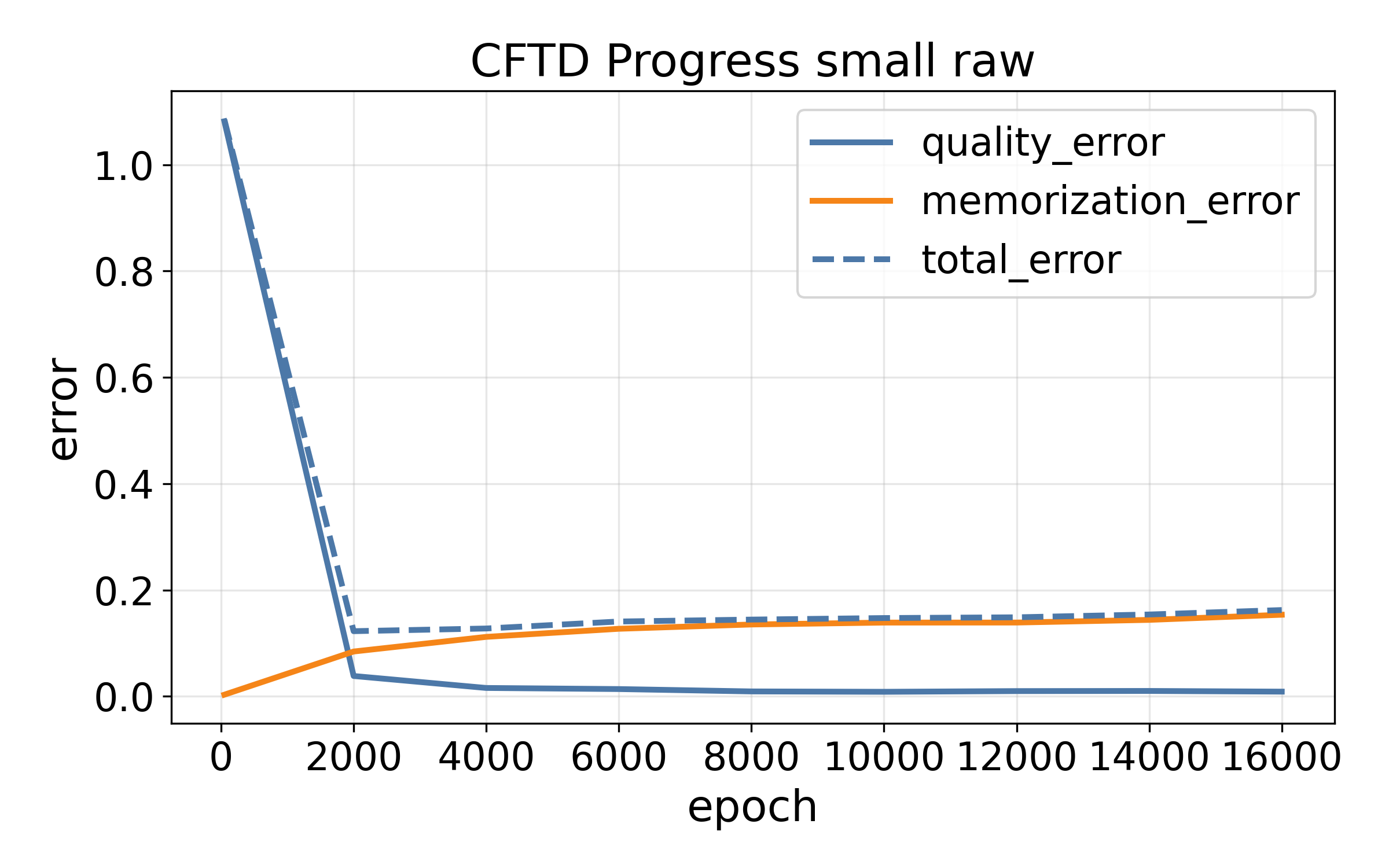}
        \caption{CFTD over epochs (unrelaxed).}
    \end{subfigure}
    \hfill
    \begin{subfigure}{0.45\textwidth}
        \centering
        \includegraphics[width=\linewidth]{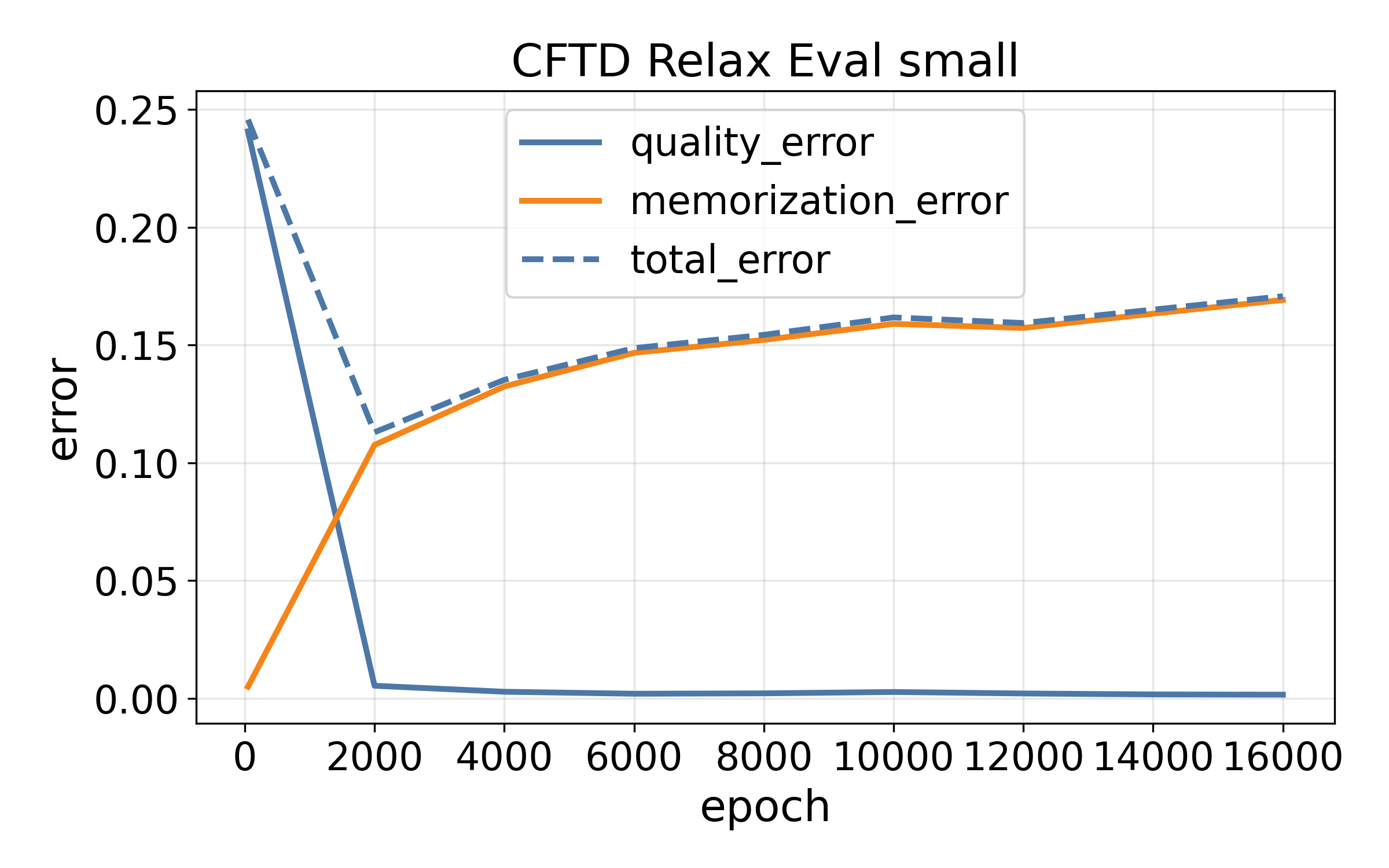}
        \caption{CFTD over epochs (relaxed).}
    \end{subfigure}
\caption{Comparison of our proposed material generator trained on an MP20 subset of 2,000 materials only. }
\label{fig:macematgen}
\end{figure}

\begin{figure}[h!]
    \centering

    \begin{subfigure}{0.49\textwidth}
        \centering
        \includegraphics[width=\linewidth]{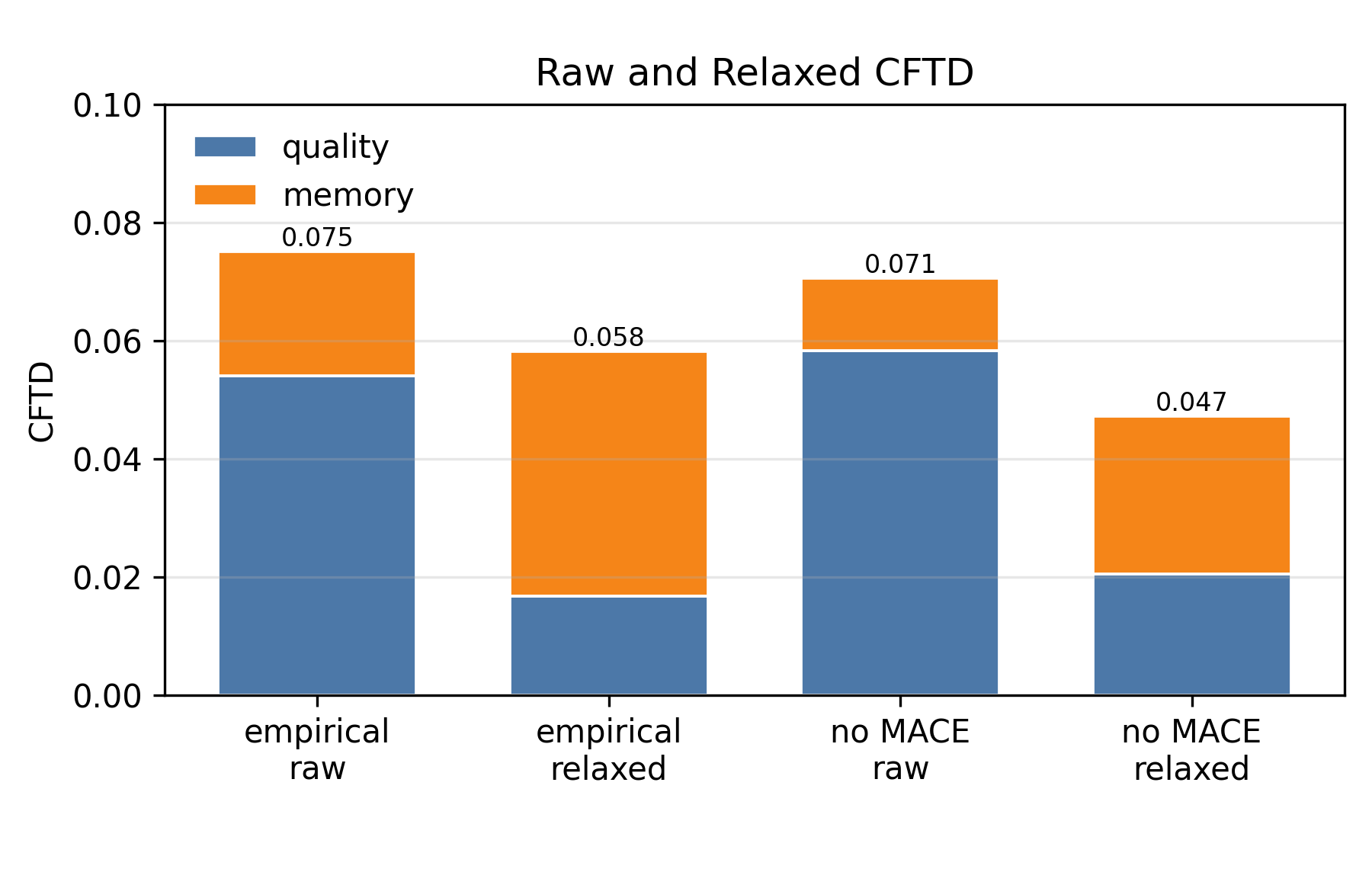}
        \caption{CFTD for the final models (empirical and none), raw and relaxed.}
    \end{subfigure}
    \hfill
    \begin{subfigure}{0.43\textwidth}
        \centering
        \includegraphics[width=\linewidth]{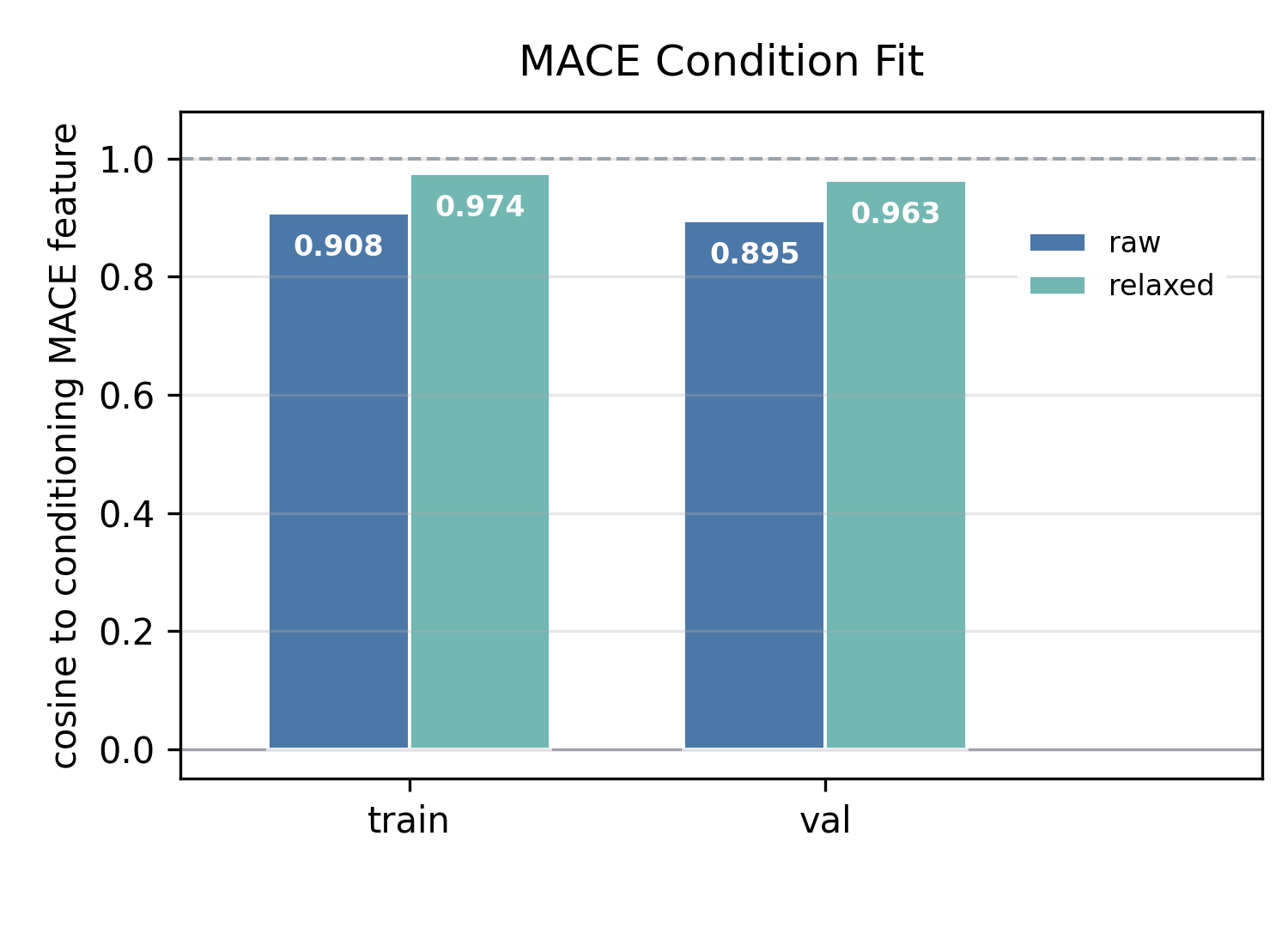}
        \caption{Cosine similarity for the empirical distribution between generated materials and condition.}
    \end{subfigure}
\caption{CFTD of models. Cosine similarities are for materials generated conditioned on empirical features, and they are tested to verify that they satisfy the condition.}
\label{fig:mace_cond_fit}
\end{figure}

\paragraph{Full Training}
To investigate whether coarse MACE features can be a useful guide for a material generative model, we tested whether adding empirical MACE features improves performance on LeMat-GenBench \cite{betala2026lematgenbenchunifiedevaluationframework}, i.e., whether the guidance provides a useful downstream signal. More precisely, we compare the MACE-conditional model from above against an otherwise identical model that does not use MACE features. In the MACE-conditional model, generation is initialized by sampling the pair ((n,h)), where (n) is the number of atoms in the unit cell and (h) is the corresponding pooled MACE feature vector projection. In the empirical variant considered here, we use the empirical distribution of the MP20 train set as the prior over these conditions, i.e., $(p(n,h)=\frac{1}{N}\sum_i \delta_{(n_i,h_i)})$. The remaining components of the generator then sample atom types, lattice, and fractional coordinates conditional on this sampled pair. Since the conditions are taken directly from the empirical distribution of the training set, one may expect increased memorization. This can potentially be improved by generating the MACE feature conditions as well, and we discuss possible updates of the generator in Appendix \ref{app:ext}. At the same time, because MACE features encode structural, chemical, and physical information, they may guide the generator toward improved (meta)stability. To investigate such potential upsides, we compare the conditional and unconditional models using both CFTD and LeMat-GenBench metrics. 
Furthermore, we check whether the generated materials indeed comply with the MACE feature conditioning. This means that when applying the MACE featurizer on the generated samples, we check whether the resulting features matches the condition we fed to the model. 

The results are shown in Fig. \ref{fig:mace_cond_fit}. In Fig. \ref{fig:mace_cond_fit} (a), we see that the empirical MACE-conditional model exhibits higher memorization but better quality coverage (i.e., a lower quality value). The cosine distances in Fig. \ref{fig:mace_cond_fit} (b) indicate a very high degree of agreement between the features of the generated materials and the conditions, as they barely change after relaxing the structures to their energetic ground state with MACE-MP-0b3. We trained both models for 1,000 epochs and saved the model with the lowest validation loss. 

On LeMat-GenBench, the empirical MACE-conditioned model improved validity and stability over the No-MACE baseline 
($V=98.5$ vs.\ $97.2$, $S=11.7$ vs.\ $7.7$, $M=51.6$ vs.\ $44.4$), while the No-MACE model attains slightly higher novelty 
($N=61.0$ vs.\ $51.2$). Both models have similar SUN and MSUN profiles; the empirical MACE-conditioned model has SUN $0.8$, No-MACE has $0.5$. For MSUN, the empirical MACE-conditioned model reaches $16.7$, No-MACE reaches $18.0$. This shows that the empirical MACE-conditioned model does precisely what it should: guide the models towards more stable but less novel samples (since we fed in exact train MACE features). This shows high potential for learning a novel MACE feature space that combines stability improvements with greater novelty. 

Note that while the CFTD evaluation of the MACE-conditioned model shares the same underlying MACE representation, the pre-relaxation and evaluation procedure of LeMat-GenBench does use a more thorough relaxation procedure with 3 different MLIPs. 

\section{Conclusions}
In this paper, we introduced a dual featurizer framework for evaluating material generative models, extending the existing TNovD. We related our featurizers to the more classical terms of coverage and novelty in the respective featurizer geometries. We showed that our resulting CFTD exhibits consistent behavior compared to TNovD and can detect mode collapse, indicating a more distributional flavor than the cSUN metrics. We also showed that our metric  detects overfitting, making it desirable as a model checkpointer. However, it needs to be investigated more closely how our CFTD compares to other possible checkpoint strategies, such as the validation loss or MSUN, as taken in \cite{crystalite}. 

We further investigated the utility of the MACE features beyond generative model evaluation. We show that they inform both composition and structure and can be used to guide a conditional material generative model. This opens up a very interesting line of research for a \emph{hierarchical} model that first generates coarse MACE features and then builds a generative model for structure creation on top of them. 
A further benefit of conditioning a generative model on MACE features is that it opens a route to property-conditional materials generation: a MACE feature generator could be steered towards regions in the MACE feature space corresponding to structures with the desired property, thereby obtaining a property-conditional generator without retraining a separate conditional model for each target.
For the \emph{hierarchical} model, one needs to investigate the learnability-identifiability tradeoff, i.e., investigate how well the choices we make can inform the generator and how easy it is to learn the MACE feature space. Other summary statistics might be viable, such as taking atom-level features or the max or standard deviation instead of pooling. This provides a new way of designing material generators, which we believe is a fruitful area for future research. We discuss potential extensions to conditional generation in Appendix \ref{app:ext}.

\section*{Code Availability}
The code for the CFTD evaluation metric is published at https://github.com/BAMeScience/cftd. It includes a pre-trained Identity featurizer (a contrastive GNN) that was trained on MP20 data, which allows users to easily score their own generative model, if also trained on this dataset. It also includes the code to retrain the Identity featurizer on new data, if that is necessary.

\bibliographystyle{unsrt}
\bibliography{bibi}
\newpage
\appendix
\section{Primer on Optimal Transport}
We will briefly recall the basics of Optimal Transport (OT), following \cite{peyré2020computationaloptimaltransport}; for a more theoretical, abstract overview, see \cite{villani}. 

Optimal Transport is concerned with the task of how to efficiently transport mass from one probability measure $\mu$ to another one $\nu$. Specifically, it tries to find a coupling $\pi \in \Pi(\mu,\nu)$ such that $\pi = \mathrm{argmin}_{\pi \in \Pi(\mu,\nu)}\int \Vert x-y\Vert d\pi(x,y)$, where $\Pi(\mu, \nu)$ is the set of probability measures with marginals $\mu$ and $\nu$. 

In the case of empirical, disjoint measures of the same size, this reduces to finding the optimal permutation $\pi$, i.e., which point $x_i$ should be transported to which other one $y_j$, which can be solved using the Hungarian algorithm \cite{hungarian}. In practice, we use the POT package \cite{flamary2021pot} with the standard ot.emd function.

Note that OT is also available when the two probability measures have different numbers of samples. However, we then need to split the mass. For instance, consider $\mu = \delta_0$ and $\nu = \frac{1}{2}\delta_0+ \frac{1}{2}\delta_{-2}$, then "half" of the mass of $\delta_0$ gets mapped to 0, whereas the other half gets mapped to -2. This also explains how OT penalizes mode collapse. If a point is overrepresented in one (like the 0) but not in the other measure, it manifests as a transportation error. This is also the reason why TNovD and CFTD penalize distributional shifts, see \cite{hagemann2025transportnoveltydistancedistributional} for an illustration.

Regarding runtime, Hungarian matching (or the LP solver we use) scales as $n^3$ in the number of particles \cite{peyré2020computationaloptimaltransport} and is therefore more expensive than the Sinkhorn algorithm \cite{cuturi_sinkhorn}. However, for a larger number of particles, our method could utilize similar heuristical couplings as done in OT flow matching \cite{pooladian_fm}. To scale to more particles, we outline two potential ways. 

\paragraph{Expected Slice Transport}
To overcome poor scaling, we propose using expected slice transport, introduced in \cite{liu2025expected}. 
The main idea is to use slicing to reduce dimensionality when calculating the optimal transport plan. Assume we want to find a transport between the measure $\mu = \frac{1}{n}\sum_i \delta_{x_i}$ and $\nu =\frac{1}{m}\sum_i \delta_{y_j} $, then we define the sliced probability measure with respect to a vector $\Vert \theta \Vert = 1$ as $\theta_{\#} \mu =\frac{1}{n}\sum_i \delta_{\theta^T x_i} $ and similarly for $\nu$. Now for each $\theta$, find an optimal transport plan $\tilde{\gamma}_{\theta} \in \Pi(\theta_{\#}\mu, \theta_{\#}\nu)$. Since this trick makes it a one dimensional problem, it can now easily be calculated via sorting, before being lifted to a genuine transport between the original points $x_i$ and $y_j$, which we call $\gamma_{\theta}$ (note that this is easiest to understand in the case $n=m$ but also works in $n \neq m$). 
All transport plans are averaged to find the so-called \emph{Expected Sliced Transport} (EST) 
$$\gamma =\mathbb{E}_{\theta \sim S^{d-1}}[\gamma_{\theta}].$$

Note that in \cite{liu2025expected} they have a more general construction, which we can disregard when working with empirical measures and assuming no projection clashes. 

\paragraph{Multiscale Optimal Transport}
In \cite{gerber2017}, another strategy is outlined that reduces scaling to $n \log{n}$. There we find a multiscale partition of training and generated features, i.e., we denote the data by $X_j$ at scales $j=0,...,J$, with 0 being the coarsest scale and $J$ the full data. Each scale is indexed further by a number of clusters, i.e., $X_j = \{(X_j)_k\}_{k=1}^{K_j}$. Each cluster has a representative point (the center of mass), and we solve an OT problem between each scale represented by this point. We now solve OT between training points $(X_j)_k$ at scale $j$ and generated points $(G_j)_k$. Since each finer scale cluster $(X_{j+1})_{k'}$ is contained in one coarse scale cluster, we can propagate that we only solve at scale $j+1$ with the clusters that were supported at scale $j$. Depending on the internal structure, this reduces run time to $n \log{n}$ in the case of a genuine tree partitioning. 

These two directions could support scaling our transport distance to much larger sample sizes. However, neither yields OT couplings and the implications for novelty detection remain to be investigated, as they could match suboptimally or smear mass. 

\section{Transport Novelty Distance and Coverage-Novelty Bounds}
\label{app:tnovd}
We discuss here the recently introduced Transport Novelty Distance and firstly introduce the mathematical notions of coverage and novelty. 
\paragraph{Coverage and Novelty}
In early generative modeling, people proposed precision and recall, as well as continuous versions thereof, to evaluate generative models \cite{NEURIPS2019_0234c510,prec_rec1, prec_rec2}. Following \cite{prec_rec2}, precision is defined as the fraction of the generated distribution $Q$ that is in the support of the true data distribution $P$. More specifically, these are generated samples that lie within the support of the training samples. 
Conversely, all samples from the true data distribution $P$ that are contained in the generated distribution $Q$ yield the recall. This precision--recall viewpoint motivates the two quantities used in our analysis. It provides the conceptual starting point for our analysis, because it separates two complementary failure modes of generative models: generating samples outside the true data distribution and failing to cover parts of the true data distribution. To connect this viewpoint to our transport-based metric, we use the closely related radius-based notions of coverage and novelty. \emph{Coverage} measures whether each data sample is approximated by at least one generated sample, and is therefore aligned with recall. \emph{Novelty} measures whether generated samples remain outside a prescribed neighborhood of the training data, allowing us to distinguish genuinely new samples from memorized ones. Thus, coverage and novelty provide the concrete quantities that CFTD is designed to control. More specifically, we want high coverage regarding the quality of the generated structures (the quality parameters of the generated samples are supposed cover the training data distribution), while we also want high novelty regarding the structural and compositional aspects (the generated structures are supposed to not exactly memorize structures from the training set).

We define the \emph{coverage} \cite{pmlr-v119-naeem20a} of the measure $\mu = \frac{1}{n} \sum_{i=1}^n \delta_{x_i}$ and $\nu = \frac{1}{m} \sum_{i=1}^m \delta_{g_i}$ of radius $r$ as 
$$\mathrm{Cov}_r(\mu, \nu) = \frac{1}{n}\sum_{i=1}^n 1_{\exists j: \Vert x_i -g_j \Vert \leq r}.$$
If the distance of at least one generated sample $g_j$ to the investigated sample $x_i$ is smaller than or equal to $r$, the respective sample $x_i$ is counted as covered by the generated distribution. 
\emph{Coverage} therefore indicates, how much of the true data distribution $\mu$ is covered by the generated distribution $\nu$. 

Similarly, we define the \emph{novelty} \cite{negishi2026continuoussunstableunique} with radius $r$ as 
$$\mathrm{Nov}_r(\mu, \nu) = \frac{1}{m} \sum_{j=1}^m 1_{\forall i: \Vert x_i -g_j \Vert > r}.$$
If all samples $x_i$ from the true data distribution have a distance larger than $r$ to the investigated generated sample $g_j$, the respective sample $g_j$ is counted as novel.
For the final \emph{novelty}, we sum over all samples $g_j$, for which we cannot find any $x_i$ from the true data distribution up to a distance of $r$. As such, \emph{novelty} controls how much of the generated distribution $\nu$ is not included in the training distribution $\mu$.   

Note that \cite{pmlr-v119-naeem20a} defines the error using the distances between k-nearest-neighbors, which is precisely due to the fact that the threshold needs to be chosen data-adaptively.

\paragraph{Transport Novelty Distance}
Here we quickly recall the main definition of TNovD \cite{hagemann2025transportnoveltydistancedistributional}, where we aimed to find a quality-novelty metric for unconditional generative models that jointly penalized generated distributions that are either unrealistic or copies of the training set.  

With available training samples $(x_i)_{i=1}^n$ and generated samples $(g_i)_{i=1}^m$, the TNovD follows two steps:
\begin{enumerate}
    \item Step 1: Calculate the pairwise distance matrix $C_{i,j} = \Vert x_i -g_j \Vert$ and the OT plan $\pi(\mu, \nu) \in \mathbb{R}^{n,m}$ between $\mu = \frac{1}{n}\sum_{i=1}^n \delta_{x_i}$ and $\nu = \frac{1}{m} \sum_{i=1}^m \delta_{g_i}$. This is realized as a doubly stochastic matrix. 
    \item Step 2: Calculate the actual TNovD metric as $$\mathrm{TNovD}(\mu, \nu) = \sum_{i=1}^n \sum_{j=1}^m \pi_{i,j}(\mu, \nu) (\mathrm{ReLU}(C_{i,j}-\tau) + M\ \mathrm{ReLU}(\tau-C_{i,j})).$$
\end{enumerate}
During step 2, $\tau$ serves as a threshold that splits $C_{i,j}$ into a memorization and a quality regime, which are supposed to reflect the above introduced novelty and quality coverage, respectively. Since the memorization values are below the threshold $\tau$, they are naturally lower than those of the quality regime. To adapt to this, a hyperparameter $M$ is tuned to scale the memorization values. The two hyperparameter $\tau$ and $M$ are therefore crucial for the success of this algorithm. 

We will now show that we can bound the quality coverage and novelty using the TNovD, i.e. a low TNovD translates to a high quality coverage \emph{and} a high novelty. While the coverage bound is natural and trivial to formulate, the novelty requires a non-degeneracy condition that ensures that whenever a generated sample is close to a training sample, the OT plan also assigns mass to it. In the following simple lemma, we highlight the connections between TNovD and the established metrics coverage and novelty. 

\begin{lemma}
\label{lem:tnovd}
Let $\mu = \frac{1}{n} \sum_{i=1}^n \delta_{x_i}$ and $\nu = \frac{1}{m} \sum_{j=1}^m \delta_{g_j}$ be as above and assume that $r > \tau$. Then,  
$$1- \mathrm{Cov}_r(\mu, \nu) \leq \frac{\mathrm{TNovD}(\mu, \nu)}{r-\tau}.$$

Now assume $r < \tau$, and assume that there exists $\alpha >0$, $\min_i C_{i,j} \leq r\implies \sum_{i: C_{i,j} \leq r} \pi_{i,j} \geq \frac{\alpha}{m}$. Then we obtain that 

$$1-\mathrm{Nov}_r(\mu, \nu) \leq  \frac{\mathrm{TNovD}(\mu,\nu)}{\alpha\ M\ (\tau-r)} .$$

\end{lemma}

This is a positive statement for TNovD: It shows that we can indeed control both novelty and coverage. On the other hand, to increase both novelty and quality coverage, we would need to choose $r > \tau$ and $r < \tau$ simultaneously.

\begin{proof}
Fix the coupling $\pi(\mu, \nu)$. We now rewrite 
$$\mathrm{TNovD}(\mu, \nu) = \sum_{i=1}^n \sum_{j=1}^m \pi_{i,j} (C_{i,j}-\tau)1_{C_{i,j} > \tau} + M (\tau-C_{i,j}) 1_{C_{i,j} \leq \tau}.$$

We can drop the second part, then this implies $$\mathrm{TNovD}(\mu, \nu) \geq \sum_{i=1}^n \sum_{j=1}^m \pi_{i,j} (C_{i,j}-\tau)1_{C_{i,j} > \tau}.$$

Hence, for $r > \tau$ that 
$$\mathrm{TNovD}(\mu, \nu) \geq \sum_{i=1}^n \sum_{j=1}^m \pi_{i,j} (r-\tau)1_{C_{i,j} > r}.$$

Using this, we directly get  $$\sum_{i,j} \pi_{i,j} 1_{C_{i,j} > r} \leq \frac{\mathrm{TNovD}}{r-\tau}.$$

By properties of the OT plan, namely that $\sum_{j=1}^m \pi_{i,j} = \frac{1}{n}$ we get 
$$\sum_{i,j} \pi_{i,j} 1_{C_{i,j} >r} \geq \sum_{i=1}^n 1_{\forall j: \Vert x_i - g_j \Vert > r} \frac{1}{n} =  1- \mathrm{Cov}_r(\mu, \nu).$$

For the second part, we define the set of non-novel points, i.e., $S = \{j: \min_i C_{i,j} \leq r\}$. These are the generated points that are $r$-close to a training point. 

In the same manner, we bound 
$$\mathrm{TNovD}(\mu, \nu) \geq \sum_{j \in S} \sum_{i: C_{i,j} \leq r} M\ \pi_{i,j} (\tau - r) \geq \frac{|S| \alpha}{m} M (\tau - r).$$

Since $S$ is precisely the set of generated points that are $r$-close to training samples, we can redefine novelty via $$\mathrm{Nov}(\mu, \nu) = \frac{1}{m} \sum_{j =1}^m 1_{j \in S^c} = \frac{1}{m} |S^c|.$$

Using $|S| + |S^c| = m$, we have
$1-\frac{|S|}{m} = \mathrm{Nov}_r$. 

This overall implies the claim by rearranging that $$1-\mathrm{Nov}(\mu,\nu) \leq \frac{\mathrm{TNovD}(\mu,\nu)}{\alpha M (\tau-r)}.$$
\end{proof}

Note that for the CFTD, we now redefine \begin{align*}\mathrm{Cov}_r(\mu, \nu) = \frac{1}{n} \sum_{i=1}^n 1_{\exists j: \Vert F^{mace}(x_i) - F^{mace}(g_j) \Vert \leq r},\ \mathrm{Nov}_r(\mu, \nu) = \frac{1}{m} \sum_{j=1}^m 1_{\forall i: \Vert F^{id}(x_i) - F^{id}(g_j) \Vert > r}.
\end{align*}

\begin{theorem}
\label{thm:cftd}
Define the cost functions $C^{id}_{i,j} = \Vert F_{id}(x_i) - F_{id}(g_j) \Vert$ and $C^{mace}_{i,j} = \Vert F_{mace}(x_i) - F_{mace}(g_j) \Vert$ for the empirical measures $\mu = \frac{1}{n}\sum_{i=1}^n \delta_{x_i}$ and $\nu = \frac{1}{m}\sum_{j=1}^m \delta_{g_j}$. 

Further, denote $\pi$ an optimal coupling associated to the cost $C = \frac{1}{2}C^{id} + \frac{1}{2}C^{mace}$.

Denote $r_{mem} < \tau_{mem}$ and similarly $r_{qual} > \tau_{qual}$ (Goldilocks Zone). 

We then have the coverage (quality) as
$$1-\mathrm{Cov}_{r_{qual}}(\mu,\nu) \leq \frac{\mathrm{CFTD}(\mu,\nu)}{r_{qual}-\tau_{qual}} $$

Assume that there exists $\alpha >0$ such that for every $j$ $\min_i C^{id}_{i,j} \leq r_{mem}  \implies \sum_{i: C^{id}_{i,j} \leq r_{mem}} \pi_{i,j} \geq \frac{\alpha}{m}$, then we have novelty as 
$$1- \mathrm{Nov}_{r_{mem}}(\mu, \nu) \leq \frac{\mathrm{CFTD(\mu,\nu)}}{\alpha\ M (\tau_{mem}- r_{mem})}.$$

\end{theorem}

This shows that we can indeed control both novelty and quality coverage independently, which allows us to explicitly look at the coverage of structural quality and structural memorization detached from one another. It also shows that both parameters are bound by the CFTD.

\begin{proof}
Denote the $\mathrm{CFTD}$ in its probabilistic form via $$\mathrm{CFTD}(\mu,\nu) = \mathbb{E}_{\pi}[M\mathrm{ReLU}((\tau_{mem}-C^{id}))+ \mathrm{ReLU}(C^{mace}-\tau_{qual})].$$

Note that here $C^{id}$ and $C^{mace}$ denote the random variables which are obtained from sampling $(i,j) \sim \pi$ and evaluating the costs. $\mathbb{E}_{\pi}$ is the expectation value with respect to $\pi$

On the event $\{C^{id} \leq r_{mem}\}$, we have $$\mathrm{CFTD}(\mu,\nu) \geq M (\tau_{mem}-r_{mem})\pi(C^{id} \leq r_{mem}),$$

and similarly, $$\mathrm{CFTD}(\mu,\nu) \geq (r_{qual}-\tau_{qual})\pi(C^{mace} \geq r_{qual}).$$

We can proceed as in the case of the proof of Lemma \ref{lem:tnovd} and show that $$\pi(C^{mace} \geq r_{qual}) =  \sum_{i,j} \pi_{i,j} 1_{C^{mace}_{i,j} \geq r_{qual}} \geq \frac{1}{n}\sum_{i = 1}^n 1_{\forall j: C^{mace}_{i,j} \geq r_{qual}} = 1-  \mathrm{Cov}_{r_{qual}}(\mu,\nu).$$

For the novelty we proceed similarly by defining $S = \{j: \min_i C^{id}_{i,j} \leq r_{mem}\}$. 

By assumption we have that $\sum_{i: C^{id}_{i,j}\leq r_{mem}} \pi_{i,j} \geq \frac{\alpha}{m}$, hence $$\pi(C^{id} \leq r_{mem}) = \sum_{j\in S}\sum_{i: C_{i,j}^{id} \leq r_{mem}} \pi_{i,j} \geq \frac{|S|\alpha}{m}, $$
from which we can infer the claim in the same way as before. 
\end{proof}

\section{On the fineness of the Identity features}
\label{sec:infonce}
In this section, we provide a brief mathematical argument explaining why a contrastive loss \cite{oord2019representationlearningcontrastivepredictive} is a good choice for the fine featurizer (i.e., to detect novelty). We want to show that if the InfoNCE loss is small, the probability that materials are mapped to the same features vanishes. This is by no means a new technical argument, but helpful for our story. A similar result, albeit 
in a much different language, has been shown in \cite{pmlr-v124-nozawa20a}. It basically says that a small InfoNCE loss yields good separation between distinct materials, thereby preventing false positives. 

\begin{proposition}
\label{info_nce:sep}
Let $x$ be a material, with positive augment $x^+$ (should be close) and negative augments $x_1,...,x_K$ (should not be close) with feature versions $z = F_{id}(x), z^+ = F_{id}(x^+), z_1,...,z_K = F_{id}(x_1),...,F_{id}(x_K)$
Further denote the InfoNCE loss for $x$ by $l_{nce} =-  \log\left(\frac{\exp(z^T z^+)}{\exp(z^T z^+)+\sum_{k=1}^K \exp(z^T z_k) } \right).$ Then the probability of the event $A_{\gamma} = \{\exists j: z^T z_j \geq z^T z^+ - \gamma\}$ can be bounded by 
$$\mathbb{P}(A_{\gamma}) \leq \frac{\mathbb{E}(l_{nce})}{\log(1+\exp(-\gamma))}.$$

\end{proposition}
\begin{proof}
On the event $A_{\gamma}$, we have that by definition for some fixed $j$ that 
$$\frac{\exp(z^T z^+)}{\exp(z^T z^+)+\sum_{k=1}^K \exp(z^T z_k) } \leq \frac{\exp(z^T z^+)}{\exp(z^T z^+)+ \exp(z^T z_j)} \leq \frac{1}{1+\exp(-\gamma)}.$$

Taking logs yields that $$l_{nce} \geq \log(1+\exp(-\gamma)).$$
Hence, by Markov, we have $$\mathbb{P}(A_{\gamma}) \leq \mathbb{P}(l_{nce} \geq \log(1+\exp(-\gamma))) \leq \frac{\mathbb{E}(l_{nce})}{\log(1+\exp(-\gamma))}.$$
\end{proof}

This exemplifies the role of the fine featurizer: it defines an augmentation-dependent identity geometry in which structures related by the chosen augmentations are kept close, while other materials are encouraged to be separated. One could also take more classical approaches, such as the novelty "featurizers" employed in \cite{negishi2026continuoussunstableunique}; however, they need to be adapted to have a fixed length to fit the OT framework.

\section{Hyperparameter Selection}
\label{app:hyp_sel}
Here, we provide a precise description of how the hyperparameters are chosen in our framework. 

The input to the hyperparameter selection is as follows:
\begin{enumerate}
    \item The training data set
    \item The validation data set 
    \item The $\beta$-Goldilocks percentage, i.e., the relative number of points that are categorized into the Goldilocks zone. 
\end{enumerate}

From there on, we calculate the featurized versions of the training and the validation samples with respect to the two featurizers $F^{id}, F^{mace}$, and find the optimal coupling with respect to the coupling $C = \frac{1}{2} C^{id} + \frac{1}{2}C^{mace}$. Then, we assign that $\frac{1-\beta}{2}$ mass should belong to the "memorized" and $\frac{1-\beta}{2}$ belongs to the "low quality" zone. 

This allows us to enumerate the coupled cost. Now, we set $\tau_{mem} = \mathrm{Quantile}_{\frac{1-\beta}{2}}(C^{id})$ and $\tau_{qual} = \mathrm{Quantile}_{1-\frac{1-\beta}{2}}(C^{mace})$ with respect to the OT plan $\pi$ between train and validation, which just says that the thresholds are chosen such that only $\frac{1-\beta}{2}$ of the mass are categorized below $\tau_{mem}$ or above $\tau_{qual}$. Further, during the calculation of $\mathrm{CFTD}$, we choose these regions to be mutually exclusive, favoring a categorization as memorized via setting the "low quality ones" as the ones with $C^{mace}$ above $\tau_{qual}$ \emph{that are not categorized as memorized}. 

Now, we can calculate the memorized value $\sum_{i,j} \pi_{i,j} (\tau_{mem} - C^{id}_{i,j})$ as well as the "low quality" value  $\sum_{i,j} \pi_{i,j} (C^{mace}_{i,j}- \tau_{qual})$ and choose $M$ as the fraction of "low quality" value over memorized, i.e., scaling the values to match on the validation set. We remark that this is purely heuristic and intended to retain interpretability. It is a great future direction to improve this with concrete data labeling or explicit data perturbation procedures. 

\paragraph{Fine Featurizer Parameters}
The fine featurizer (to detect memorization) was trained for 10 epochs using an equivariant GNN with batch size 128, a contrastive InfoNCE loss and an augmentation function with modified Pettifor nearest-neighbor and next-nearest-neighbor substitution probabilities of $0.3$ and $0.01$, respectively, and Gaussian coordinate noise applied with probability $0.5$ and scale $0.02$.

\section{On the coarseness of MACE features}
\label{app:coarse_mace}
We investigate here the nature of the coarseness of the chosen MACE features, i.e. their ability to represent the quality of a structure. First, we will discuss what information the MACE features actually contain, particularly whether we can reconstruct space groups, lattice, and compositional information. 

Then, we will contrast the fine (for novelty) and coarse (for quality) features by probing how they treat augmentations, whether they separate them, and by comparing their cosine similarities. 

\subsection{What information is contained in the MACE features?}
\label{app:info_mace}
To further elucidate the role of the MACE features, we examine whether we can reconstruct space groups, composition, and lattice parameters, by simply learning a standard feedforward neural network from the MACE feature space to the respective structural parameters. We perform composition, lattice, and space-group analyses on the MP20-dataset. As it turns out, composition is the best-controlled part of a material, i.e., the easiest one to reconstruct from the MACE features. 
To check whether the reconstruction of space groups is inferred by the good prediction of the composition, we repeat the space group prediction on the Carbon24 data set \cite{carbon2020data,xie2022crystal}. This is interesting, since each material in this dataset consists solely of carbon, and the results tackles the hypothesis that MACE features encode only composition and no structural information. Given the architecture and success of pre-trained MACE models, this was, to a certain extent, already expected. However, since we apply mean pooling to the features and restrict ourselves to a certain number of random projections to create the feature space, we needed to assess the influence of these changes on the final features. Further, we test four different projection lengths of the mean-pooled invariant MACE features (32, 64, 128, and 256), with larger projections generally conveying more information. 

For the composition, we measure the $L1$-Distance between the predicted probability (i.e., the fraction of atoms in the generated structures relative to the full set) and the true frequency in the validation set. For the lattice, we compute the MAE between predicted and true lattice vectors and angles. The results can be seen in Fig.\ref{fig:lat_comp} and show that the composition is quite predictable, whereas an accurate prediction of the lattice is harder. Generally, both improve as the number of projections increases. 

Regarding the space group reconstruction, we report the accuracy of the predicted top space group for the two datasets in Fig.\ref{fig:sg_curve}. The space group accuracy is around 0.35-0.40 on the MP20 and 0.55-0.58 on the Carbon24 dataset.

These experiments show that the projections of the mean-pooled MACE features capture both chemical and geometric information. However, they do not suffice to perfectly reconstruct any parameter. The closest one is chemical composition. This underlines the coarse character of the MACE features. Generally, information improves as we add more projections. 

\begin{figure}
    \centering

    \begin{subfigure}{0.38\linewidth}
        \centering
        \includegraphics[width=\linewidth]{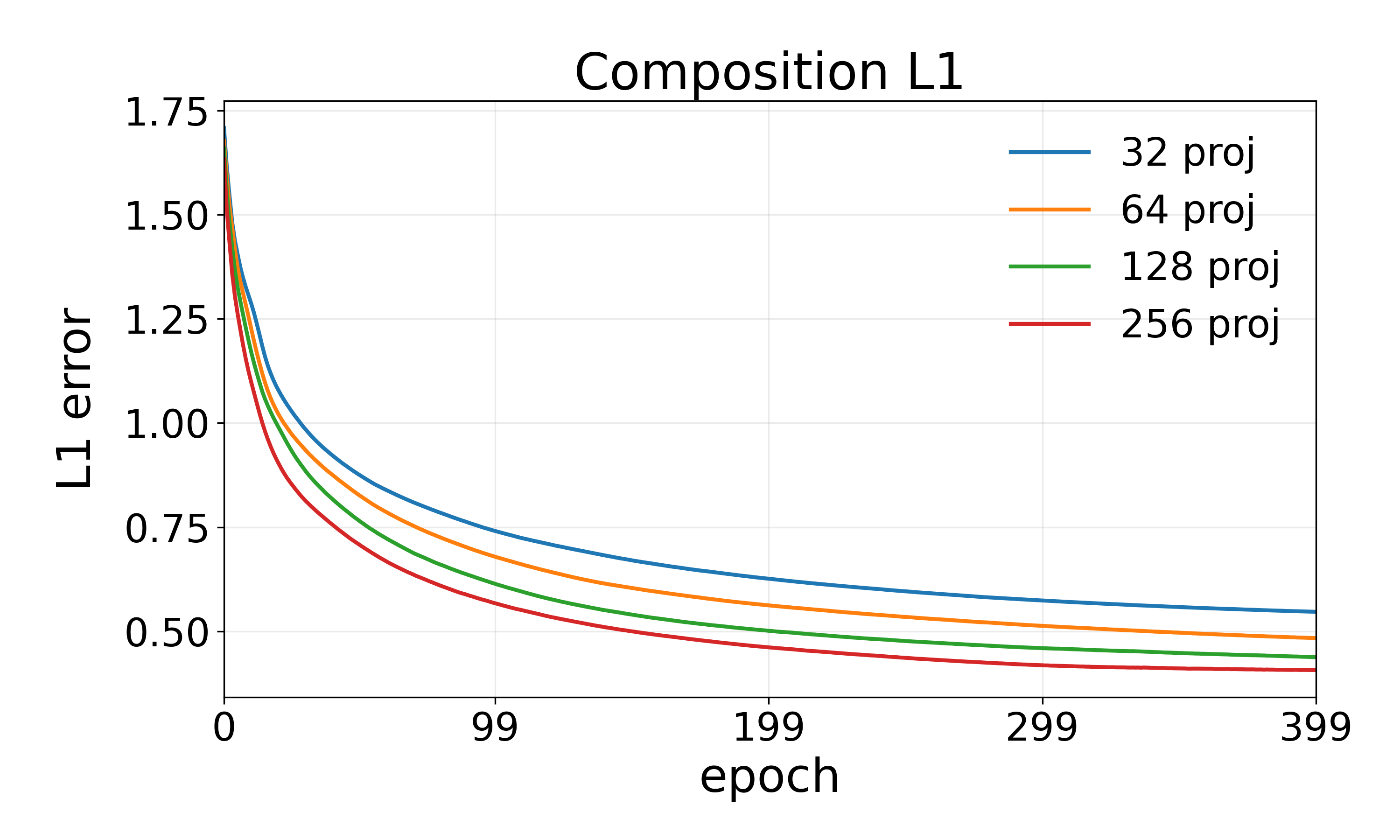}
        \caption{Composition curve.}
        \label{fig:comp_mp20}
    \end{subfigure}
    \hfill
    \begin{subfigure}{0.55\linewidth}
        \centering
        \includegraphics[width=\linewidth]{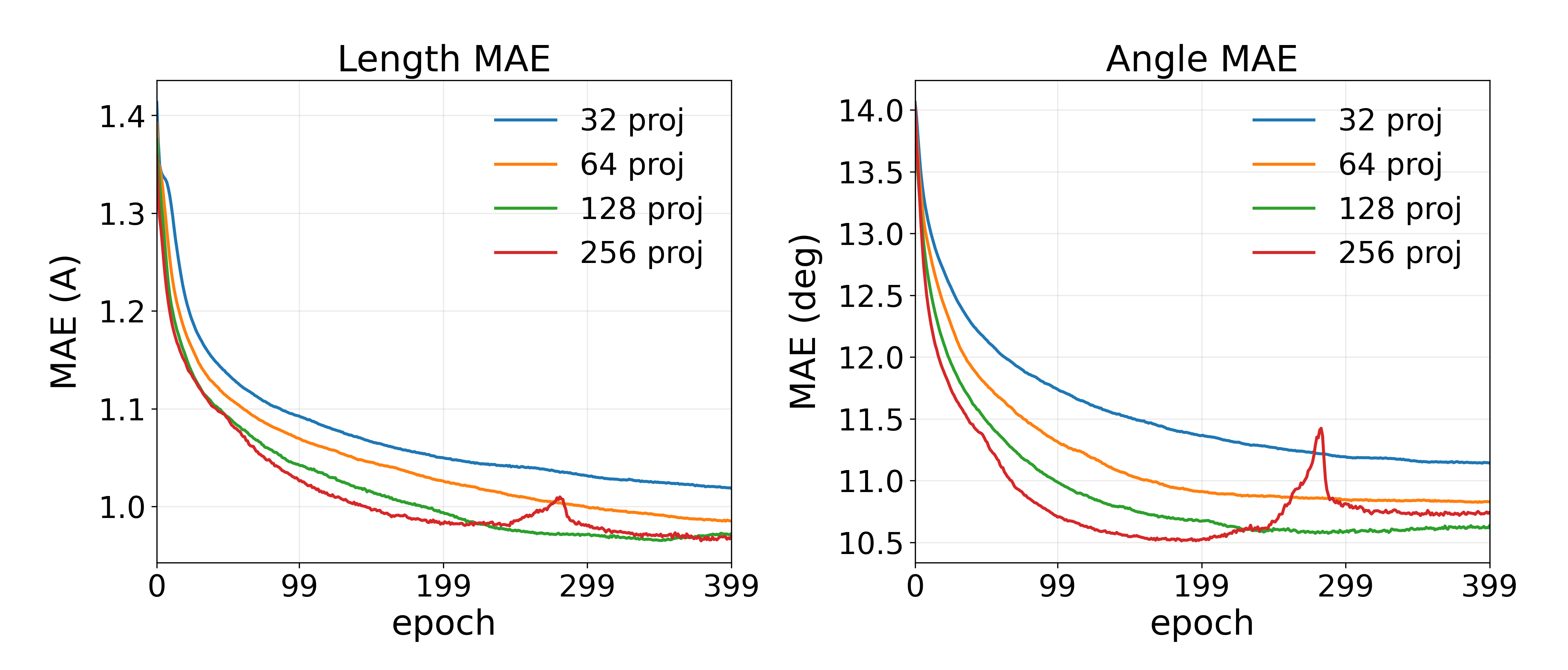}
        \caption{Lattice curve.}
        \label{fig:lattice_mp20}
    \end{subfigure}

    \caption{Composition and lattice reconstruction curves on MP20.}
    \label{fig:lat_comp}
\end{figure}
\begin{figure}
    \centering
    \begin{subfigure}{0.48\linewidth}
        \centering
        \includegraphics[width=\linewidth]{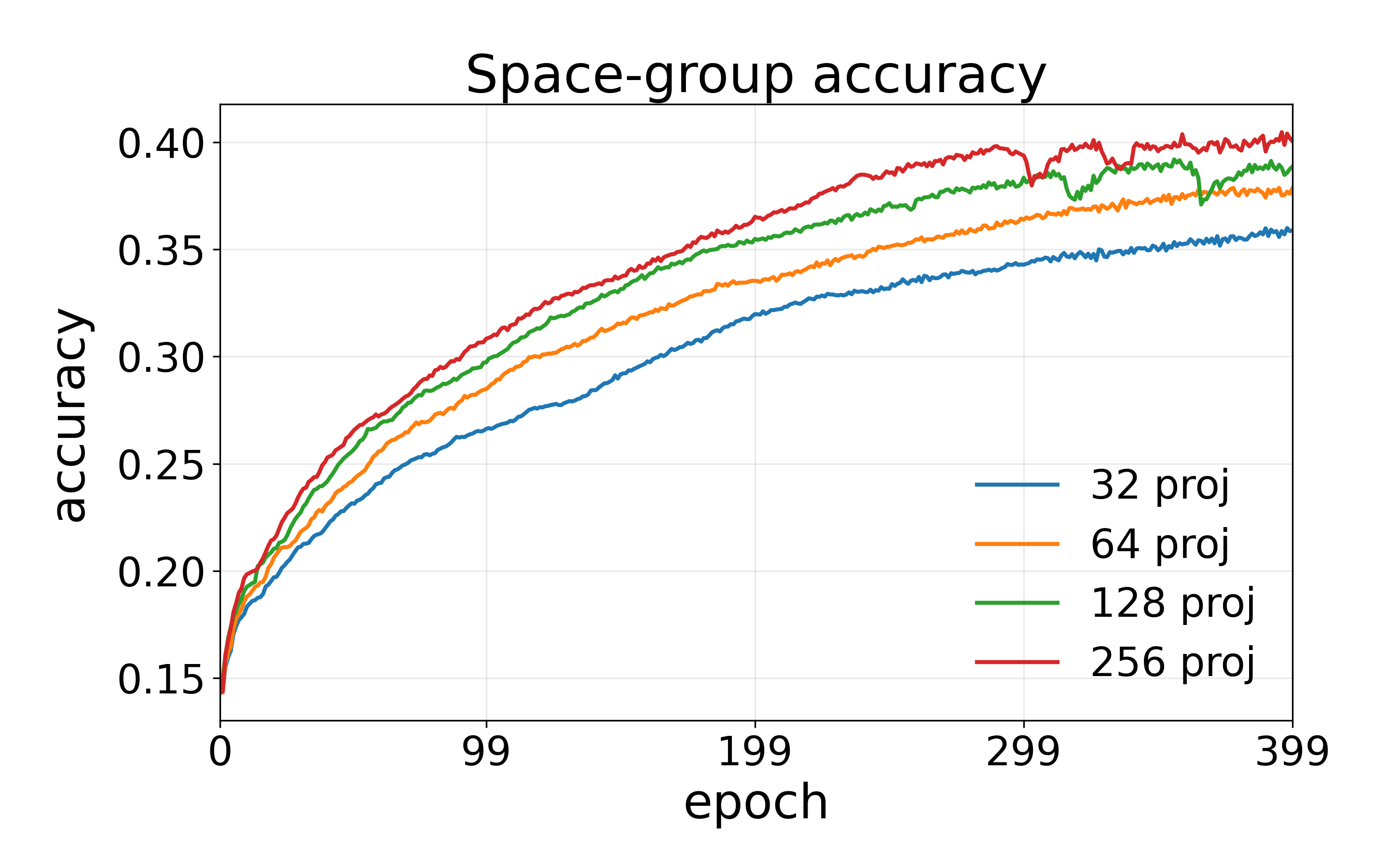}
        \caption{Space group accuracy for MP20.}
        \label{fig:sg_mp20}
    \end{subfigure}
    \hfill
    \begin{subfigure}{0.48\linewidth}
        \centering
        \includegraphics[width=\linewidth]{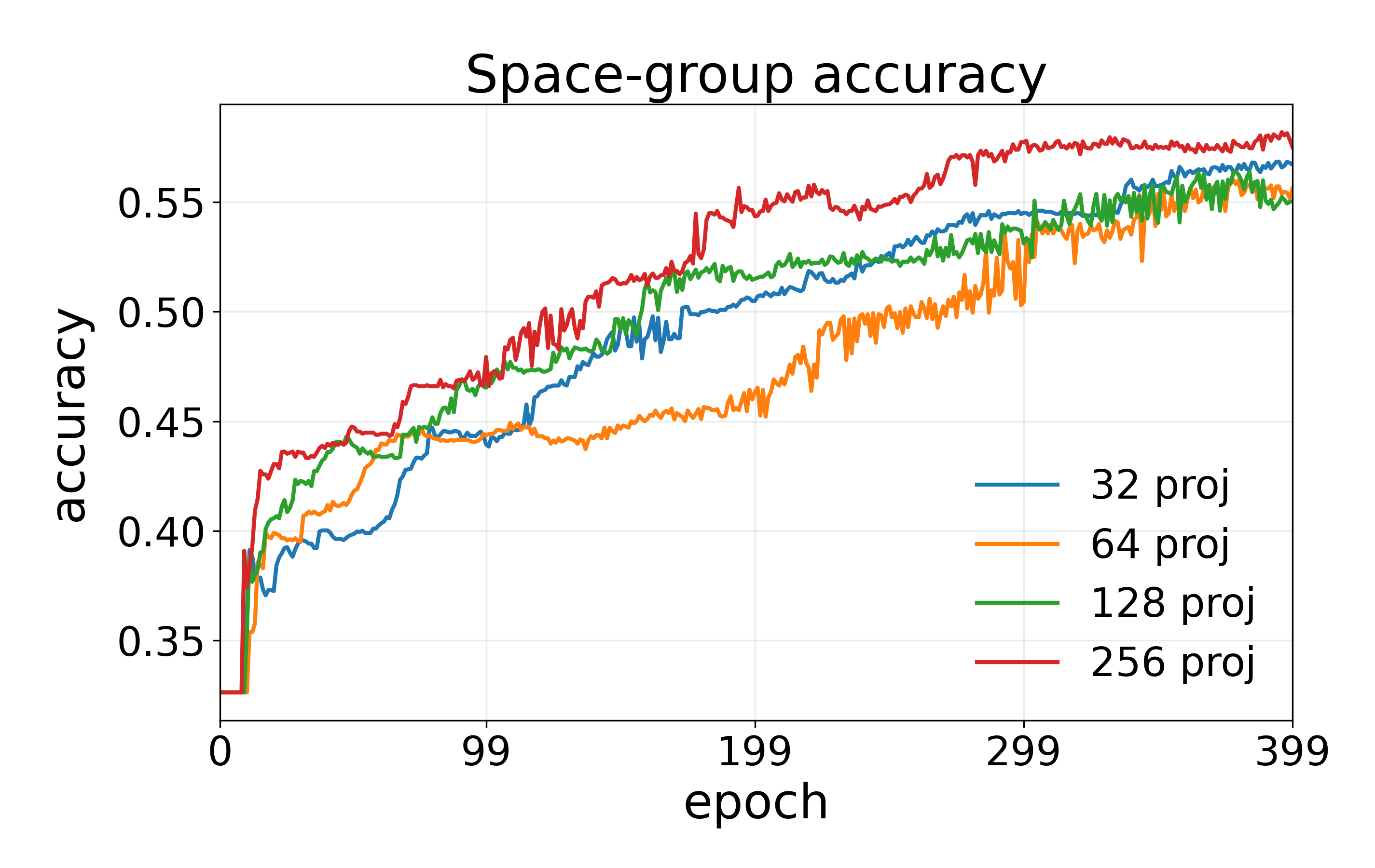}
        \caption{Space group curve for Carbon24.}
        \label{fig:sg_carbon24}
    \end{subfigure}

    \caption{Space group reconstruction curves.}
    \label{fig:sg_curve}
\end{figure}

\subsection{What makes the Identity featurizer fine?}
\label{app:fine_coarse}
The role of the Identity featurizer is to provide a different kind of "geometry" differentiating between positive (those structures that should be close) and negative pairs (those structures that should not be close) as indicated by Proposition \ref{info_nce:sep}. Therefore, we investigate whether the contrastive GNN indeed separates positive and negative pairs differently than the MACE featurizer. To test this, we randomly select 5000 samples from MP20 train set and sweep over $\gamma$ as in the Proposition. We measure the number of negative pairs that exceed the threshold. More specifically, we generate for each material a positive augment and check whether there exists a different material that exceeds the augment cosine similarity with a subtracted $\gamma$. Hence, as $\gamma$ increases, the probability of hitting a negative pair (a different material) instead of the augment increases. We do this for both the MACE and the Identity featurizer, for which the results can be seen in Fig.\ref{fig:info_nce}. Generally, the contrastive GNN provides a much better separation, as indicated by the much lower probability of margin violation. This explains the role of the fine Identity featurizer, which induces a completely different geometry than the coarse MACE featurizer. 

\begin{figure}
    \centering
    \includegraphics[width=0.5\linewidth]{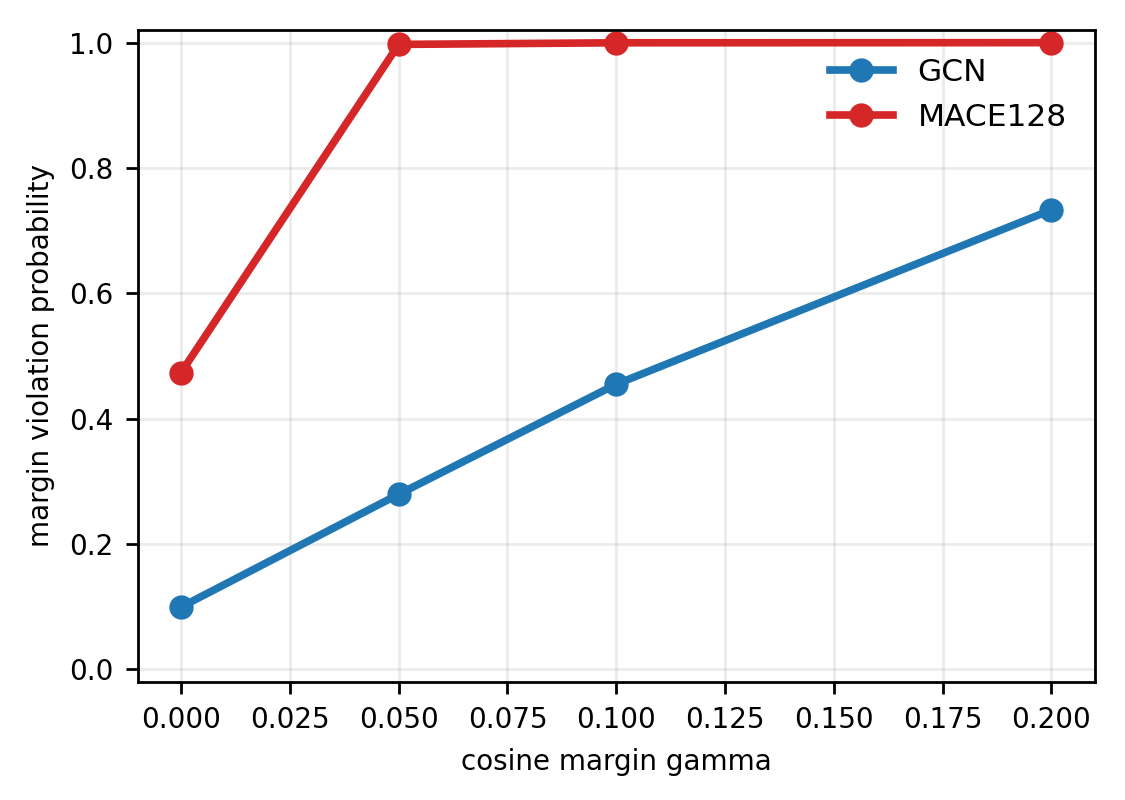}
    \caption{Separation between the MACE featurizer and the Identity featurizer.}
    \label{fig:info_nce}
\end{figure}

\subsection{A cosine investigation}
To further understand the different featurizers, we randomly drew two times 5000 samples from the MP20 train set and formed random pairs between the two created subsets. 
These samples were then transformed into the two feature spaces, using the MACE (coarse) and Identity (fine) featurizers. The features were not only calculated on "clean" structures directly drawn from the MP20 train set, but also on "noisy" structures. We induced this noise by altering the fractional coordinates with Gaussian noise and applied a relatively high standard deviation of 0.05. Fig. \ref{fig:hist_gcn_mace_joint} compares the cosine similarity for each featurizer, i.e., Identity or MACE, and randomly selected feature pairs of, in total, four different sets. The sets cover feature pairs for clean vs. clean, clean vs. noisy, noisy vs. clean, and noisy vs. noisy structures. For the Identity featuzier, cosine similarity values are mainly focused around 0 and high similarities are the absolute exception for all four sets. The low similarity values indicate that the contrastive GNN is indeed trained to achieve separation of the different samples, even when noise is present. For the MACE featurizer, we see a different behavior: for the clean-clean pairs, we observe that relatively few have a cosine similarity below 0.5. The MACE tail, however, changes for clean vs. noisy or noisy vs clean feature pairs. This indicates that the MACE feature space also accounts for perturbations, as it is trained not only on "clean" and stable data. The histograms in Fig. \ref{fig:hist_gcn_mace_joint} thereby further underline the different geometries learned by both featurizers.

\begin{figure}
    \centering
    \includegraphics[width=0.8\linewidth]{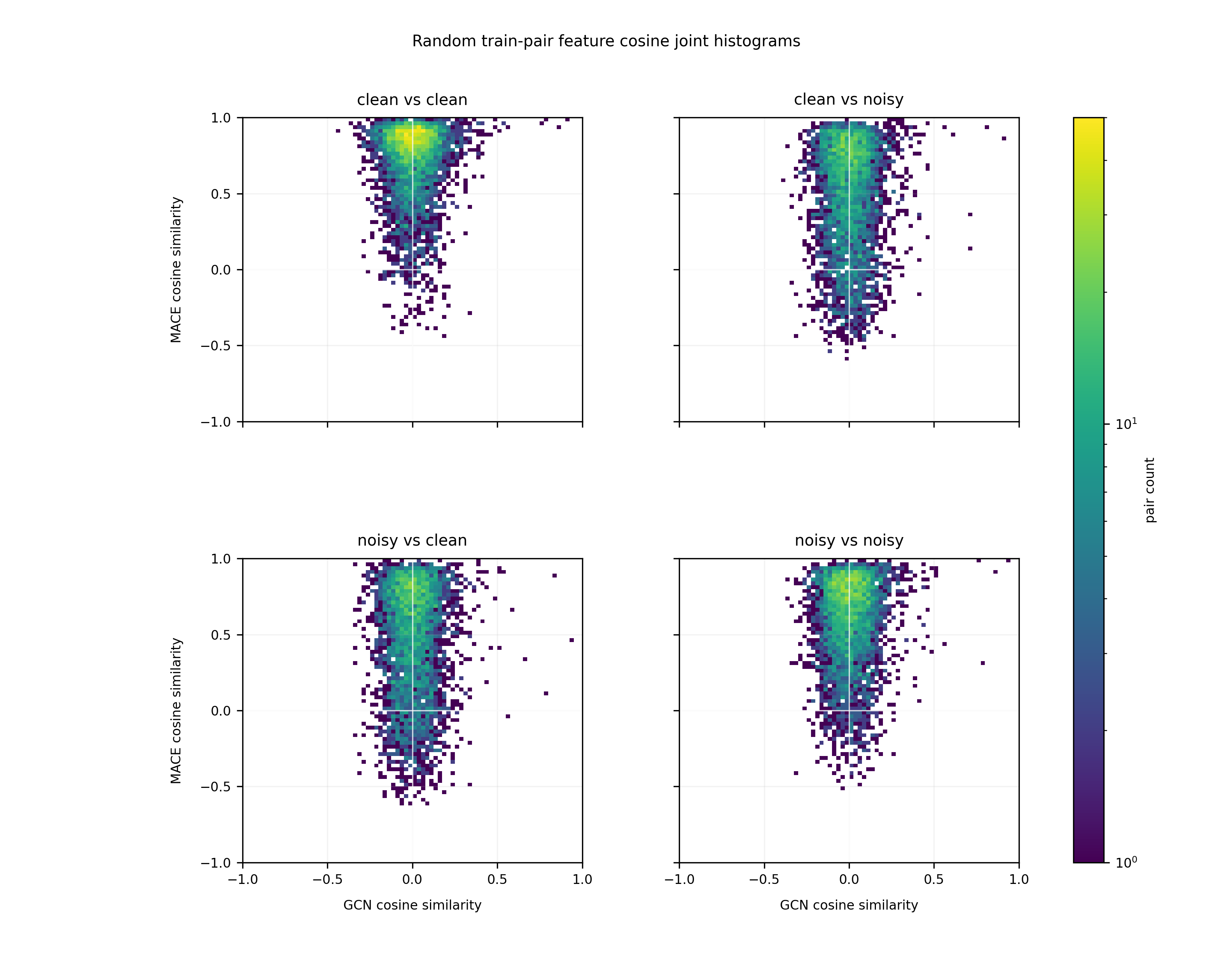}
    \caption{Joint histograms of cosine similarities between random pairs of clean or noisy data for the two featurizers.}
    \label{fig:hist_gcn_mace_joint}
\end{figure}

\section{On the implicit stability of MACE features}
\label{app:stability-mace}
A natural question that arises is what notion of stability the MACE features really encode. We call this \emph{implicit} stability. Of course, due to mean pooling and the random projections, exact node-level information is lost. Since MACE predicts energies and forces, it is still reasonable to expect that it captures these even in a condensed feature space. We test this with the following experiments on formation energies.

\paragraph{MP20 Sweep:} We test how well we can predict formation energies of MP20 from the mean pooled MACE feature projections and trained a simple multilayer perceptron (MLP) to predict the formation energies using the MAE as the loss function. The results are shown in Fig.\ref{fig:mace_form_energy} for a varying number of projections of the mean pooled invariant MACE features as input for the MLP. We can see that all of the variants capture some notion of stability and greatly improve with training. More projections also improve the information encoded. Therefore, Fig.\ref{fig:mace_form_energy} shows that mean pooled MACE feature projections are a reasonable proxy for estimating formation energy, and possibly also local stability. 

\begin{figure}
    \centering
    \includegraphics[width=0.8\linewidth]{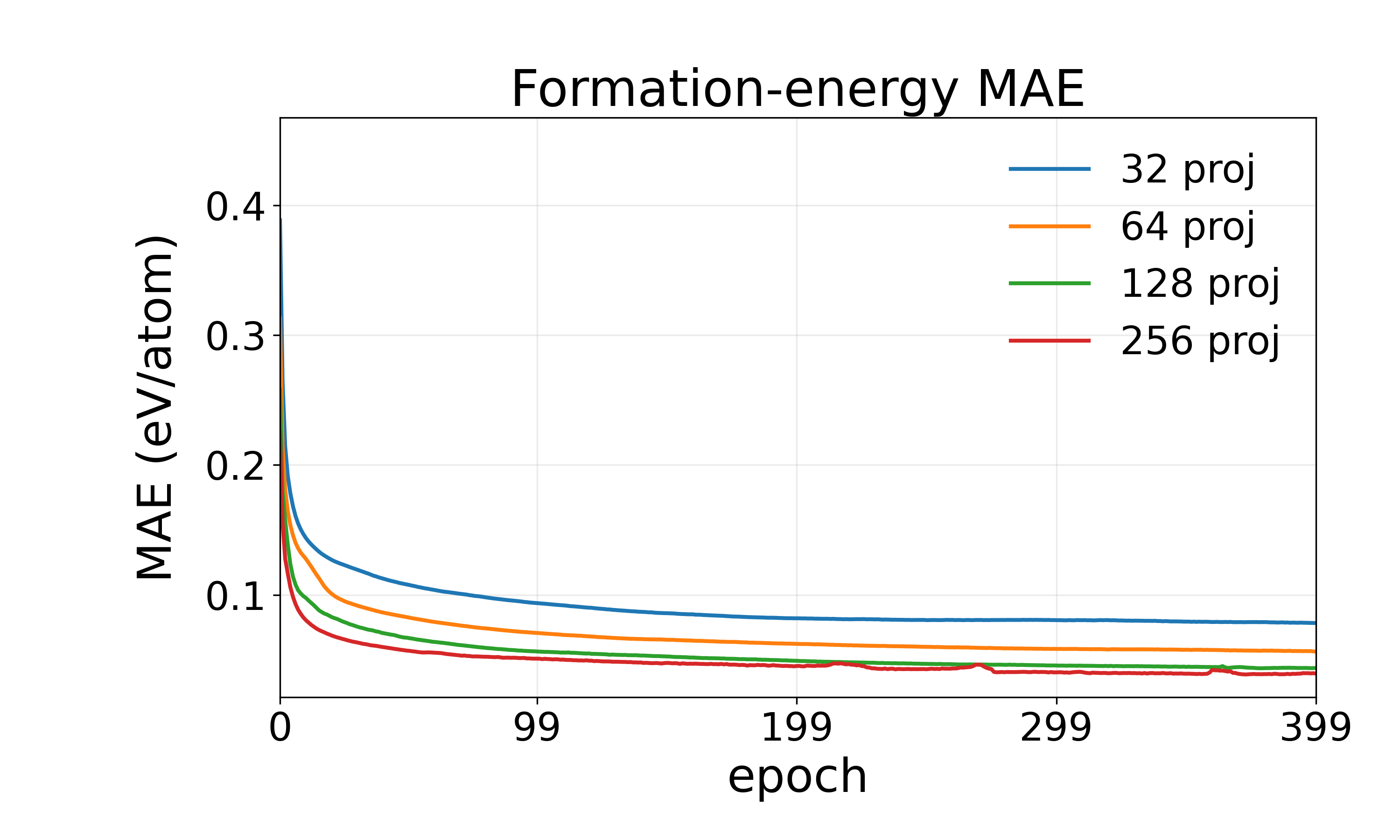}
    \caption{MAE of formation energy per atom over epochs for different lengths of mean pooled MACE feature projections. }
    \label{fig:mace_form_energy}
\end{figure}

\paragraph{Experiments on Polymorphs: } We continue our investigation by analyzing the formation energies of polymorphs (i.e., compositions where different crystal structures exist). With this test, we want to check whether the differences in their energy above the hull (essentially, their differences are equivalent to the formation energy differences for the polymorphs) can be sufficiently separated by the MACE featurizer. We check all polymorphs of SiO2 and ZnS in the MP dataset \cite{jain_mp}. We specifically choose polymorphs because they have the same composition, which we can therefore exclude this as a confounder. In the results shown in Fig.\ref{fig:poly_hull}, both investigated polymorphs exhibit a near linear correlation between their energy above hull differences with the distance between their respective MACE feature projections and although several outliers are visible in the plots, a general trend persists. This trend is likely unsurprising, as the readout phase for the invariant features in the first layers of the MACE models uses a linear readout. The experiment underscores our general thesis that pooled MACE feature projections are still a useful proxy for precise stability information. Nevertheless, there may be several outliers, as shown in the figure. 

\begin{figure}[htbp]
    \centering

    \begin{subfigure}{0.8\textwidth}
        \centering
        \includegraphics[width=\linewidth]{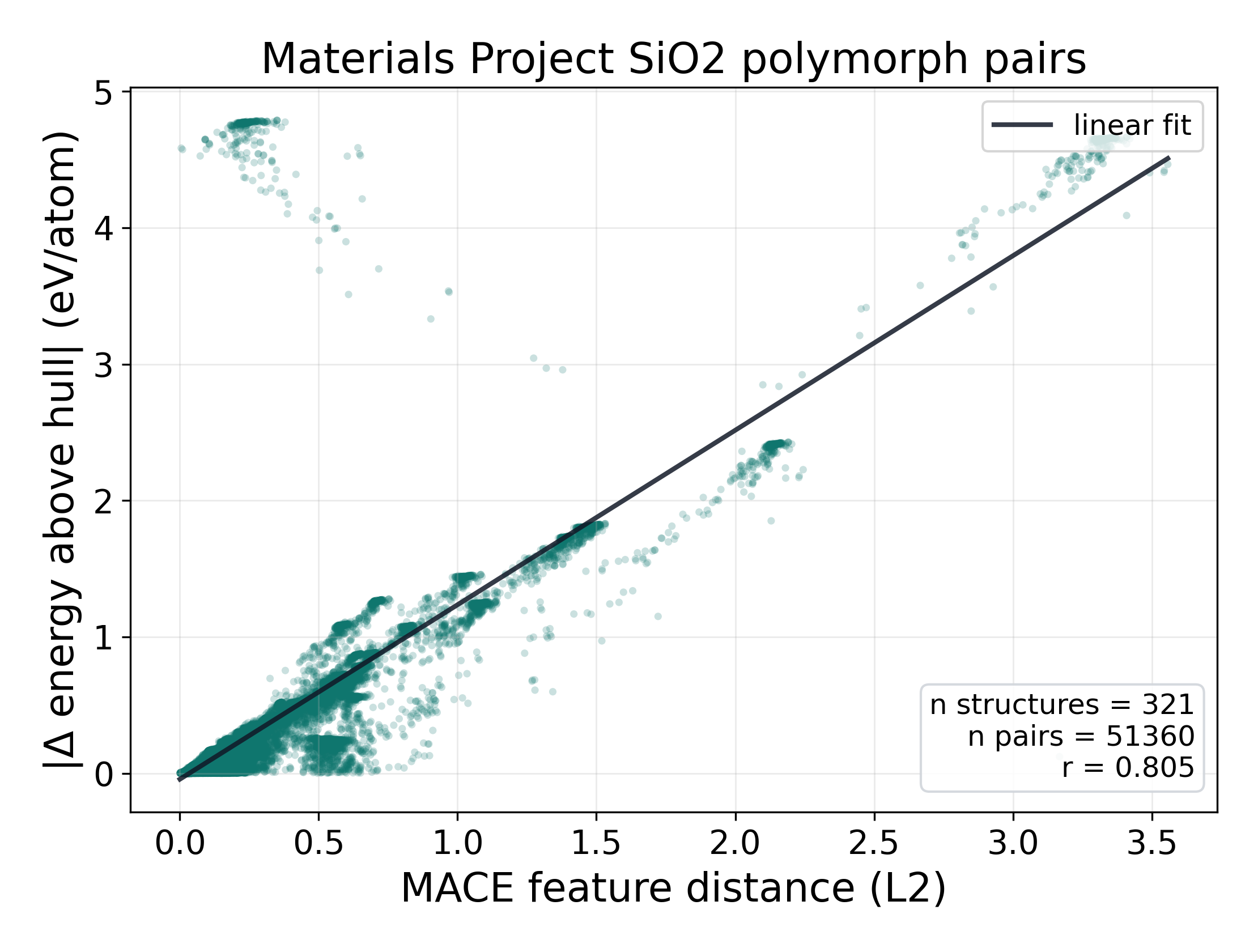}
        \caption{Hull-energy difference versus MACE feature L2 difference for SiO$_2$ polymorphs.}
        \label{fig:sio2_poly_hull}
    \end{subfigure}

    \vspace{0.5em}

    \begin{subfigure}{0.8\textwidth}
        \centering
        \includegraphics[width=\linewidth]{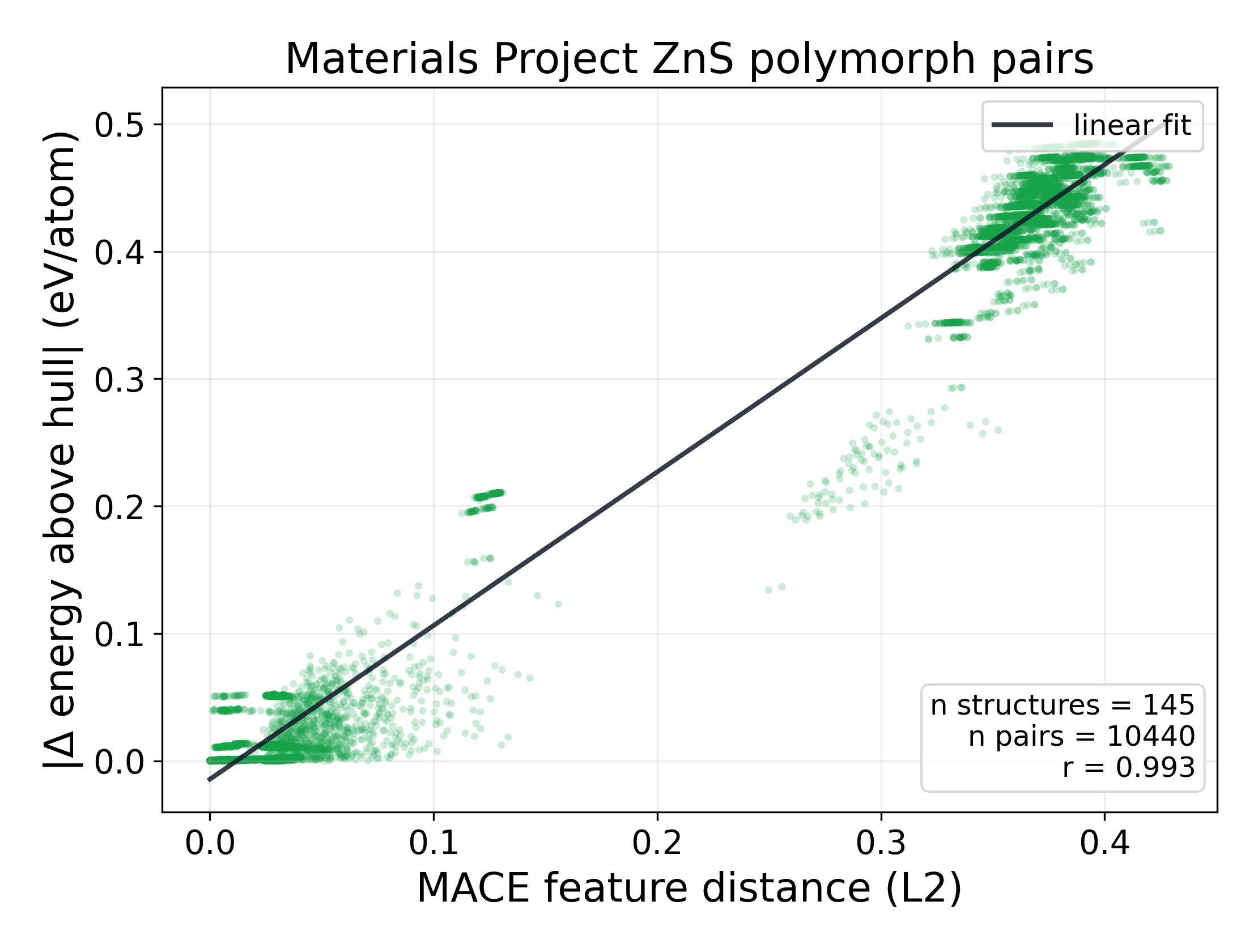}
        \caption{Hull-energy difference versus MACE feature L2 difference for ZnS polymorphs.}
        \label{fig:zns_poly_hull}
    \end{subfigure}

    \caption{Polymorph plots showing hull-energy differences as a function of MACE feature L2 differences.}
    \label{fig:poly_hull}
\end{figure}

\section{Toy Experiment on Perov-5}
\label{sec:perov_5}

\begin{figure}[H]
    \centering

    \begin{subfigure}{0.45\textwidth}
        \includegraphics[width=\linewidth]{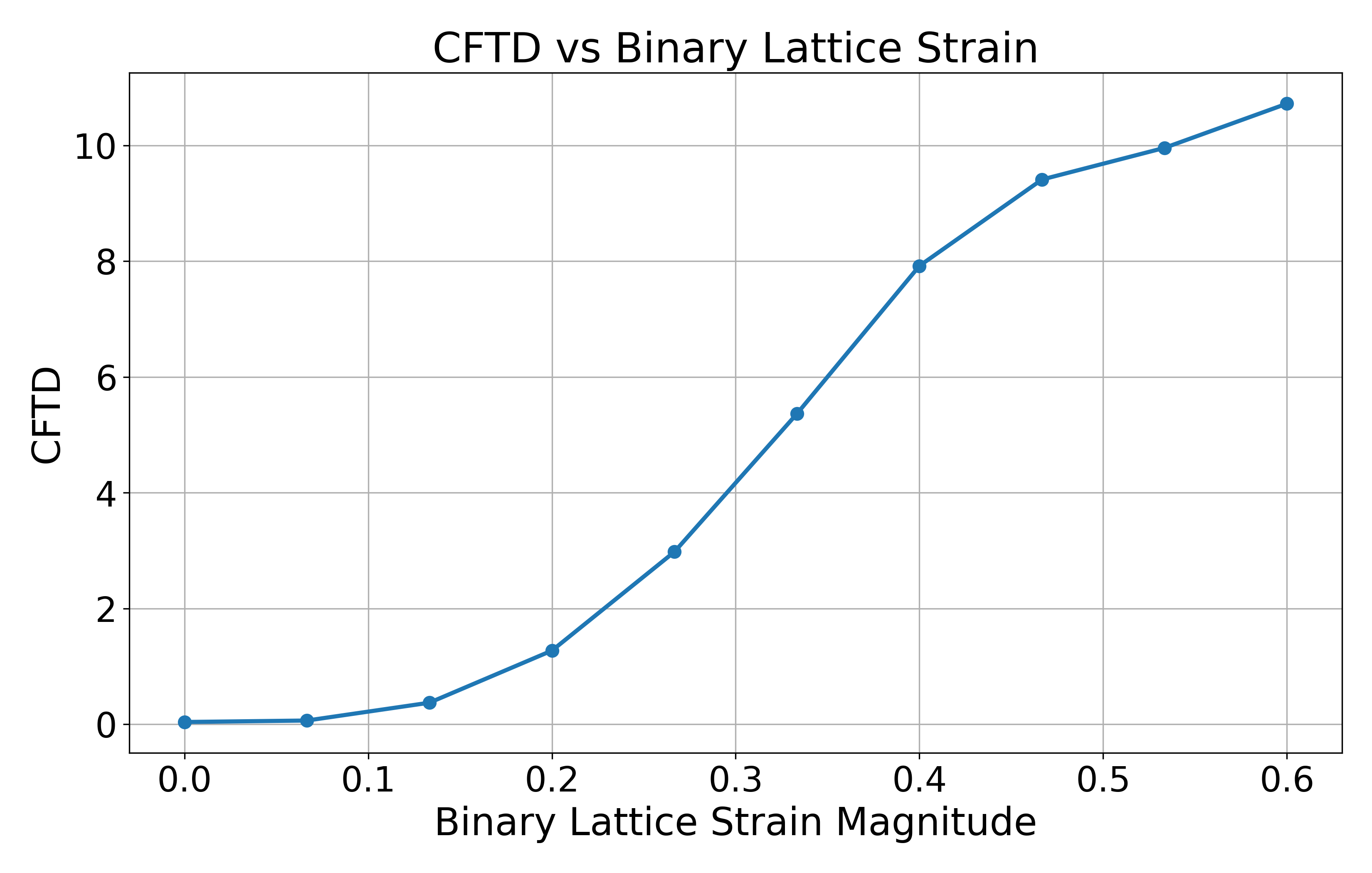}
        \caption{CFTD over lattice strains.}
    \end{subfigure}
    \hfill
    \begin{subfigure}{0.45\textwidth}
        \includegraphics[width=\linewidth]{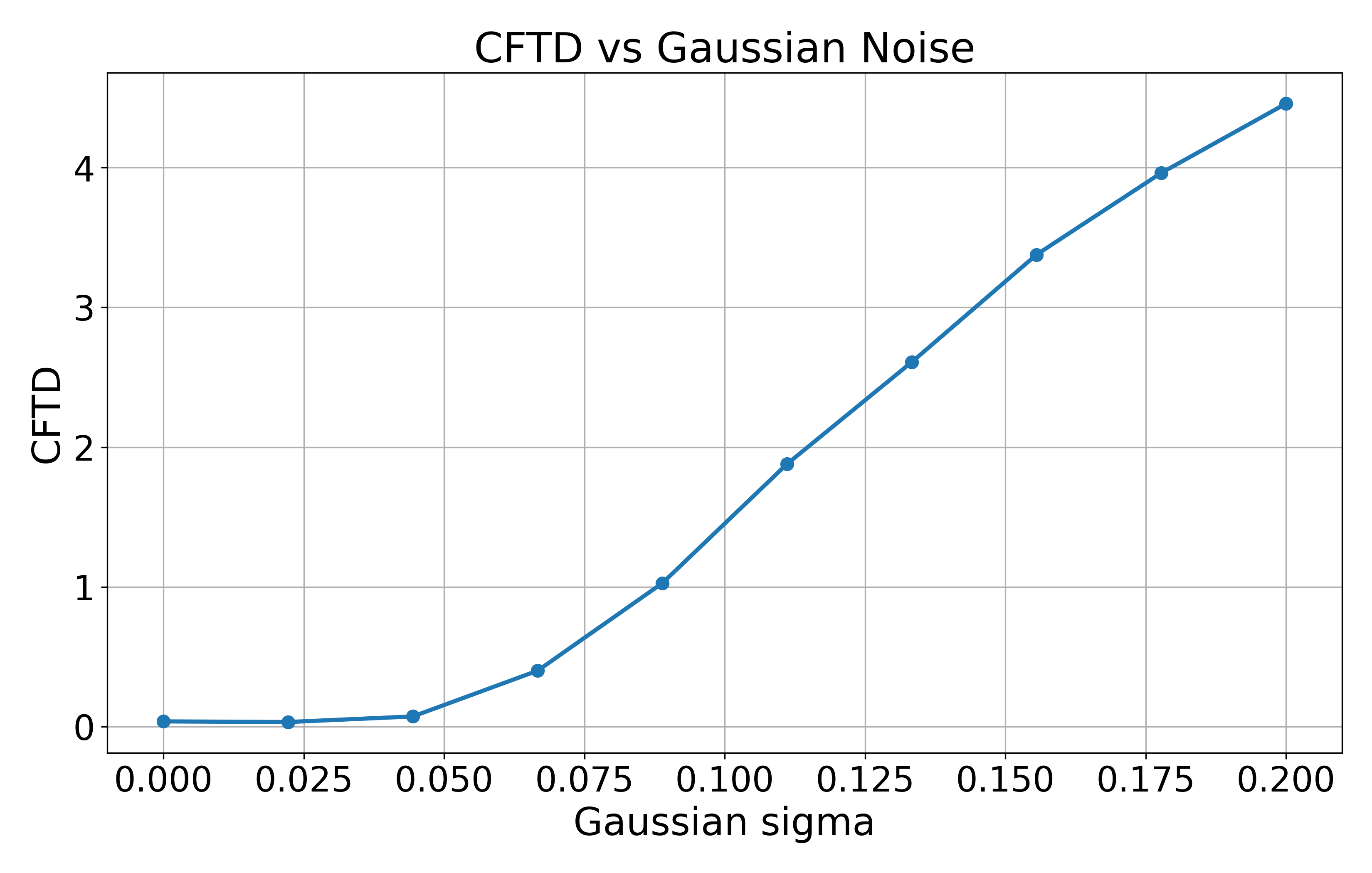}
        \caption{CFTD over Gaussian noise.}
    \end{subfigure}

    \vspace{0.5cm}

    \begin{subfigure}{0.45\textwidth}
        \includegraphics[width=\linewidth]{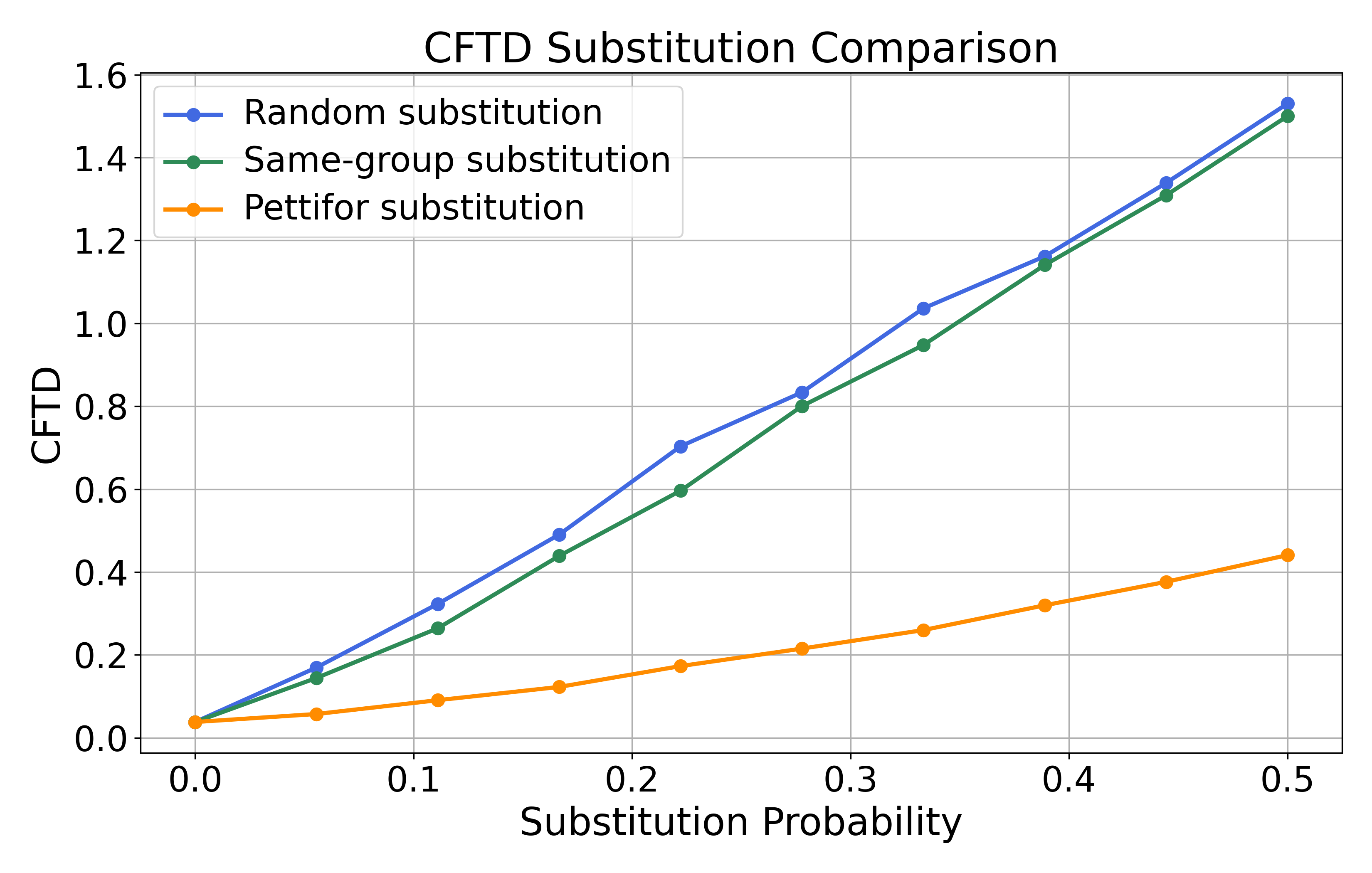}
        \caption{CFTD over substitutions.}
    \end{subfigure}
    \hfill
    \begin{subfigure}{0.45\textwidth}
        \includegraphics[width=\linewidth]{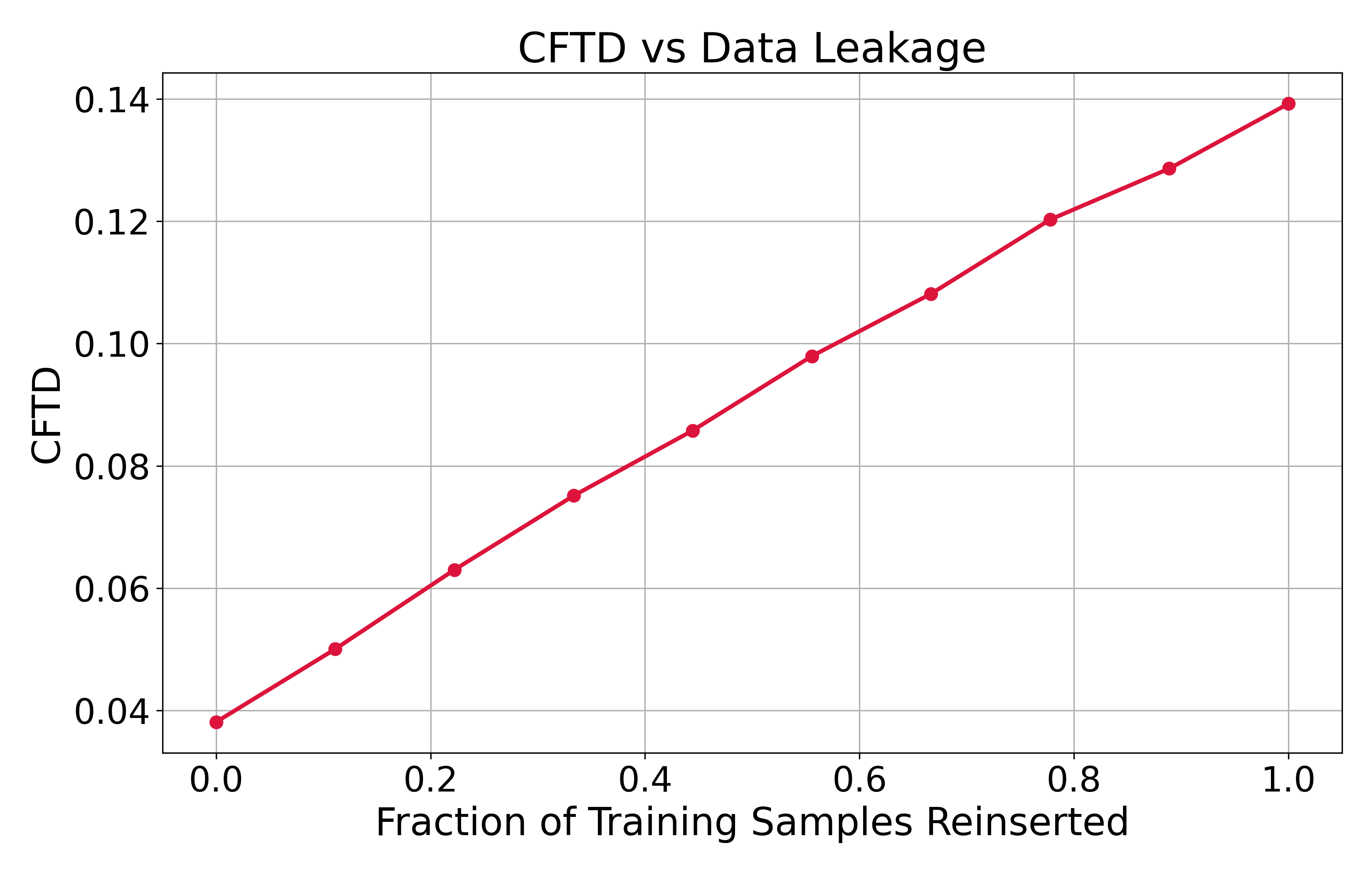}
        \caption{CFTD over train leakage.}
    \end{subfigure}

    \caption{CFTD behavior across different perturbation settings on Perov-5.}
    \label{fig:CFTD_perts_perov}
\end{figure}

To show that the toy experiments from Fig. \ref{fig:CFTD_perts} in \label{Quality-Novelty Checks} extend to larger training sets and different structures, we repeat them for the Perov-5 dataset \cite{castelli_perov5, castelli_perov52, xie2022crystal}. It contains more materials that share a similar structure but different compositions (over 18,000 in total), making it structurally different from MP20. We repeat the same deformations as before and also test for data leakage. 

Note that we reused the Identity featurizer that was trained on MP20, as well as the same MACE featurizer. The results can be seen in Fig. \ref{fig:CFTD_perts_perov}. Note that for this experiment, we chose modified Pettifor nearest neighbors from the MP20 set, not only from the atoms contained in Perov-5. A close investigation reveals that for the first addition of Gaussian noise (0.022) to the atomic positions, the CFTD dips slightly, which indicates that for such low numbers of noise, memorization decreases more quickly than quality penalization kicks in. 

As for MP20, the plots of the individual experiments reveal large differences in the absolute values of the new metric. Artificial crystal degeneration is penalized more strongly than data leakage, and it responds particularly sensitively to lattice strain and changes in atomic positions. This is again different from the behavior observed in TNovD. We argue that this result comes from the inclusion of the coarse MACE featurizer, which enables the new CFTD to now accurately capture the quality of a crystal structure. The much lower values observed in the data leakage experiment result from our implementation of tuning the hyperparameter $M$. The optimal value of $M$ is determined by aligning the naturally lower values of the memorization regime with the values of the quality regime on the same scale. This adjustment is performed on a split of MP20 train and validation set and, since the quality regime in this scenario has lower CFTD values, $M$ will scale the memorization regime accordingly. 

\section{Ablation study on Goldilocks zone $\beta$}

\begin{figure}[h!]
    \centering

    \begin{subfigure}{0.45\textwidth}
        \centering
        \includegraphics[width=\linewidth, scale = 0.8]{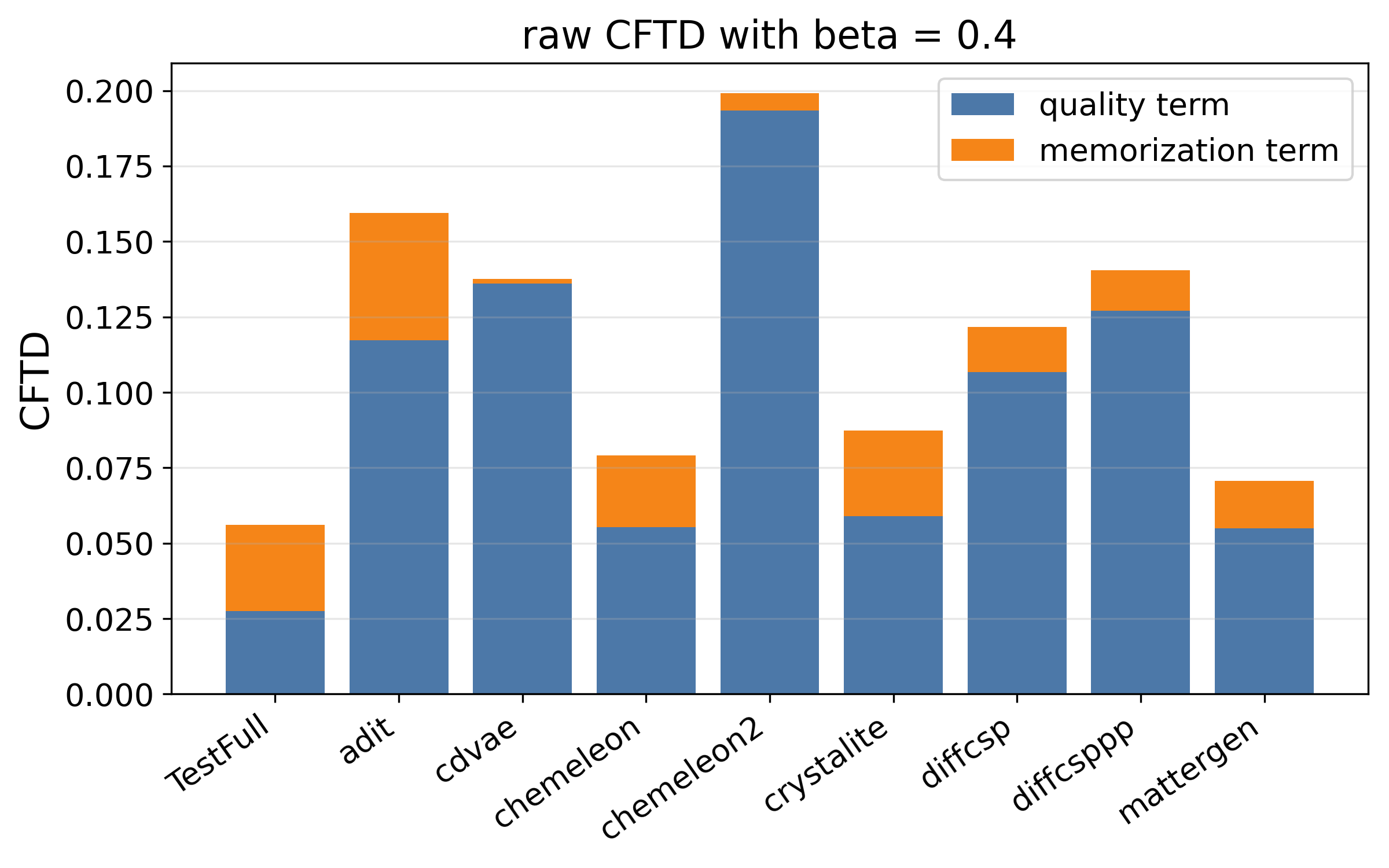}
        \caption{CFTD for unrelaxed models $\beta = 0.4$.}
    \end{subfigure}
    \hfill
    \begin{subfigure}{0.45\textwidth}
        \centering
        \includegraphics[width=\linewidth, scale = 0.8]{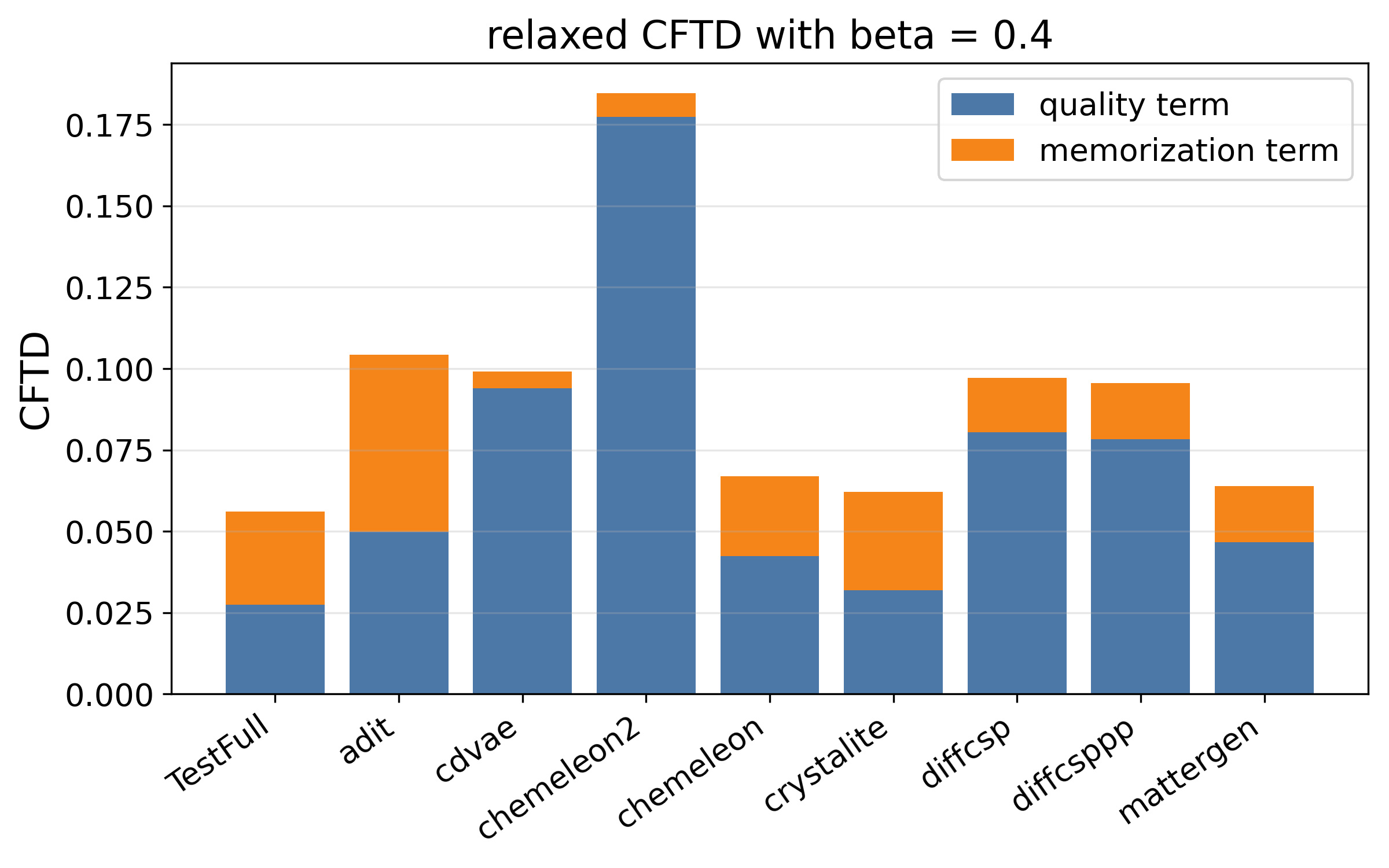}
        \caption{CFTD for relaxed models $\beta = 0.4$.}
    \end{subfigure}
    
    \begin{subfigure}{0.45\textwidth}
        \centering
        \includegraphics[width=\linewidth, scale = 0.8]{imgs_model/model_eval_raw.png}
        \caption{CFTD for unrelaxed models $\beta = 0.6$.}
    \end{subfigure}
    \hfill
    \begin{subfigure}{0.45\textwidth}
        \centering
        \includegraphics[width=\linewidth, scale = 0.8]{imgs_model/model_eval_relaxed.png}
        \caption{CFTD for relaxed models $\beta = 0.6$.}
    \end{subfigure}
    \begin{subfigure}{0.45\textwidth}
        \centering
        \includegraphics[width=\linewidth, scale = 0.8]{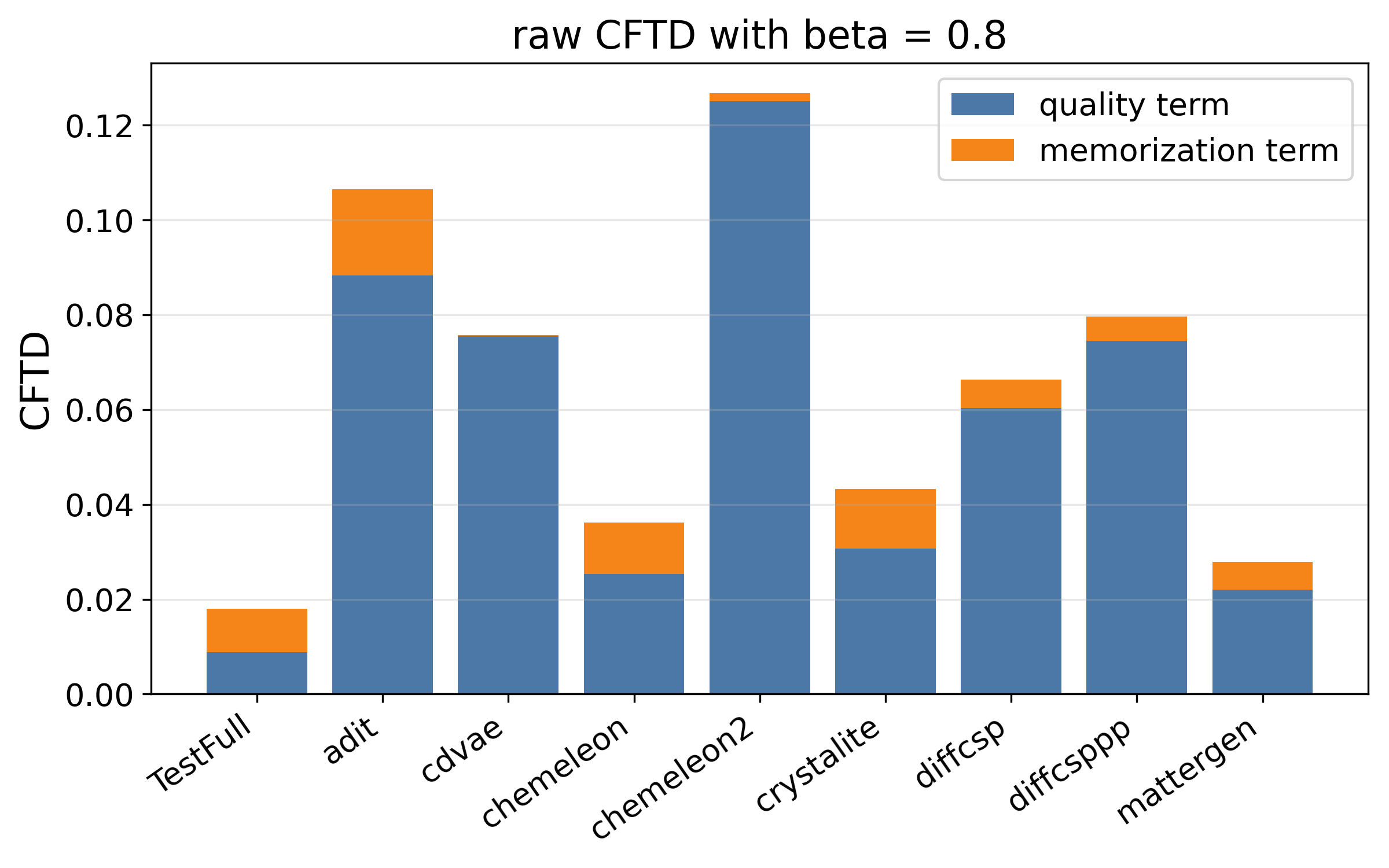}
        \caption{CFTD for unrelaxed models $\beta = 0.8$.}
    \end{subfigure}
    \hfill
    \begin{subfigure}{0.45\textwidth}
        \centering
        \includegraphics[width=\linewidth, scale = 0.8]{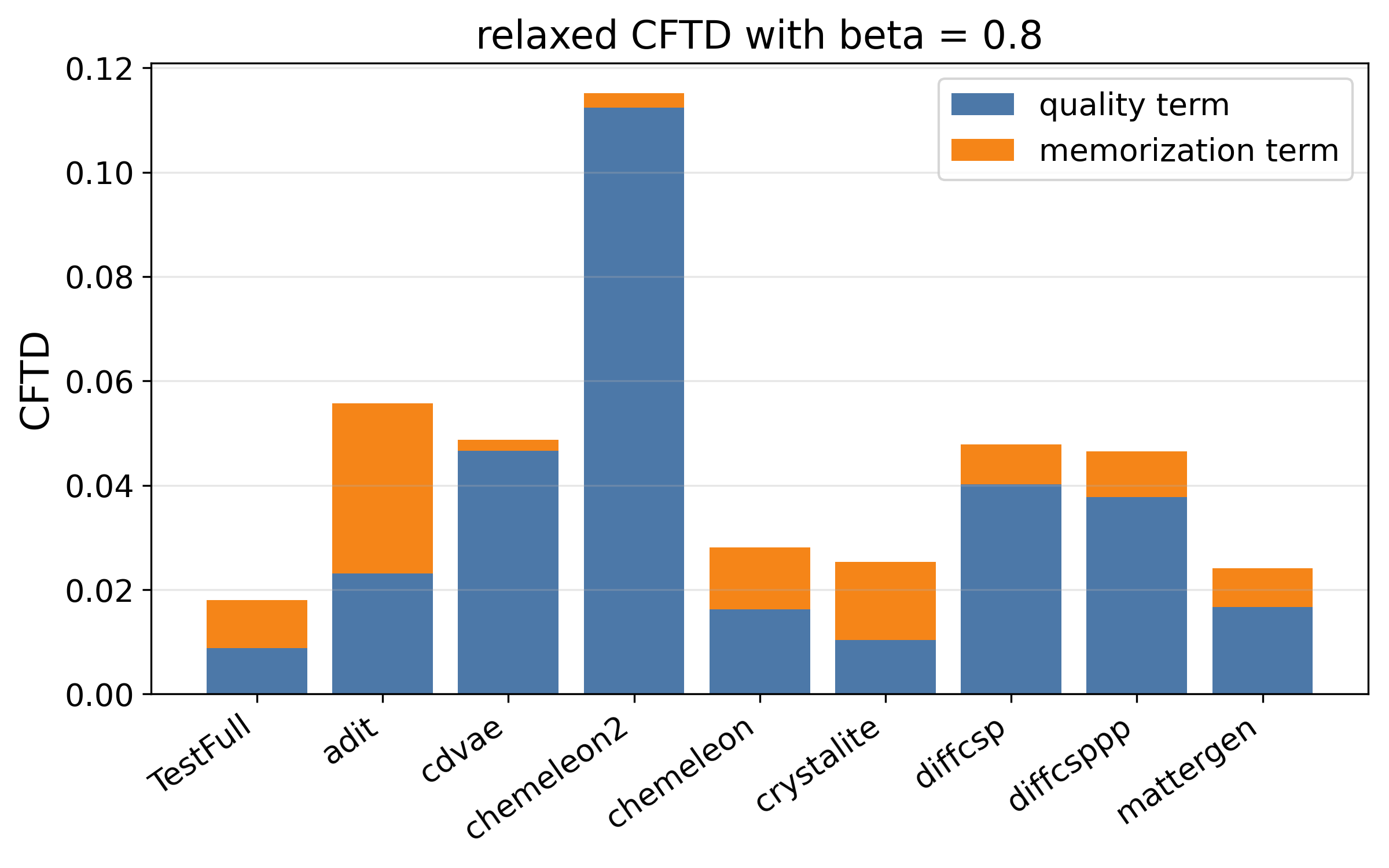}
        \caption{CFTD for relaxed models $\beta = 0.8$.}
    \end{subfigure}
\caption{Comparison of common material generative models: relaxed and unrelaxed for varying $\beta$. }
\label{fig:models_cftd_beta}
\end{figure}

Here, we ablate how much influence the size of the Goldilocks zone has. This is controlled by the parameter $\beta$. The higher $\beta$ is, the less quality and memorization errors are penalized. Hence, we need to ablate whether our conclusions, i.e., the internal rankings of the models, stay consistent with varying $\beta$. For this, we run model ablation and comparison experiments for $\beta \in \{0.4, 0.6, 0.8\}$ on the raw and relaxed versions of the generated materials. In Fig. \ref{fig:models_cftd_beta}, we can see that the general trend of the performance persists very well across different $\beta$, indicating that the results by our CFTD are highly robust. In particular, the strongest and weakest models remain clearly separated, and only minor ranking changes occur among models with similar CFTD values.

\section{Ablation on the balance factor $M$}
Similarly, we ablate the results with respect to the factor $M$. Usually, it is chosen that, on a calibration set (e.g., the validation set), both the quality and memorization penalties are balanced. However, this can be changed according to the type of model that is supposed to be favored. For instance, one could weigh memorization less than quality, favoring more conservative, training-like models, or penalize it more heavily. Therefore, we ablate on calibrating $M$ such that both terms are balanced, and then set $M_{new} \in \{0.5M, M, 1.5M\}$. 

The results on the models rankings are shown in Fig.\ref{fig:models_cftd_M}, where the relative ranking does not change for unrelaxed models, but it does for relaxed ones. 
In particular, Crystalite benefits strongly from a lower memorization penalty. For $M_{new}=0.5M$, it becomes the best generated model after relaxation. Also ADiT's ranking benefits from a lower memorization penalty. Generally, $M$ is a tuning knob for the quality-novelty profile, and hence rankings are not expected to be stable.

\begin{figure}[h!]
    \centering

    \begin{subfigure}{0.45\textwidth}
        \centering
        \includegraphics[width=\linewidth, scale = 0.8]{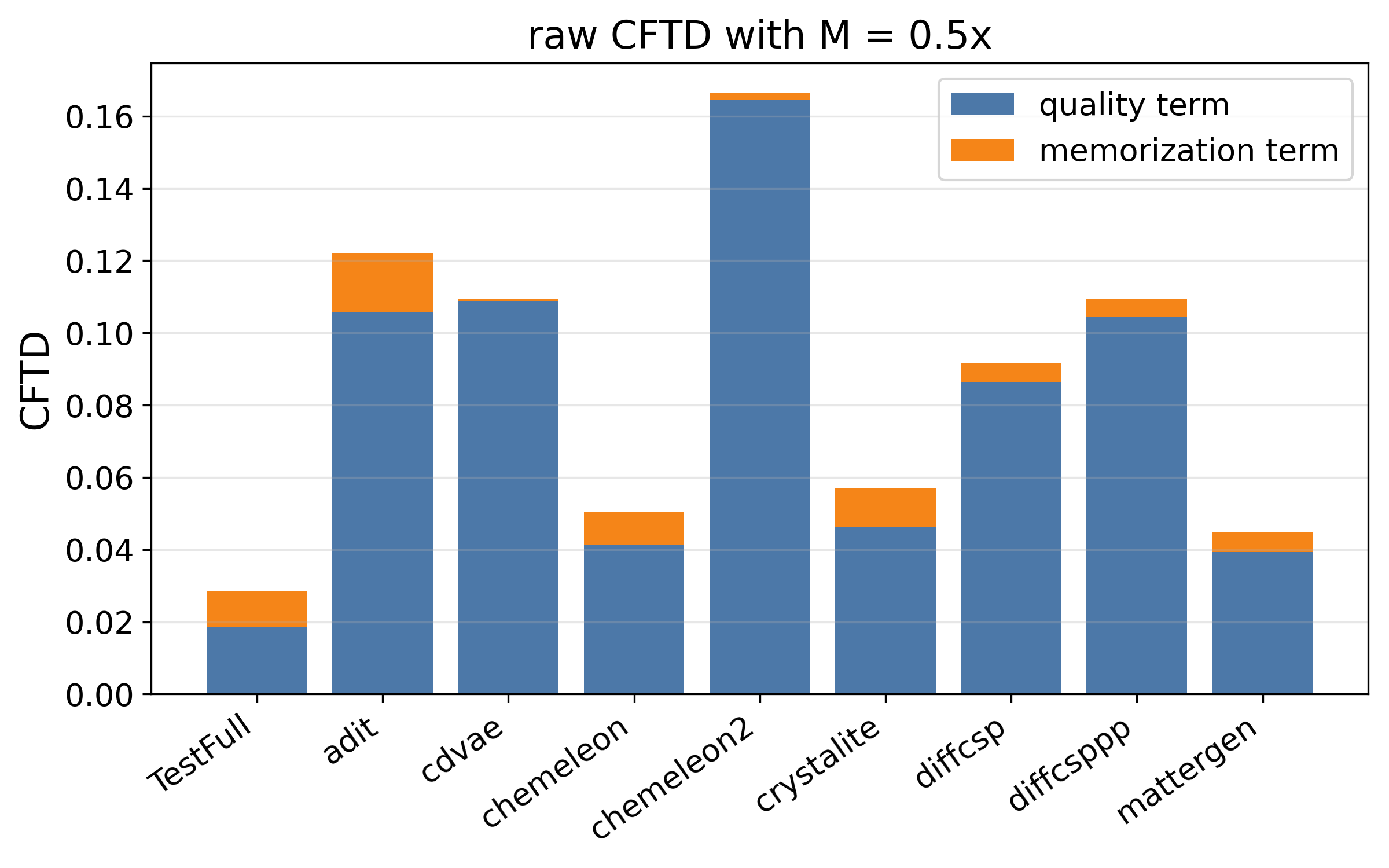}
        \caption{CFTD for unrelaxed models $M_{new}= 0.5M$.}
    \end{subfigure}
    \hfill
    \begin{subfigure}{0.45\textwidth}
        \centering
        \includegraphics[width=\linewidth, scale = 0.8]{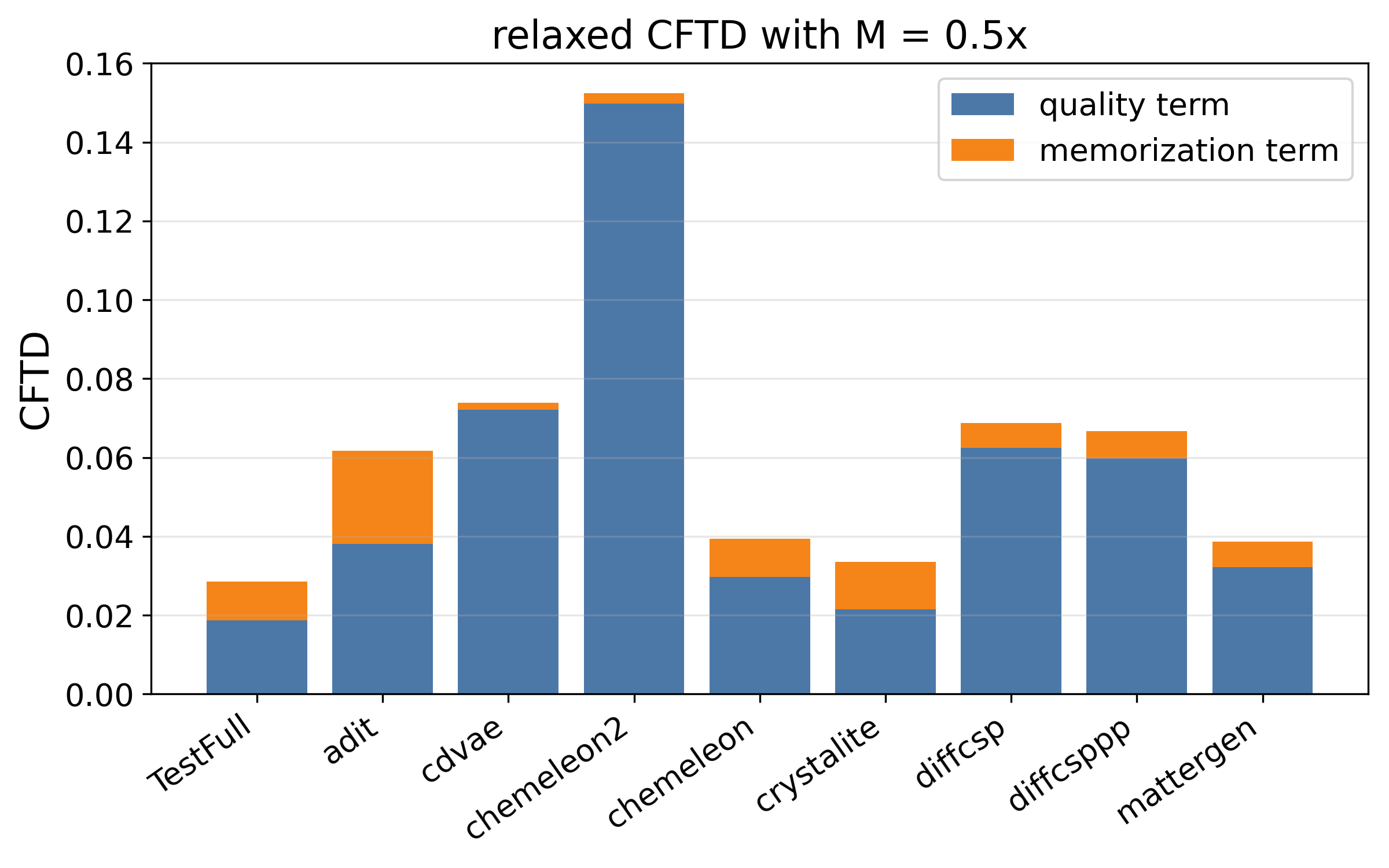}
        \caption{CFTD for relaxed models $M_{new}= 0.5M$.}
    \end{subfigure}
    
    \begin{subfigure}{0.45\textwidth}
        \centering
        \includegraphics[width=\linewidth, scale = 0.8]{imgs_model/model_eval_raw.png}
        \caption{CFTD for unrelaxed models $M_{new} = 1M$.}
    \end{subfigure}
    \hfill
    \begin{subfigure}{0.45\textwidth}
        \centering
        \includegraphics[width=\linewidth, scale = 0.8]{imgs_model/model_eval_relaxed.png}
        \caption{CFTD for relaxed models $M_{new} = 1M$.}
    \end{subfigure}
    \begin{subfigure}{0.45\textwidth}
        \centering
        \includegraphics[width=\linewidth, scale = 0.8]{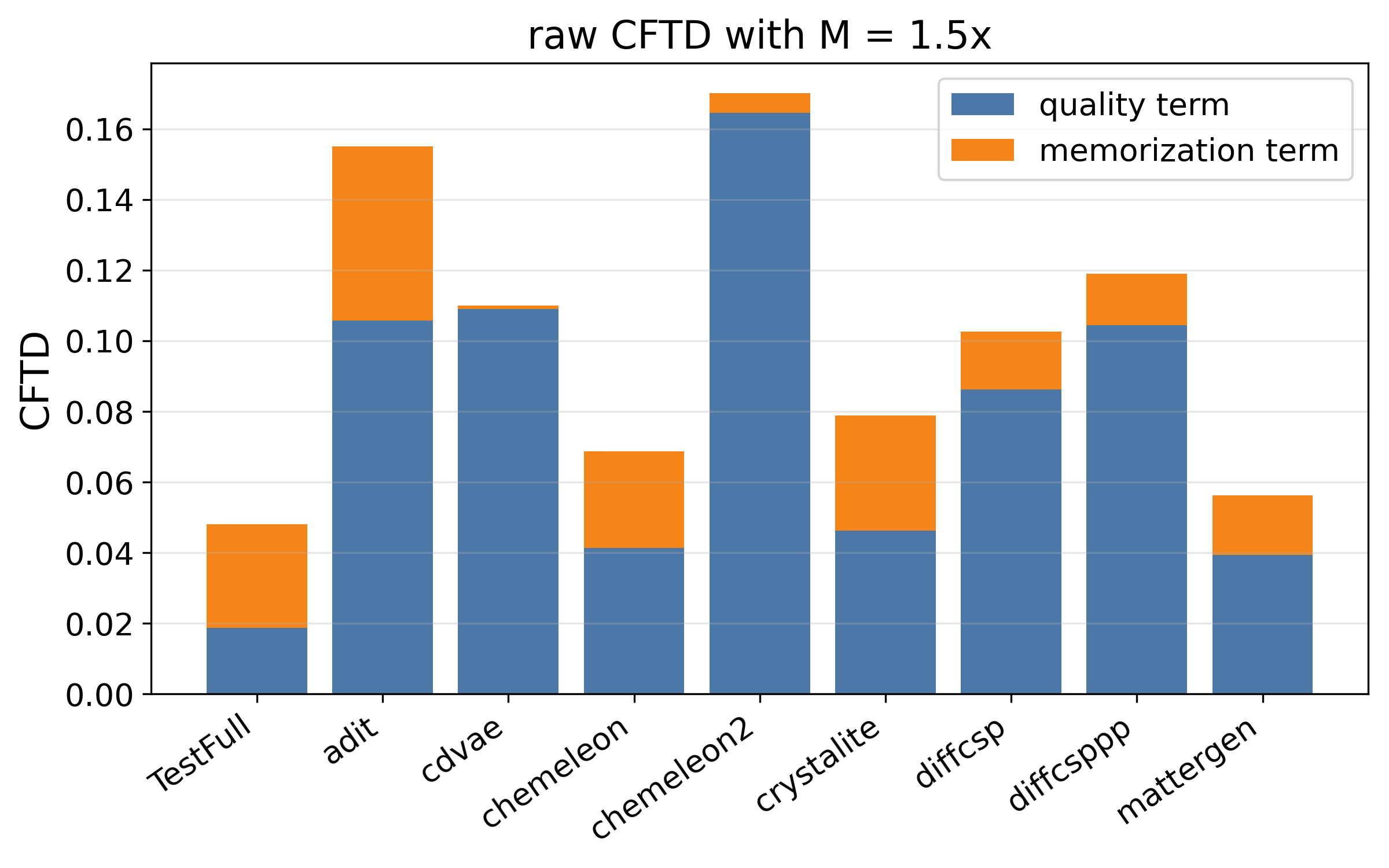}
        \caption{CFTD for unrelaxed models $M_{new} = 1.5M$.}
    \end{subfigure}
    \hfill
    \begin{subfigure}{0.45\textwidth}
        \centering
        \includegraphics[width=\linewidth, scale = 0.8]{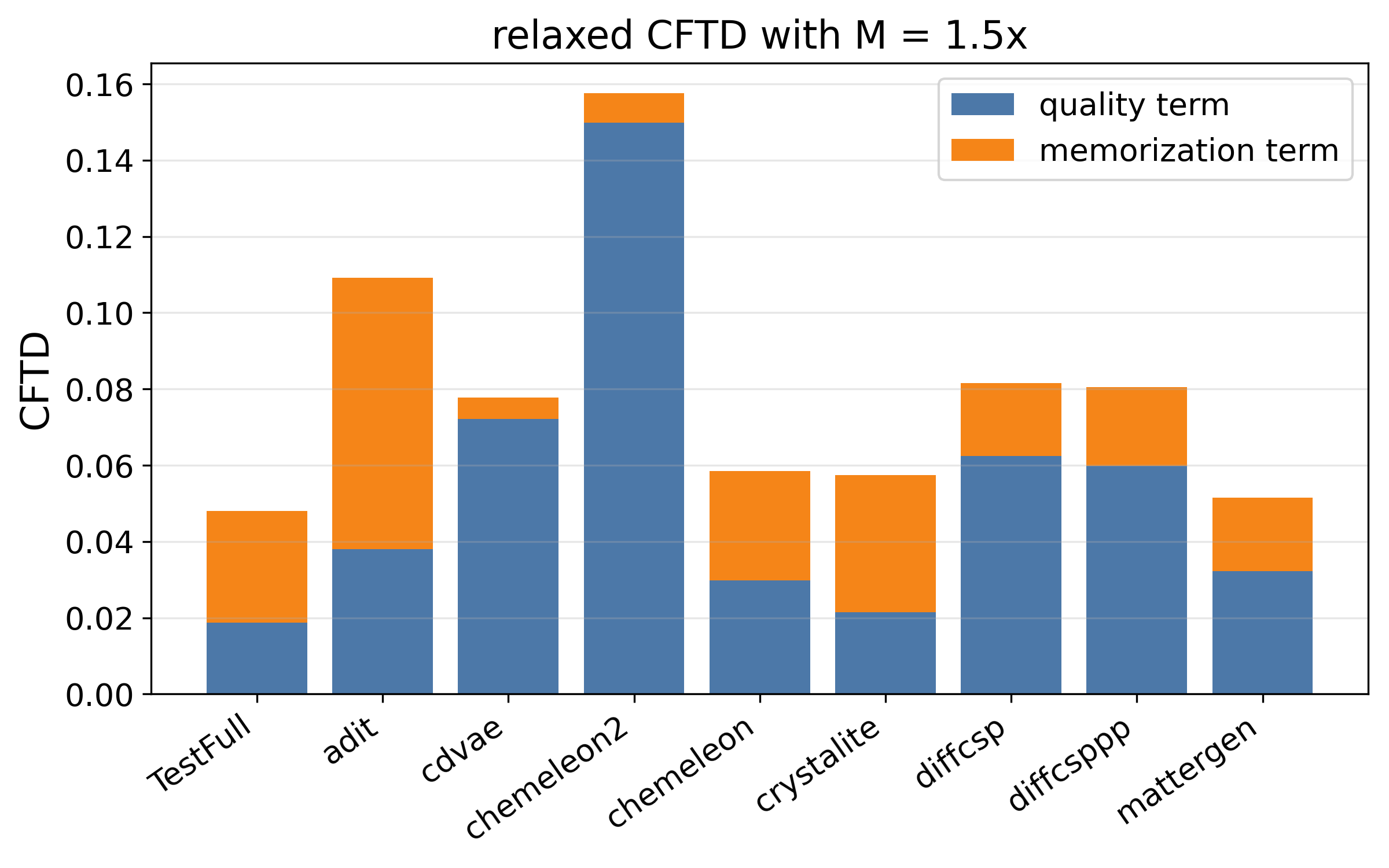}
        \caption{CFTD for relaxed models $M_{new} = 1.5M$.}
    \end{subfigure}
\caption{Comparison of common material generative models: relaxed and unrelaxed for varying $M$. }
\label{fig:models_cftd_M}
\end{figure}

\section{Model ranking}

Here, we comment in more detail on the rankings of the models and how they compare to the rankings implied by continuous SUN\cite{negishi2026continuoussunstableunique}. For this comparison, we dismiss Crystalite as it was not available in their study. 

Also, they ran their results only on unrelaxed models and we therefore compare the rankings only for those. Table \ref{tab:ranking_comparison} shows that our method largely agrees, with only minor differences between MatterGen and the test set, and DiffCSP++. The only major disagreement is the ranking of Chemeleon2. Note that their ranking changes significantly depending on the choice of the tunable weight hyperparameters for stability ($w_S$), uniqueness ($w_U$), and novelty ($w_N$). In their study, they show results for varying stability weights $w_S$. For instance, with $w_S = 0$, CDVAE is ranked 2nd, since novelty is overvalued and it is very novel by generating low-quality structures. 

\begin{table}[htbp]
\centering
\begin{tabular}{|l|c|c|}
\hline
\textbf{Model} & \textbf{cSUN Rank} & \textbf{CFTD Rank} \\ \hline
MP20 Test Set  & 3 & 1 \\ \hline
MatterGen      & 2 & 2 \\ \hline
Chemeleon      & 4 & 3 \\ \hline
DiffCSP        & 5 & 4 \\ \hline
CDVAE          & 6 & 5 \\ \hline
DiffCSP++      & 7 & 6 \\ \hline
ADiT           & 8 & 7 \\ \hline
Chemeleon2     & 1 & 8 \\ \hline
\end{tabular}
\caption{Rankings taken from the cSUN study \cite{negishi2025continuous} with $w_S = 1$ 
(SUN with elm+am) and our CFTD ranking for $\beta=0.6$ and calibrated $M$, so that quality and memorization are balanced.}
\label{tab:ranking_comparison}
\end{table}

\section{Model Description}
\label{sec:model}

\paragraph{Atom type sampling}
We model the sequence of atom types as $(a_1, ..,a_n)$ with n as the number of atoms in a unit cell. Each atom type takes a value in the discrete set $a_i \in \{1,...,K\}$, where $K$ is the number of distinct atom types we wish to model (i.e., atom types present in the dataset). Sequence modeling is therefore hard, since we have $K^n$ different possible sequences. As the atom types are categorical, we use the framework of discrete flow matching \cite{gat2024discrete}. 
Within this framework, we are given training samples $A_1 \sim P(A|n,h)$, where $A = (a_1,...,a_n)$ is a random variable drawn according to the law (or dataset). We start with a fully masked atom-type assignment $A_0 = (m,...,m)$, where $m$ is a mask token. Then,  we interpolate between the fully masked $A_0$ and the unmasked vector $A_1$ of atom types, by setting $A_t = (1-B)\ A_0 + B\ A_1$, where $B$ is a vector of independent Bernoulli random variables, where an entry is one with the probability $t \in [0,1]$. Then, similar to flow matching, we learn a conditional probability on the categorical simplex $\Delta^{K-1}$, i.e., probabilities on the atom types for each atom given the partially unmasked structure, which we denote by $q_{\theta}(A_1|A_t,t)$. We train this via forward KL, i.e., maximize the cross entropy 

$$\mathbb{E}_{a_1 \sim p, t \in U([0,1])}[ - \log q_{\theta}(a_1|a_t, t)].$$

Here, in practice, we condition on all relevant variables, i.e., the number of atoms and the pooled MACE feature projections as additional conditioning. 

After training our categorical distribution on $\{1,...,K\}^n$, we sample following the masked diffusion procedure from \cite{sahoo2024simple}. From the definition, it follows for each component of the path at time $t$ that $p(A_t^i = m) = 1-t$ and correspondingly $p(A_t \neq m) = t$. Then, for sampling, the probability that a previously masked token $A_{t_1}$ is unmasked at a later time $A_{t_2}$ is given by $p(A_{s} =  m | A_{t} = m) = \frac{1-s}{1-t}$, and hence  $p(A_{s} \neq  m | A_{t} = m) = \frac{s-t}{1-t}$. Following this, we start with $s_0 = (m,...,m)$ and then for each time step of $T = (t_1,...,t_T)$ steps, then for each position in $s_t$ we do nothing if it is already unmasked or we unmask with probability $\frac{t_{k+1}-t_k}{1-t_k}$. This matches the conditional formulas given by the path. Furthermore, as $t$ approaches 1, all tokens become unmasked.

\paragraph{Lattice sampling}
For lattice and position sampling, we mainly use continuous-flow matching. In contrast to discrete-flow matching, we here interpolate continuously in the space of random vectors, not in the space of measures. We follow mainly the presentation in \cite{liu2023flow}. Assume we already drew the positions $X_1$. Flow matching interpolates between an auxiliary latent $X_0 \sim \mathcal{N}(0,I)$ and $X_t = t\ X_1 + (1-t)X_0$ to learn its velocity $v_t(X_t) \approx X_1-X_0$. It can be shown that considering the Ordinary Differential Equation (ODE) $dX_t =v_t(X_t)$ with initial $X_0 \sim \mathcal{N}(0,I)$ yields (approximate) samples from the positions.

For lattice generation, we follow \cite{jiao2024spacegroupconstrainedcrystal} and its proposed polar decomposition. We apply this decomposition to the lattice matrix $L$, which is given by the lattice vectors. With real volume, this is invertible, and we can find a decomposition $L = Q \exp(S)$ with an orthogonal $Q$ and symmetric $S$. By utilizing the space of symmetric matrices as a vector space, we can find a basis and decompose $S = \sum_{i=1}^6 k_i B_i$, where $B_i$ is the basis and $k_i$ are the coefficients. The basis matrices are taken from~\cite{jiao2024spacegroupconstrainedcrystal}. Since the space of
$3 \times 3$ symmetric matrices are six-dimensional, we choose the following basis:
\[
B_1 =
\begin{pmatrix}
0 & 1 & 0 \\
1 & 0 & 0 \\
0 & 0 & 0
\end{pmatrix},
\quad
B_2 =
\begin{pmatrix}
0 & 0 & 1 \\
0 & 0 & 0 \\
1 & 0 & 0
\end{pmatrix},
\quad
B_3 =
\begin{pmatrix}
0 & 0 & 0 \\
0 & 0 & 1 \\
0 & 1 & 0
\end{pmatrix},
\]
\[
B_4 =
\begin{pmatrix}
1 & 0 & 0 \\
0 & -1 & 0 \\
0 & 0 & 0
\end{pmatrix},
\quad
B_5 =
\begin{pmatrix}
1 & 0 & 0 \\
0 & 1 & 0 \\
0 & 0 & -2
\end{pmatrix},
\quad
B_6 =
\begin{pmatrix}
1 & 0 & 0 \\
0 & 1 & 0 \\
0 & 0 & 1
\end{pmatrix}.
\]

For the flow matching onto the lattice coefficients $k_i$, we need to convert the lattice representation $(a,b,c, \alpha,\beta, \gamma)$ into the spaces of the $k_i$. From the $(a,b,c,\alpha, \beta, \gamma)$ we can compute the metric matrix $G = LL^T$. From this, we can find an eigendecomposition, $G = V \Lambda V^T$, where $\Lambda$ is the diagonal matrix of eigenvalues and $V$ are the eigenvectors. Now the matrix $S = V \log (\Lambda^{1/2}) V^T$ satisfies $\exp(S)\exp(S)^T = G$, since $S$ is symmetric itself. Now we convert these matrix entries into the representation with respect to the k-coefficients. It is easy to see that $k_1 = S_{0,1}$, $k_2 = S_{0,2}$ and $k_3 = S_{1,2}$, since $B_1$ is the only basis matrix with non zero entry at $(0,1)$ (and similar for the other two). For the remaining $k$, we need to solve a linear system. For $(0,0)$ we have $S_{0,0} = k_4 + k_5 + k_6$, and further $S_{1,1} = -k_4 + k_5 + k_6$ and $S_{2,2} = -2k_5  + k_6$. It is simple to check that $k_4 = \frac{1}{2}(S_{0,0}- S_{1,1})$, $k_5 = \frac{1}{6}(S_{0,0}+S_{1,1}-2S_{2,2})$ and $k_6 = \frac{1}{3}(S_{0,0}+ S_{1,1}+S_{2,2})$ solve this. Now we can get from $(a,b,c, \alpha,\beta,\gamma)$ to the coefficients, and use flow matching with a latent $Z \sim \mathcal{N}(0,I)$ to the $k_i$, from which we can build a representation of the lattice matrix $S$.

\paragraph{Position sampling}
For the positions, we also employ continuous-flow matching conditioned on the pooled MACE features projections, the lattice, and atom types. The main technical difficulty is that material generation introduces a periodic structure, so we cannot naively use flow matching in the Euclidean space. Let $\mathbb{T}^3_n$ denote n atoms (points) living each in the 3-dimensional torus space with the periodic structure. Further, we make the assumption that $[x+ \delta] = [x_1+\delta, ..., x_n + \delta] = [x]$, i.e., we have translational invariance for $x \in \mathbb{R}^{n,3}$ and $\delta \in \mathbb{R}^3$. We perform flow matching on the fractional coordinates. Assume we have $X_0 \in \mathbb{T}^{3}_n$ as a latent random vector, and $X_1 \in  \mathbb{T}^3_n$.

The general idea of flow matching is now to regress onto the velocity $X_1 - X_0$. However, $X_1-X_0$ is in general not in $[0,1]^{n,3}$ due to the periodic boundary conditions, so we need to be careful there. For an element in $\mathbb{T}^3_n$ , there are many different elements representing $[X_1-X_0]$, namely $X_1-X_0+l+\delta$ with $l \in  \mathbb{Z}^{n,3}$ and $\delta \in \mathbb{R}^3$. In particular, since flow matching is defined to live on the tangent space \cite{wu2026dmflow, sriram2024flowllm}, we use the representative given by $P(\mathrm{wrap}(X_1-X_0))$, where $\mathrm{wrap}(x) = (x+0.5 \ \mathrm{ mod}\  1)-0.5$ and $(P(x))_i = x_i - \sum_{k=1}^n \frac{1}{n} x_k $. This can be seen by either Riemannian geometry \cite{chen2024flow} or by choosing the representative $\Delta = \mathrm{argmin}_{l,\delta} \Vert x_1-x_0 + l + \delta \Vert^2$. 

Similarly, we choose the congruent path representative $\mathrm{frac}(X_0 + t \Delta)$, which has derivative (up to the discontinuity points of $\mathrm{frac}$) $\Delta$.

Then we regress our velocity onto $$L_{pos}(v_{\theta}) = \mathbb{E}_{x_0,x_1,t}[\Vert v_{\theta}(t,x_t) -\Delta \Vert^2].$$ 
This also has neat differential-geometric interpretations \cite{chen2024flow}.

\section{Possible Extensions}
\label{app:ext}
\subsection{Hierarchical Generation}
One of the motivating pillars of this work is to strip away complexity from generating composition via LLM ideas or positions via GNNs. The first one seems particularly prone to overfitting, whereas the latter often struggles with finding stable or space group diverse structures, as indicated by the results in \cite{betala2026lematgenbenchunifiedevaluationframework} (see also their leaderboard). Conveniently, the pooled MACE features projections now live in a fixed-dimensional continuous space, where we can apply ideas from standard generative modeling in Euclidean space. In particular, this enables us to learn $p(n,h)$ via a generative model, for instance, a transformer using the flow matching framework. Then, we could pursue precisely like in our material generator, and conditionally sample atom types, lattice and positions, therefore being a full "hierarchical" MACE conditional generative model. 

In particular, this allows us to sample "new" MACE features, therefore bypassing the composition-overfitting problem. However, a priori, the MACE feature space does not seem to have any particular structure. Many choices we made, such as mean-pooling, are arbitrary. This suggests a deeper analysis of how to compress the information in the MACE features while making it as "learnable" as possible by a generative model. We leave this for future work, but preliminary results suggest that we can achieve a high novelty generator with a similar MSUN rate.

Further, another direction to explore is using \emph{atom-level} features of the MACE embeddings. We did not pursue this since OT deliberately requires a fixed dimension. However, for material generation, finer-grained information might be helpful. 

\subsection{Conditional Generation}
In this paper, we were mainly concerned with unconditional material generative models, but both the metric idea and the generative model can be extended for conditional generators, albeit extending the metric is not trivial and might need ideas like conditional Wasserstein-like couplings \cite{chems_2025}. We want to focus on extending a MACE generator to a conditional one. Indeed, we could subsample or steer a MACE generator (say a generative model in the MACE feature space) to sample only from $p(h|c)$. This is possible as soon as we have paired data $(h_i,c_i)$ and would avoid training another conditional generative model that takes the condition as input for a factorization. This might provide a significant speed-up as soon as the condition is identifiable enough using the MACE features. The experiments in Appendix \ref{app:info_mace} suggest that MACE features indeed encode important material aspects.

\end{document}